\documentclass[journal]{IEEEtran}
\usepackage{graphicx}
\usepackage{amsmath}
\usepackage{amsfonts}
\usepackage{amsthm}
\newtheorem{proposition}{Proposition}
\newtheorem{definition}{Definition}
\newtheorem{lemma}{Lemma}
\usepackage{url}
\usepackage[ruled,vlined]{algorithm2e}
\usepackage[capitalize]{cleveref}
\Crefname{equation}{Eq.}{Eqs.}
\Crefname{figure}{Fig.}{Figs.}
\Crefname{tabular}{Tab.}{Tabs.}
\Crefname{section}{Sect.}{Sects.}
\Crefname{proposition}{Prop.}{Props.}

\usepackage{tikz}
\usetikzlibrary{arrows.meta,positioning,shapes.geometric}
\usetikzlibrary{decorations.pathreplacing,fit,backgrounds}
\usepackage[most]{tcolorbox}
\usepackage{booktabs}
\usepackage{multirow}
\usepackage{placeins}

\newtheorem{remark}{Remark}
\newtheorem{example}{Example}

\newcommand{\suppname}{the appendices}
\newcommand{\Suppname}{The appendices}
\newcommand{\suppshort}{appendices}
\newcommand{\mainname}{the body of the article}

\providecommand{\phantomsection}{}
\begin{document}

\title{Agentic Service Markets \\ Across the Computing Continuum: \\ A Polymatroidal Architecture}

\author{Lauri~Lov\'{e}n,~\IEEEmembership{Senior Member,~IEEE},
Alaa Saleh,~\IEEEmembership{Student Member,~IEEE}, 
Reza Farahani,~\IEEEmembership{Member,~IEEE},
Ilir Murturi,~\IEEEmembership{Senior Member,~IEEE},
Sujit Gujar,
Miguel Bordallo López,~\IEEEmembership{Senior Member,~IEEE},
Praveen Kumar Donta,~\IEEEmembership{Senior Member,~IEEE}, and 
Schahram Dustdar,~\IEEEmembership{Fellow,~IEEE}

\thanks{L. Lov\'{e}n (lauri.loven@oulu.fi) and A. Saleh are with the Future Computing Group, University of Oulu, Finland. 
R. Farahani is with the Distributed Systems Group, TU Wien, Vienna, Austria.
I. Murturi (Corresponding Author) is with the Department of Mechatronics, University of Prishtina, Kosova.
S. Gujar is with Machine Learning Laboratory, International Institute of Information Technology (IIITH), Hyderabad, India.
M. B. López is with the Multimodal Sensing Lab, University of Oulu, Finland.
P. K. Donta is with the Department of Computer and Systems Sciences, Stockholm University, Sweden.
S. Dustdar is with ICREA Barcelona, Spain and the Distributed Systems Group, TU Wien, Vienna.
}%

\thanks{Funding: This work was supported by the Research Council of Finland through the 6G Flagship program (grant 318927) and the CO2CREATION Strategic Research Council project (grant 372355), by the EC through HEU NEUROCLIMA project (GA 101137711) as well as the Interreg Aurora ResilientEdge project (Grant Number: 20373282), the ERDF (project numbers A81568, A91867), and by the Business Finland through the Neural pub/sub research project (diary number 8754/31/2022).
}}

\markboth{IEEE Transactions on Services Computing}%
{Lov\'{e}n \MakeLowercase{\textit{et al.}}: Agentic Service Markets Across the Computing Continuum}
\maketitle
\begin{abstract}
Autonomous AI agents are becoming economic actors. They compose deadline-bound services across the device--edge--cloud continuum and contend for capacity no single operator owns. This article asks when decentralised pricing can coordinate them exactly and truthfully, and what their service dependencies must satisfy. We model the pipelines as service-dependency graphs under governance constraints and identify the decisive property: laminarity of the induced leaf-block family. Where it holds, the feasible allocation space is a polymatroid and efficient incentive-compatible clearing exists per epoch; where a crossing breaks it, exactness and truthfulness can fail, while clearing prices survive on every totally unimodular family. The formal results instantiate polymatroid, gross-substitutes and Vickrey--Clarke--Groves theory; the contribution is the bridge from dependency structure to that machinery and an integrator architecture that exposes a laminar interface. The node-level evaluation ($32{,}020$~runs), at testbed-calibrated congestion, locates the boundary. One crossing costs exactness in $14.7\%$ of high-load rounds and none on laminar instances, and truthfulness fails where exactness does, across contention. An integrator's inner exposure restores exactness at a cost in admitted volume, and before translation overhead it cuts median latency by $16$ to $17$ percent. Against truth-free posted prices the tuned market leads at medium load and trails a demand-responsive one in four of six high-load cells. Recorded and composed agent workloads drive the allocation. For platform designers, laminar pipelines admit exact decentralised pricing, truthful under the VCG mechanism, and a crossing needs an integrator's inner exposure.
\end{abstract}

\IEEEpeerreviewmaketitle

\section{Introduction}\label{sec:introduction}
AI has moved from a niche capability to a mainstream requirement, embedding deadline-bound model decisions into everyday workflows that run across heterogeneous device--edge--cloud environments. Strict latency and quality-of-service (QoS) constraints there require careful allocation and orchestration of compute, storage, bandwidth and data-processing resources~\cite{shahid2026iot}.
This coordination challenge becomes even harder with emerging applications such as multi-modal inference, collaborative agents, and autonomous cyber-physical systems, which place heavy pressure on the management plane to coordinate dynamic workloads under capacity variability, changing network conditions, and policy constraints~\cite{meuser2024edgeai}. 
More recently, the service ecosystem has shifted toward \emph{agentic computing}, where autonomous AI agents generate tasks, procure services, negotiate resources, compose multi-step pipelines, and adapt their behavior based on observed performance, incentives, and environmental feedback~\cite{deng2025agenticservicescomputing,google2025ap2,donta2025agenticai}. Unlike static clients, these agents shape system load, so orchestration is no longer purely centralized~\cite{kokkonen2022autonomy} but a distributed process in which agent decisions directly drive resource contention and end-to-end QoS.

Modern distributed AI services form layered supply chains whose stages consume compute, storage and bandwidth at distinct tiers of the continuum~\cite{meuser2024edgeai}. These dependencies form a multi-layer directed acyclic graph (DAG) that complicates end-to-end placement, scheduling and capacity management, and governance requirements such as trust, access control, data locality and regulatory compliance further restrict which service paths are admissible~\cite{fornara2008institutions}. Traditional management-plane approaches optimize centrally, but cross-domain deployments spanning several operators and trust boundaries make unified centralized control impractical, which motivates economic coordination by pricing and auctions~\cite{mcafee1992dominant}. Neither that nor na\"{i}ve market control alone addresses the latency sensitivity, structural dependencies, autonomous agents and governance constraints of agentic AI service ecosystems, where dependency graphs introduce complementarities that can destabilize markets or yield inefficient placement. Shapira et al.~\cite{shapira2026agents} show that autonomous agents in multi-agent deployments can engage in unauthorized resource consumption and false reporting, failure modes that formal mechanism design aims to deter.

Consider a municipal operator running a latency-critical traffic-incident service. A monitoring agent calls an inference service that depends on camera streams processed at street-level edge nodes and on a cloud-hosted foundation model, and an alert that arrives after the vehicle has passed is worth nothing. Raw frames may not leave the city domain, the cloud provider sets its own terms, and neither operator will cede control of its capacity.

Central allocation across the two domains is unavailable, so prices must carry the coordination. Whether a price-based market settles depends on the shape of the dependency graph, and clearing prices survive well beyond the substitutable case. This article asks which structures stay inside that region, and what architecture recovers the property when the structure does not supply it.

This article develops a unified framework for agentic service markets across the computing continuum, in which autonomous agents buy and sell deadline-bound AI services, many provided by agents in turn. It integrates latency-aware valuations, a dependency-aware resource model of multi-layer service compositions, and governance constraints in which trust, policy and locality set a coordinate-wise upper bound on each resource. We identify the structural condition under which the feasible allocation set forms a \emph{polymatroid}, a convex region with submodular (diminishing-returns) capacity constraints that enables efficient polynomial-time optimization: a laminar leaf-block family, for which tree and series--parallel (SP) service-dependency topologies are sufficient.
Under slice encapsulation, agents' valuations satisfy the \emph{gross-substitutes} (GS) property (services substitute for one another), which, on the polymatroidal supply these regimes induce, guarantees market-clearing prices (\emph{Walrasian equilibria}) by discrete convex analysis~\cite{murota2003discrete}. Truthful bidding is a dominant strategy (dominant-strategy incentive compatibility, DSIC) under the Vickrey--Clarke--Groves (VCG) payments~\cite{vickrey1961counterspeculation,clarke1971multipart,groves1973incentives} of \cref{prop:mechanism}, which need no hypothesis on the valuation class. A crossing in the dependency graph breaks the laminar family on the supply side, costs the greedy allocation rule its exactness and, through it, truthfulness, and arbitrary graphs can also push valuations outside the GS domain. We therefore propose a \emph{hybrid management architecture} in which cross-domain integrators encapsulate service compositions into governance-compliant slices whose agent-facing interfaces are polymatroidal by construction, as an exact interface where a cluster's leaf blocks coincide and as an inner exposure where a crossing must be removed; a local-marketplace layer beneath them completes the design and is characterised analytically. \cref{tab:novelty} states each result against the machinery it instantiates, against the closest contemporaneous parallel work (Amin~\textit{et~al.}~\cite{amin2026market}), and against what the article does not claim, with an expanded positioning in \suppname{}.

\begin{table*}[!t]
\centering
\caption{Each result against the machinery it instantiates, the closest parallel work, and what the article does not claim.}
\label{tab:novelty}
\scriptsize
\renewcommand{\arraystretch}{1.05}
\setlength{\tabcolsep}{3pt}
\begin{tabular}{@{}p{2.4cm}p{3.1cm}p{5.7cm}p{6.3cm}@{}}
\toprule
\textbf{Result} & \textbf{Known} & \textbf{What this article adds} & \textbf{What the article does not claim} \\
\midrule
\cref{prop:polymatroid}: polymatroidal feasibility on laminar leaf-block families, with trees and series--parallel DAGs as sufficient conditions & Edmonds~\cite{edmonds1970submodular}; Fujishige~\cite{fujishige2005submodular}; Amin~\textit{et~al.}~\cite{amin2026market} on SP flow networks & The bridge from a multi-output, capacity-typed service-dependency DAG to a polymatroid, under leaf-block semantics in which series--parallel terminals are virtual and are not leaves & Nothing for general DAGs or outside the terminals-and-leaves convention. The region is exact under the article's token counting on every DAG. Laminarity decides only polymatroidality. Laminarity follows from the semantics and the capacities rather than from the drawn edges. No cross-epoch claim \\
\addlinespace[2pt]
\cref{lem:gs-valuations}: gross-substitutes valuations on the slice interface & Kelso--Crawford~\cite{kelso1982job}; Gul--Stacchetti~\cite{gul1999grosssubstitutes}; OXS~\cite{lehmann2006combinatorial}; discrete convex analysis~\cite{murota2003discrete} & The conditions under which slice encapsulation makes an agent's aggregate valuation the convolution of its unit-demand task valuations, hence an OXS assignment valuation, which is gross substitutes & Agent-level rather than a per-task bidder set. Fails under shared agent budgets, cross-task latency coupling or shared token caps, and is not closed under summation across agents \\
\addlinespace[2pt]
\cref{prop:mechanism}: Walrasian equilibrium, polynomial computation, agent-level DSIC & Kelso--Crawford~\cite{kelso1982job}; Murota~\cite{murota2003discrete}; VCG~\cite{ausubel2005vickrey} & Agent-level DSIC on a service-DAG-induced integer polymatroid, with the agent's aggregate OXS valuation as the bidder and a Clarke pivot that removes the agent in full & Attaches to the exact welfare maximiser rather than to the auction's price path. Weak budget balance under down-closed supply. No group strategy-proofness, no cross-epoch guarantee, no operator-side credibility. Clinching only under one task per agent. The truth-assuming single-domain planners the evaluation compares against are counterfactuals that bound the comparison \\
\addlinespace[2pt]
\cref{prop:encapsulation}: exact scalar interface where a cluster's leaf blocks coincide & Polymatroid contraction~\cite{fujishige2005submodular} & A scalar slice interface exported from a sub-DAG, the capacity lemma identifying that scalar as the sub-DAG's internal bottleneck, and the quotient-graph condition for polymatroidality & Requires one homogeneous slice type and a single-node interface per integrator. Otherwise the summary can over-commit and the integrator exposes a smaller polymatroid inside the deliverable region (the inner-exposure result in \suppname{}). Nothing for multi-output integrators, internal latency--throughput trade-offs, or shared leaves. The evaluated contraction is this interface on the parallel fan and an inner exposure on the tree and the crossing instance \\
\addlinespace[2pt]
Catalogue interface for multi-resource recipes (\suppshort{}) & Direct sums of polymatroid ranks~\cite{fujishige2005submodular}; Edmonds' greedy~\cite{edmonds1970submodular} & The exact boundary: at positive internal capacities the catalogue rank gives the resource-feasible region exactly when recipes within a resource-sharing group are identical, and over-commits by a computable factor when they are merely proportional & No exactness for proportional but unequal recipes. No polymatroidal catalogue region for non-proportional sharing, where the integrator can expose only a conservative inner polymatroid. The multi-dimensional case is only bounded \\
\addlinespace[2pt]
Three-layer hybrid architecture (\cref{sec:hybrid-architecture}) & Network slicing~\cite{foukas2017network}; continuum orchestration~\cite{sedlak2026service}; service function chaining~\cite{li2020service}; service composition~\cite{oasis2007bpel} & Integrator encapsulation as the structural route through regimes outside the laminar leaf-block hypothesis, and the division of labour that keeps the agent-facing interface polymatroidal & No deployment. Integrators are non-strategic infrastructure here. The local-marketplace and inter-market layers are characterised analytically and are not exercised; the evaluation runs the integrator's slice interface alone \\
\addlinespace[2pt]
Evaluation method (\cref{sec:evaluation}) & Auction-based edge allocation~\cite{habiba2023repeated,zheng2023resource}; measured agent instrumentation & A component ablation that varies the leaf-block family of \cref{def:feasibility-region} at a fixed node set and capacity vector, a misreport test across contention, recorded and composed agent workloads, and an emulated-testbed replay & A simulation on measured demand profiles, with no production deployment. The testbed is a single-machine emulation that measures the load response of real inference, against which the node model's congestion level is calibrated, and carries no encapsulation claim. The recipe experiment and two per-tier sweeps keep the per-tier environment by design. No claim of throughput superiority over a tuned production scheduler \\
\bottomrule
\end{tabular}
\end{table*}

The article makes three contributions. First, a structural bridge: a unified model of latency-aware valuations, resources and governance over a service-dependency DAG (\cref{sec:preliminaries}), and the leaf-block semantics, under which a node's leaf block is the set of leaves it reaches and a laminar leaf-block family, for which trees and series--parallel structures are sufficient, induces a polymatroidal feasible region with its valuation and mechanism consequences (\cref{prop:polymatroid,lem:gs-valuations,prop:mechanism}). Second, an architecture: integrator encapsulation exports an exact scalar slice interface where a cluster's leaf blocks coincide (\cref{prop:encapsulation}) and an inner exposure where they cross, that interface preserves the coordinate-wise governance restrictions, and a three-layer hybrid design, of which the evaluation exercises the integrator's slice interface, carries it across administrative domains (\cref{sec:hybrid-architecture}). Third, the evaluation (\cref{sec:evaluation}) comprises a five-experiment component ablation on instances that differ in their leaf-block family at a fixed node set and capacity vector, a test of the incentive guarantee under strategically misreporting agents across contention, an allocation driven by one recorded agentic workload and a composition we build from two recordings, whose family crosses, and a replay on an emulated testbed.
\section{Related Work}
Prior work on edge, fog and cloud platforms, on service composition and service function chaining, on agentic behaviour in distributed systems, on governance and trust, and on market-based resource allocation is surveyed in \suppname{}, which states for each lineage how it differs from this article and sets the lineages side by side in a comparison table.
The service-dependency DAGs used here are the object of service composition seen from the allocation side: composition fixes which services a task needs, and the question this article asks is which dependency structures admit a market that prices them. Repeated and double auctions for mobile edge computing~\cite{habiba2023repeated,zheng2023resource} trade resources independently of one another, whereas a task here needs a bundle whose dependency structure decides exactness and incentives, and \suppname{} sets these auction methods against the present mechanism.

Recent large-language-model (LLM) serving systems optimise the \emph{intra-node} substrate our workload model abstracts as per-stage demand (e.g.\ PagedAttention/vLLM manages KV-cache memory for single-engine throughput~\cite{kwon2023vllm}), and agentic benchmarks such as AgentBench evaluate an agent's task \emph{capability}~\cite{liu2024agentbench}. Neither addresses the cross-agent, cross-domain resource-market structure (dependency topology, incentive compatibility, governance) on which this article focuses.

Concurrent work by Amin et al.~\cite{amin2026market} establishes an SP-tractability and VCG-equivalence boundary for cooperative capacity-sharing on transport networks, using single-source--single-sink flow primitives and path-based pricing. We obtain a structurally analogous boundary for service-dependency DAGs in the compute continuum, with multi-output capacity-typed nodes and integrator encapsulation rather than path-based pricing as the route through non-SP regimes (\cref{sec:encapsulation}).
\section{Preliminaries}
\label{sec:preliminaries}
We consider a distributed service-computing environment spanning devices, edge platforms, and cloud infrastructure \cite{meuser2024edgeai}, populated by autonomous AI agents that generate tasks, compose services, and interact economically across the continuum, and \cref{fig:epoch} places the elements defined here in one mechanism epoch. The framework rests on standard assumptions from network and service management and from mechanism design, enumerated in \suppname{}, whose notation table collects the symbols used from here on. Agents condition their valuations on a commonly accepted system state~$s_t$.

\begin{figure}[!t]
    \centering
    \includegraphics[width=\columnwidth]{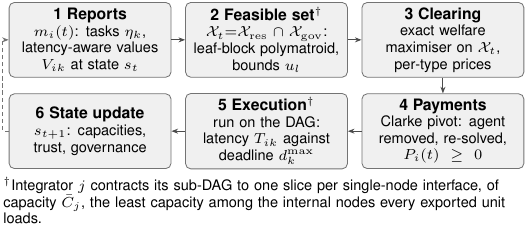}
    \caption{One epoch of the slice market. Agents report tasks and latency-aware valuations at the shared state (1), the agent-facing feasible set is the leaf-block polymatroid of \cref{prop:polymatroid} truncated by governance bounds (2), clearing is the exact welfare maximiser at per-type prices (3), payments are full-agent Clarke pivots (4, \cref{prop:mechanism}), execution realises latency against deadlines (5), and the state update opens the next epoch (6).}
\label{fig:epoch}
\end{figure}

The set of agents $\mathcal{A}$ includes service consumers, providers, integrators, and autonomous AI agents. Each agent $i \in \mathcal{A}$ has a private type $\theta_i(t) = (\ell_i, B_i(t), \mathcal{T}_i(t), \kappa_i, \rho_i, \phi_i(t))$ drawn from a type space $\Theta_i$, encoding location, budget, tasks, capabilities, role, and governance attributes \cite{fornara2008institutions}. Time proceeds in discrete periods $t = 1,2,\dots$.

Each task $k$ issued at time $t$ has parameters $\eta_k(t) = (q_k, d_k^{\max}, \lambda_k)$, where $q_k$ is workload, $d_k^{\max}$ a deadline, and $\lambda_k$ a latency-sensitivity parameter. If task $k$ is allocated to agent $i$, it experiences latency $T_{ik}$ and yields latency-aware value $V_{ik}(T_{ik}, q_k) = v_{ik}(q_k)\,\delta_{ik}(T_{ik})$. Here $v_{ik}(q)$ is the base value of completing task $k$ at quality $q$, and $\delta_{ik}$ is an exponential or deadline-based decay function with $\delta_{ik}(T) = 0$ for $T > d_k^{\max}$, so a task completed after its deadline is worth zero, with no salvage. Within each mechanism epoch the latency and quality attributes $(T_s, q_s)$ of available slices are \emph{exogenous}, fixed by the state~$s_t$ (\cref{lem:gs-valuations}, condition~(c)), while the congestion the round's admitted mix realises follows that mix and is charged against realised value.

\subsection{Service-Dependency Model}
Let $\mathcal{R}$ denote the set of service types, where services in $\mathcal{R}$ range from infrastructure primitives (CPU, GPU, bandwidth, storage) through data-processing and inference services to compound agentic capabilities.  The available capacity of service $r$ at location $\ell$ and time $t$ is $C_{r,\ell}(t)$, representing throughput (e.g., requests/s for an inference endpoint, GB/s for a data stream, or cores for raw compute).

Resources exhibit structural dependencies: certain services require bundles of underlying resources. We model this by a DAG $G_{\mathrm{res}} = (\mathcal{R}, E)$, where an edge $E = (r \to r')$ indicates that service $r'$ depends on service $r$ (e.g.\ an inference service depends on a pre-processing pipeline and a model-hosting service, each in turn depending on compute and storage). Node capacities $C_v$ (suppressing location and time within each epoch) are \emph{integral} (request slots, GPU cores, or bandwidth quanta), and allocations are integer-valued (the agentic instances' measured token capacities are floored). Tree and series--parallel structures arise naturally in hierarchical service compositions and yield polymatroidal feasibility under gross-substitutes valuations~\cite{murota2003discrete,fujishige2005submodular,amin2026market}. Arbitrary DAGs introduce complementarities that can undermine tractability in allocation and mechanism design \cite{bikhchandani1997competitive}.

\subsection{System State}

At each time $t$ the system state $s_t = (C(t), \Lambda(t), D(t), R(t), G(t))$ collects resource capacities, network latencies, exogenous load, agents' trust states and the governance state, which is distinct from the resource DAG $G_{\mathrm{res}}$ and specifies applicable policies and compliance constraints~\cite{fornara2008institutions}. It evolves with exogenous events and the realized allocation.

\subsection{Social Welfare and Objective}

At each time $t$, the social welfare generated by an allocation $x_t$ is
\begin{equation}
   W(t) = \sum_{i \in \mathcal{A}} \sum_{k \in \mathcal{T}_i(t)} V_{ik}(T_{ik}, q_k) - \sum_{j \in \mathcal{A}} \text{cost}_j(x_t) - \mu \,\text{Ext}(x_t),
\end{equation}
where $\text{cost}_j(x_t)$ is provider $j$'s operational cost and $\text{Ext}(x_t)$ captures externalities such as energy use or congestion, with trade-off parameter $\mu \ge 0$. In the evaluation $\mathrm{cost}_j$ is zero and $\mathrm{Ext}$ penalises utilisation above capacity, which only over-committing arms incur, so $W$ is realised latency-discounted value. Because $W(t)$ depends only on the allocation, payments are incentive instruments and do not alter social surplus. The per-epoch mechanism (\cref{sec:mechanism-design}) maximises demand-side valuations $\sum_i v_i$, with costs and externalities as planner-side deployment parameters. Allocative comparisons therefore use this objective, the exact allocative optimum on true values at zero queue over $\mathcal{X}_t$, the region of \cref{def:feasibility-region} under the round's governance bounds. Realised comparisons use the ex-post prefix optimum, the best value-ordered admission the realised congestion permits, and the distance between them includes the unpriced congestion externality.

In each epoch the system operator maximizes $W(t)$ subject to feasibility (\mbox{$x_t \in \mathcal{X}_t$}, defined below), incentive compatibility, individual rationality, and budget balance.

\subsection{Governance and Feasible Allocations}

Governance feasibility $\mathcal{X}_{\mathrm{gov}}(G(t),R(t))$ imposes \emph{coordinate-wise capacity restrictions}: for each leaf service~$l$, an upper bound $u_l(t) \in [0, C_l]$ determined by trust, locality, or role-based access control (e.g.\ $u_l = 0$ for services routed through a provider whose trust falls below threshold, or in a disallowed jurisdiction). Formally, $\mathcal{X}_{\mathrm{gov}} = \{x \ge 0 : x_l \le u_l(t) \text{ for all } l \in L(G)\}$, where $L(G) \subseteq \mathcal{R}$ denotes the leaf services consumed directly by agents. This coordinate-wise form abstracts multi-stakeholder trust boundaries, data-locality restrictions, and capacity-partitioning policies without modeling institutional processes explicitly.

An allocation $x_t$ specifies resources and execution pathways. Feasibility is $\mathcal{X}_t = \mathcal{X}_{\mathrm{res}}(s_t, G_{\mathrm{res}}) \cap \mathcal{X}_{\mathrm{gov}}(G(t),R(t))$, where $\mathcal{X}_{\mathrm{res}}$ enforces capacity and DAG constraints and $\mathcal{X}_{\mathrm{gov}}$ enforces governance via coordinate-wise upper bounds (the polymatroidal structure of $\mathcal{X}_t$ under laminar leaf blocks is established in \cref{sec:structural-regimes}).

\subsection{Messages, Mechanisms, and Utilities}

Agent $i$ sends a message $m_i(t)$ from a message space $\mathcal{M}_i$, reporting bids, asks, task parameters, or policy preferences. We focus on \emph{direct mechanisms} in which $\mathcal{M}_i = \Theta_i$ (each agent reports its type). A mechanism
\begin{equation}
\mathcal{M} : (s_t, m_1(t),\dots,m_{|\mathcal{A}|}(t)) \mapsto \bigl(x_t, P_1(t),\dots,P_{|\mathcal{A}|}(t)\bigr)
\end{equation}
maps messages and the current state to an allocation and payments, where $P_i(t)$ is the net payment by agent~$i$ (positive when paying, negative when receiving).

The per-period utility of agent $i$ under message profile $m(t) = (m_1(t), \ldots, m_{|\mathcal{A}|}(t))$ and type profile $\theta(t) = (\theta_1(t), \ldots, \theta_{|\mathcal{A}|}(t))$ is \emph{quasilinear} in payments:
\begin{equation}
u_i(m(t); \theta(t)) = \sum_{k \in \mathcal{T}_i(t)} V_{ik}(T_{ik}, q_k) - P_i(t),
\end{equation}
For notational convenience, we write $u_i(t)$ when the dependence on messages and types is clear from context. This quasilinear, private-values structure is required for the DSIC and VCG results in \cref{sec:mechanism-design}.
\section{Mechanism Design}
\label{sec:mechanism-design}

\subsection{Incentive Properties}

For direct mechanisms, in which agent $i$ reports its type, $m_i(t) = \hat{\theta}_i(t) \in \Theta_i$, and the per-period utility $u_i(t)$ of \cref{sec:preliminaries}, we adopt standard incentive-theoretic desiderata.

\begin{itemize}
    \item \textit{Dominant-strategy incentive compatibility.} The mechanism is DSIC if, within each decision epoch, truthful reporting maximizes per-period utility for each agent~$i$ regardless of what other agents report~\cite{ausubel2005vickrey}:
    \begin{multline*}
       u_i(\theta_i, m_{-i};\, \theta_i)
         \;\ge\; u_i(m'_i, m_{-i};\, \theta_i) \\
       \text{for all } m'_i \in \mathcal{M}_i,\;
         m_{-i} \in \mathcal{M}_{-i},\;
         \theta_i \in \Theta_i,\; i \in \mathcal{A},
    \end{multline*}
    where the first argument is the report, the subscript $-i$ collects other agents' reports, and the argument after the semicolon is agent~$i$'s true type.
    \item \textit{Individual rationality (IR).} Under truthful reporting, $u_i(\theta_i, \theta_{-i};\, \theta_i) \ge 0$ for all $\theta_{-i}$ and all participating agents $i$.
    \item \textit{(Weak) budget balance (BB).} The mechanism runs no deficit: $\sum_{i\in\mathcal{A}} P_i(t) \ge 0$ for all $t$.
\end{itemize}

\subsection{Structural Regimes for Efficient Allocation}
\label{sec:structural-regimes}

The service--dependency DAG $G_{\mathrm{res}}$ and governance constraints induce a feasible allocation set $\mathcal{X}_t$. A \emph{polymatroid} is the polytope $\{x \ge 0 : x(S) \le f(S)\; \forall S\}$ for a normalized, monotone, submodular rank function~$f$ (formal definition in \suppname{}, and intuitively a system of capacity constraints with diminishing returns).

\begin{definition}[Service-Feasibility Region]
\label{def:feasibility-region}
In the service-dependency DAG $G_{\mathrm{res}} = (\mathcal{R}, E)$, every two-terminal series--parallel component carries a \emph{virtual source} and a \emph{virtual terminal}: routing placeholders that host no service and carry no capacity, so the node identified in a series composition is virtual. Every node that is neither a virtual source nor a virtual terminal is a \emph{service node} with capacity $C_v \ge 0$, and a service node whose only successor is the virtual terminal of $G_{\mathrm{res}}$ is a \emph{leaf}. The leaf set $L(G) \subseteq \mathcal{R}$ collects the services agents consume, and for a service node $v$ we write $L_v = \{l \in L(G) : l \text{ reachable from } v\}$, with $L_l = \{l\}$ for a leaf. For a rooted tree, $L_v$ is the leaf set of the subtree rooted at~$v$. Each allocation variable $x_l$ represents a \emph{throughput token}: a unit of demand counted consistently at every ancestor constraint. The service-feasibility region is
\begin{multline*}
\mathcal{X}_{\mathrm{res}}(s_t, G_{\mathrm{res}}) = \\
  \bigl\{x \ge 0 : x(L_v) \le C_v \;\text{for every service node } v\bigr\},
\end{multline*}
where $x(S) = \sum_{l \in S} x_l$ (every token loads every ancestor of its leaf, so under this counting the region is exactly the set of demand vectors the node capacities admit, on every DAG, and laminarity decides only whether it is a polymatroid; a routing count, in which a token passes one of several parallel nodes rather than all of them, is a different model and admits more on a parallel fan). When $G_{\mathrm{res}}$ is a rooted tree (a tree with a single designated root), the family $\mathcal{L} = \{L_v\}$ is \emph{laminar}: any two members are either nested or disjoint. Given a laminar family, the polymatroid rank function is
\begin{equation}
f(S) \;=\; \min_A \sum_{v \in A} C_v,
\end{equation}
minimized over sets $A$ of service nodes with pairwise-incomparable (hence disjoint) blocks such that $S \subseteq \bigcup_{v \in A} L_v$.
\end{definition}

\begin{proposition}[Polymatroidal Structure under Laminar Leaf Blocks]
\label{prop:polymatroid}
Let $G_{\mathrm{res}}$ have node capacities $C_v \ge 0$ and a laminar leaf-block family $\mathcal{L} = \{L_v\}$ under the leaf-block semantics of \cref{def:feasibility-region}. A rooted tree and a two-terminal series--parallel network each satisfy that hypothesis. Then:
\begin{enumerate}
\item[\emph{(i)}] the capacity function $f$ defined in \cref{def:feasibility-region} is submodular, monotone, and normalized ($f(\emptyset) = 0$);
\item[\emph{(ii)}] $\mathcal{X}_{\mathrm{res}}$ is a polymatroid with rank function $f$; and
\item[\emph{(iii)}] the base polytope $B(f) = \{x \in \mathcal{X}_{\mathrm{res}} : x(L(G)) = f(L(G))\}$ is the set of tight allocations (i.e., allocations that saturate total leaf-level throughput).
\end{enumerate}
The proof is in \suppname{}.
\end{proposition}

Governance constraints $\mathcal{X}_{\mathrm{gov}}$ impose coordinate-wise upper bounds $x_l \le u_l$ on each leaf service (\cref{sec:preliminaries}). The intersection of a polymatroid with such coordinate truncations is again a polymatroid~\cite{fujishige2005submodular}, so $\mathcal{X}_t = \mathcal{X}_{\mathrm{res}} \cap \mathcal{X}_{\mathrm{gov}}$ inherits polymatroidal structure, provided governance is expressible as independent per-leaf bounds. More complex couplings need not preserve it (details in \suppshort{}).

\subsection{From Per-Task Valuations to Gross Substitutes}

\begin{lemma}[Agent Valuations under Slice Encapsulation]
\label{lem:gs-valuations}
Suppose integrators encapsulate multi-resource service paths into composite \emph{slices}, each a discrete, indivisible unit with fixed internal routing, deterministic latency $T_s$, and quality~$q_s$. Suppose further that (a)~each task requires at most one slice (unit demand), (b)~an agent's tasks are \emph{additively separable} (no cross-task complementarity, shared budget, or feasibility coupling beyond $\mathcal{X}_{\mathrm{res}}$), and (c)~slice attributes $(T_s, q_s)$ are fixed within each epoch by the state~$s_t$. Then each task's valuation is unit-demand and satisfies GS by Gul--Stacchetti~\cite[Sec.~2 and Thm.~1]{gul1999grosssubstitutes}, and agent~$i$'s aggregate valuation over slice bundles is their convolution, an assignment (OXS) valuation~\cite{lehmann2006combinatorial} that satisfies GS as well and is M$^{\natural}$-concave~\cite{murota2003discrete}. The proof is in \suppname{}. GS is closed under convolution, though not under summation across agents, so an agent's tasks form one gross-substitutes bidder while distinct agents do not.
\end{lemma}

Condition~(b) is critical: bidding on resource bundles along DAG paths without encapsulation would destroy the aggregate's GS property, as would the couplings \cref{tab:novelty} lists.

\begin{proposition}[Efficient Mechanism Design under Structured DAGs]
\label{prop:mechanism}
Suppose $\mathcal{X}_t$ is polymatroidal (\cref{prop:polymatroid}), agent valuations are as in \cref{lem:gs-valuations} (its additive-separability hypothesis~(b) included), agents have quasilinear utility, and slices are discrete items with integral capacities. Then:
\begin{enumerate}
\item[\emph{(i)}] a Walrasian equilibrium exists with per-type prices supporting a welfare-maximising allocation, by the discrete-convex-analysis route for M$^{\natural}$-concave valuations (\cref{lem:gs-valuations}) on M$^{\natural}$-convex polymatroid supply~\cite{fujishige2003note,murota2003discrete,fujishige2005submodular}, which specialise to the classical box-supply results~\cite{kelso1982job,gul1999grosssubstitutes};
\item[\emph{(ii)}] the exact welfare maximiser is computable by the M$^{\natural}$-convex intersection algorithm of discrete convex analysis~\cite{murota2003discrete}, polynomial in the agent count~$n$, the slice-type count~$K$ and $\log V_{\max}$, the largest reported task value, given the per-agent value oracles of \cref{lem:gs-valuations}, the aggregate oracle being obtained by successive convolution of the per-agent oracles. The salary-adjustment auction of Kelso and Crawford~\cite{kelso1982job} covers the box-supply special case, and its output at a finite price granularity is not the object part~\emph{(iii)} applies to; and
\item[\emph{(iii)}] the efficient allocation, computed as the exact maximiser of reported welfare, is implementable in DSIC \emph{at the agent level} by the VCG mechanism whose Clarke pivot removes an agent in full and recomputes welfare without it~\cite{ausubel2005vickrey}: truthful reporting of an agent's entire type across all of its tasks is a dominant strategy. Truthfulness needs no hypothesis on the valuation class. VCG is individually rational, and every Clarke-pivot payment is nonnegative under down-closed supply, so the mechanism is weakly budget balanced. The truthfulness guarantee does not extend to coalitions across agents.
\end{enumerate}
The proposition is restated with its proof in \suppname{}.
\end{proposition}

\subsection{Instability in General Dependency Graphs}

When the service--dependency graph is an arbitrary DAG, dependencies generate strong complementarities: a bundle's value may strictly exceed the sum of its parts. This pushes the environment outside the GS domain~\cite{gul1999grosssubstitutes} and destroys the laminar property of \cref{prop:polymatroid}, hence the polymatroidal structure of $\mathcal{X}_{\mathrm{res}}$. A crossing needs no such complementarity: the evaluation's crossing instance keeps unit-demand valuations and fails on the supply side alone. Exactness of the greedy allocation rule goes with the polymatroid. Winner determination is then NP-hard~\cite{nisan2006bidding,rothkopf1998computationally,lehmann2006combinatorial}, and Walrasian equilibria may fail to exist~\cite{gul1999grosssubstitutes,bikhchandani1997competitive}. Existence has a sharper boundary than the GS domain. Where the leaf blocks form a laminar family, a union of two laminar families, or intervals in some order of the leaves, the incidence matrix is totally unimodular and clearing prices exist in every round~\cite{bikhchandani1997competitive}, which the existence check of the evaluation's crossing sweep confirms, and beyond those families existence can fail.

\subsection{Restoring Tractability via Architectural Encapsulation}
\label{sec:encapsulation}

\begin{proposition}[Exact Scalar Interface Where Leaf Blocks Coincide]
\label{prop:encapsulation}
Let $G_{\mathrm{res}}$ be an arbitrary service-dependency DAG. Suppose integrators partition the non-leaf nodes into disjoint clusters, each managing a connected sub-DAG internally. Assume:
\begin{enumerate}
\item[\emph{(i)}] each integrator~$j$ exposes a single composite service (slice) with scalar capacity $\bar{C}_j = \min_{v \in P_j} C_v$, the least capacity among the internal service nodes $P_j$ that carry the export, each of which every admitted external unit loads exactly once;
\item[\emph{(ii)}] the units integrator~$j$ exports are interchangeable in its internal schedule: any internally feasible total throughput can be split arbitrarily across the coordinates it exposes, so the internal schedule constrains an external demand vector only through its total (a single homogeneous slice type is the leading case);
\item[\emph{(iii)}] no external feasibility constraints couple different integrators beyond the shared quotient-graph structure, and in particular agent-facing feasibility is exhausted by the quotient graph's leaf-block constraints and the coordinate-wise governance bounds of \cref{sec:preliminaries}; and
\item[\emph{(iv)}] each integrator presents a single-node interface: one entry node and one exit node, so that only a scalar throughput crosses each quotient arc.
\end{enumerate}
Let $G'$ be the quotient graph from contracting each cluster to a single node. If $G'$ has rooted-tree or series--parallel structure, the agent-facing feasible region is polymatroidal. The proof is in \suppname{}.
\end{proposition}

Under the fork-join counting, condition~\emph{(i)} requires every loaded cluster node to reach exactly the exported leaves, so \cref{prop:encapsulation} gives an exact scalar interface where the cluster's leaf blocks coincide and merges only coinciding blocks. The remedy for a crossing is an \emph{inner exposure}: any polymatroid inside the cluster's deliverable region, a reservation box among them, keeps every admitted allocation deliverable and keeps \cref{lem:gs-valuations} and \cref{prop:mechanism} in force for any internal topology, at a welfare cost that vanishes for every welfare weighting only when the exposed region is the whole deliverable region (proof in \suppname{}). The evaluation's contraction on the tree and the crossing instance is one.

\subsection{Hybrid Market Architecture}
\label{sec:hybrid-architecture}

We propose a three-layer hybrid encapsulation architecture that secures these polymatroidal and GS conditions at the agent-facing layer even when the underlying infrastructure does not.
\paragraph*{Cross-Domain Integrators} Integrators form the agent-facing layer of \cref{fig:mechanism-arch}. Each encapsulates a complex multi-resource service path into a governance-compliant slice (\cref{prop:encapsulation}), internally managing its sub-DAG and exposing a simplified, substitutable capacity interface equal to the sub-DAG's internal bottleneck. In realistic continuum deployments, sensing, compute, and radio interactions are tightly coupled within physical or operational domains under heterogeneous ownership and trust boundaries~\cite{rosendo2022distributed}. The integrator consolidates the sub-DAG's structural dependencies and governance constraints internally, at some efficiency cost, so that complementarities are absorbed inside the domain boundary and the inter-domain market only encounters substitutable, polymatroidal slice interfaces. Governance constraints spanning multiple domains (cross-border data sovereignty, inter-domain trust, jurisdictional restrictions) are enforced at the integrator level \cite{fornara2008institutions}.

\paragraph*{Local Marketplaces} Beneath the integrators, autonomous markets at device or edge scope coordinate the fungible services and resources (compute, bandwidth, storage, standard inference endpoints) that integrators draw on to fulfil slice commitments. They clear via lightweight auctions or posted prices and enforce local governance policies (trust thresholds, data-handling rules, admission control) \cite{foukas2017network,meuser2024edgeai}. Non-fungible services, such as jurisdiction-bound datasets, are allocated by discrete assignment instead~\cite{bikhchandani1997competitive}. \Suppname{} instantiates this layering, and the governance rule each layer owns, for the scenario of \cref{sec:introduction}.

\paragraph*{Inter-Market Coordination} Local markets and integrators exchange coarse-grained signals (aggregate demand, congestion indicators, trust updates) to maintain cross-domain consistency without system-wide optimisation: global governance and cross-domain-feasibility constraints flow through the integrators, while purely local resource signals may be exchanged directly between marketplaces \cite{10.1145/3773274.3777421}. These two layers are design, characterised analytically (\cref{sec:evaluation}).

\paragraph*{Scope and Failure Modes} Violating the proposition's four conditions defines the architecture's failure modes: multi-output integrators or heterogeneous slice types reintroduce complementarities at the agent-facing layer~\emph{(i)}; exported units that are not interchangeable in the internal schedule, as under latency--throughput trade-offs, cost the scalar its meaning~\emph{(ii)}; shared accelerators, multi-tenant links or cross-domain policy couplings create hidden feasibility constraints~\emph{(iii)}; and several entry or exit nodes make the exported type heterogeneous, so no single scalar represents the region and the advertised quotient capacity can exceed what the composed graph delivers~\emph{(iv)}. A single-node interface per integrator removes that last gap, as does an inner exposure. A cross-node governance coupling is no longer a coordinate-wise bound. It breaks the laminar structure of \cref{prop:polymatroid} and manufactures a crossing leaf block when imposed on a tree. The coupling the evaluation imposes is a joint token budget shared by two leaves in different domains, one member of that class; a data-residency rule $x_i\!\in\!\mathrm{EU}\Rightarrow x_j\!\in\!\mathrm{EU}$ is another. The recovery is to \emph{slice by domain}, restoring a laminar family and with it exactness, at a cost in admitted volume set by the capacity split. The multi-output case of~\emph{(i)} and the multi-dimensional case of~\emph{(ii)}, the catalogue interface among them, are characterised rather than simulated. The coupling of~\emph{(iii)} is the structurally distinct case, characterised analytically and measured in \cref{subsec:exp3-governance}.

\paragraph*{Illustrative Example}
The hybrid architecture resolves the scenario of \cref{sec:introduction} through two cross-domain integrators and two local marketplaces.
\emph{Integrator~1} (sensor-to-feature) encapsulates camera streams, frame extraction and feature computation, keeps raw frames inside the city domain, and exposes a ``feature-extraction slice'' with capacity $\bar{C}_1$, the bottleneck of its sub-DAG.
\emph{Integrator~2} (feature-to-inference) encapsulates cloud inference and post-processing under the provider's data-handling policies and exposes a ``model-inference slice'' with capacity~$\bar{C}_2$.
The two integrators compose in series, so the market offers agents one end-to-end slice at capacity $\min(\bar{C}_1, \bar{C}_2)$, a unit-demand valuation satisfying GS (\cref{lem:gs-valuations}). The evaluation runs a single-cluster analogue, one integrator's slice price posted inside a discovered-price market, beside the single-rule arms. The quotient graph (agent $\to$ slice~1 $\to$ slice~2) is series--parallel, and each integrator meets its neighbours through one ingress and one egress endpoint, as an API boundary already does. The agent-facing feasible region is therefore polymatroidal (\cref{prop:encapsulation}), and \cref{prop:mechanism} guarantees efficient, incentive-compatible allocation. Under this DSIC guarantee, agents truthfully reveal their valuations, so a slice market that clears at the exact maximiser of \cref{prop:mechanism}(ii) achieves the welfare-maximising allocation over the exposed slice interface, which for this series quotient is the deliverable set.

\begin{figure}[!t]
    \centering
    \includegraphics[width=0.68\linewidth]{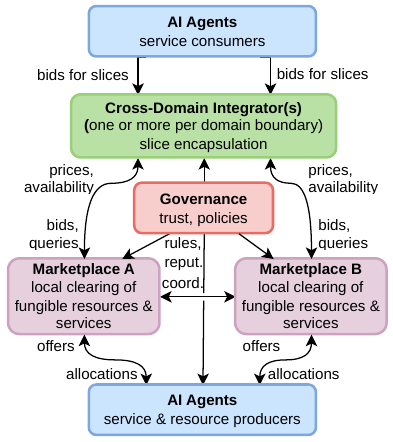}
    \caption{Hybrid architecture. Service-consuming agents trade over integrator-exposed slices (\cref{prop:encapsulation}), integrators bid for capacity at local marketplaces that coordinate provider-agent offerings, and governance flows through integrators across domains and through marketplaces locally.}
\label{fig:mechanism-arch}
\end{figure}
\section{Evaluation}\label{sec:evaluation}

We evaluate the framework through a systematic ablation over four removable components. Structural discipline~(S) is the leaf-block family of the service-dependency DAG, laminar or crossing, under \cref{def:feasibility-region}. The hybrid architecture~(H) is integrator encapsulation as a sub-DAG contraction with exponential-moving-average (EMA) price smoothing. Governance~(G) is coordinate-wise leaf caps with a cross-node coupling beside them, and the market mechanism~(M) is price-based coordination. Five ablation experiments and a cap-target probe remove or vary these components ($27{,}410$~runs) and a population sweep varies agent count at fixed instance ($1{,}680$~runs) (\cref{subsec:ablation-results}); the incentive, recipe, agentic and testbed blocks follow (\cref{subsec:beyond-ablation}) and the overhead sweep is read in \cref{sec:discussion} (run accounting in \suppname{}). The t\^{a}tonnement clearing used in the ablation is a price-discovery proxy and is not itself strategy-proof, so DSIC is verified separately under the VCG mechanism it characterises. Node-level means and paired contrasts carry Student-t $95\%$ intervals over each block's evaluation seeds, ten unless stated. The simulation models the demand side only, through the integrator's slice interface.

\subsection{Experimental Setting}\label{sec:simulation-setup}

Resources are the service nodes of the dependency graph, with per-node capacities summing to the tier totals $\{200, 300, 500\}$ and per-node base delays $\{5, 15, 50\}$~ms by tier. Every instance carries the same eight nodes, a device node $d$, three edge nodes $e_1$, $e_2$, $e_3$, and four leaves $l_1$ to $l_4$, and the same capacity vector, and the three differ only in which leaves each internal node reaches. T is a rooted tree. X is T plus the arc from $e_2$ to $l_2$, so that $e_2$'s leaf block $\{l_2, l_3\}$ crosses $e_1$'s $\{l_1, l_2\}$. S is a parallel fan, every internal node reaching every leaf, so its internal leaf blocks coincide and its region is the uniform-matroid case of \cref{def:feasibility-region}. The evaluated laminar instances are therefore a tree and a parallel fan; the series--parallel case of \cref{prop:polymatroid} rests on the proof and is exercised by the certificate and the exactness instrument on a two-terminal series--parallel instance in the crossing sweep (\suppshort{}). Agents generate latency-aware tasks over paths defined by service-dependency DAGs~\cite{li2020service,yang2022edge}. Agent population is the calibration parameter: $90$ (T), $90$ (X) and $45$ (S) agents put the three at the same offered load $\rho = \lambda N / K_c = 1.35$ at the uniform leaf mix, where $K_c$, the capacity of the binding leaf block at the demand mix, is $100$, $100$ and $50$ tasks per round. T and X are matched in bottleneck capacity and offered load. The crossing arc makes $e_2$ carry $l_2$ as well as $l_3$ and adds $e_2$ to $l_2$'s path, raising $l_2$'s reserve cost from $0.24$ to $0.32$. S is matched in rate only. At the skewed mix the offered loads are $1.62$, $1.62$ and $1.35$ and the crossing pair is never jointly over-demanded, so the crossing contrast is read at the uniform mix. Market clearing uses t\^{a}tonnement~\cite{arrow1958stability} on per-node excess demand. Agents face the induced leaf prices, and every price statistic is computed on them. Execution runs at two congestion parametrisations. The reported level, at which every node-level block runs, is the pair a released three-seed calibration sweep of the tree selects, its elasticity of median latency to relative offered load, $0.267$, closest to the testbed's $0.260$ ($0.24$ to $0.32$ over the instances on the full grid); the steep sensitivity level has elasticity $1.825$ on that sweep, and the mechanism comparison, the overhead sweep and the eight-setting parameter sweep run at both. The milliseconds are the model's scale (assumed base delays, a calibrated congestion term), and the market's clearing computation is not charged to task latency. 

Comparators are reported in three tiers. Tier~A are reference optima, unobtainable in deployment: the exact allocative optimum on true values at zero queue, and the ex-post prefix optimum, the best value-prefix admission through the realised congestion model and so a lower bound on the optimum. Tier~B are single-domain planners, feasible only where one operator controls the capacity: value-greedy on true values, a Kubernetes-style priority rank on assigned classes~\cite{k8sdocs2026priority}, and the value-ranked posted price, a planner behind a participation screen; earliest-deadline-first~\cite{liu1973scheduling} and random packing are controls. Tier~C are the deployable cross-domain rules, which need no single owner of capacity: the t\^{a}tonnement market, the static arrival-order and deadline-priority posted prices and a demand-responsive posted price, which read no value, and a mixed arm that posts the integrator's slice price inside a discovered-price market; a congestion-consistent market and an ascending-price variant are reported in \suppname{}. The market's parity with the static posted prices holds under truthful reports, which the posted prices do not need (\cref{subsec:exp7-dsic}). The architecture factor crosses every arm, and every arm carries four numbers: realised welfare over the ex-post prefix optimum, the allocative ratio to the exact optimum, its point on the welfare-versus-admission frontier, and the incentive certificate, which checks that the effective region's integral leaf-block rank is submodular and demand unit, a property of the instance and its caps shared by every arm. Per-cell tables, parameters and a crosswalk to the pipeline targets are in \suppname{}. The simulation code, the emulation harness under \texttt{emul/} and the comparison code are available at \url{https://github.com/Future-Computing-Group/agentic-economy-sim}, version \texttt{v2.1.0}, archived at \url{https://doi.org/10.5281/zenodo.22876728}; the released pipeline writes the node-level cell means, Student-t bounds and counts to \texttt{results/node-intervals.csv} and the calibration sweep to \texttt{results/calibration/}.

\subsection{Headline Results}\label{subsec:ablation-results}

\cref{tab:headline-grid} and \cref{fig:exp1-topology} report the structure experiment, whose structural measurement is allocative exactness. On T and on S, where the leaf-block family is laminar, Edmonds' greedy is the exact maximiser, so their zero shortfall follows from \cref{prop:polymatroid}. On X, where one arc makes two internal nodes' leaf blocks overlap without nesting and the rank stops being submodular, the packer falls short in $14.7\%$ [$12.4$, $17.0$] of rounds, each seed's worst round averaging $0.969$ [$0.962$, $0.976$] of that optimum and the worst of all $2{,}000$ rounds reaching $0.953$ ($10$~seeds of $200$~rounds). Price dispersion tracks load, and the instances do not order it.

\begin{table}[t]
\centering
\caption{Structure experiment at the matched operating point: $10$-seed means, uncontracted architecture, no governance. The subscript on each instance is the agent population $N$ that puts it at the same offered load, $\rho = 1.35$ at the uniform leaf mix reported here, against $K_c$, the capacity of the binding leaf block at this mix; that match is exact between T and X, whose $K_c$ is equal by construction, and a rate match only against S, whose $K_c$ is half theirs. \textit{Drop}: fraction of tasks that miss their deadline or are never admitted, a control matched on capacity at this mix and not on price rationing or on the deadline. \textit{Tok.}: admitted tokens per round. \textit{Welf.}: realised value less the congestion penalty, a.u., over all rounds. \textit{Lat.}: mean over rounds of the per-round median latency. \textit{Alloc.}: allocative ratio, the arm's admitted set scored on true zero-queue values against the exact allocative optimum. \textit{Exact.}: exactness shortfall, the share of rounds in which the greedy packer falls short of the round's exact leaf-block optimum, zero on T and S by Edmonds' theorem. $\sigma_p$: across-rounds coefficient of variation (CV) of the demand-weighted mean leaf price over all rounds, burn-in included, exactly $0$ below the dispersion onset, where the reserve pins every leaf price.}
\label{tab:headline-grid}
\scriptsize
\setlength{\tabcolsep}{3pt}
\begin{tabular}{@{}llrrrrrrr@{}}
\toprule
Inst. & Load & Lat.\ (ms) & Drop & Tok. & Welf. & Alloc. & Exact. & $\sigma_p$ \\
\midrule
\multirow{3}{*}{T$_{90}$}
 & low  & 70.7 & 0.000 & 45.09 & 34.27 & 1.000 & 0.000 & 0.000 \\
 & med  & 82.3 & 0.013 & 88.72 & 60.43 & 0.999 & 0.000 & 0.170 \\
 & high & 95.1 & 0.290 & 94.99 & 61.65 & 0.954 & 0.000 & 0.126 \\
\addlinespace[2pt]
\multirow{3}{*}{X$_{90}$}
 & low  & 70.7 & 0.000 & 45.09 & 34.25 & 1.000 & 0.000 & 0.000 \\
 & med  & 84.0 & 0.020 & 88.10 & 59.18 & 0.997 & 0.013 & 0.114 \\
 & high & 96.3 & 0.308 & 92.44 & 58.89 & 0.925 & 0.147 & 0.167 \\
\addlinespace[2pt]
\multirow{3}{*}{S$_{45}$}
 & low  & 70.5 & 0.005 & 22.41 & 17.07 & 0.996 & 0.000 & 0.000 \\
 & med  & 76.0 & 0.097 & 40.41 & 29.95 & 0.938 & 0.000 & 0.072 \\
 & high & 84.2 & 0.430 & 36.84 & 26.46 & 0.754 & 0.000 & 0.268 \\
\bottomrule
\end{tabular}
\end{table}

\begin{figure}[!t]
    \centering
    \includegraphics[width=\linewidth]{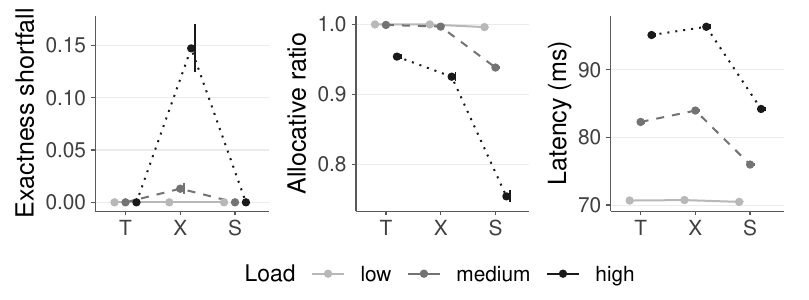}
    \caption{Structure experiment by instance (T, X, S) and load at the matched operating point: exactness shortfall (left), allocative ratio (centre) and median latency (right), with ten-seed Student-t $95\%$ intervals. Only the crossing instance X falls short of its round's exact optimum, and latency differs little between the tree and the crossing instance.}
    \label{fig:exp1-topology}
\end{figure}

\phantomsection\label{subsec:exp1-topology}%
\paragraph*{Structural discipline (S)} The Structure experiment runs the three instances at three load levels ($90$~runs). The onset of price dispersion sits at the same offered load in five of the population sweep's six instance-by-load cells and moves only on S at high load, whose $K_c$ is half the others'; that sweep is in \suppname{}.

\phantomsection\label{subsec:exp4-scalability}%
\paragraph*{Hybrid architecture (H)} The Architecture experiment crosses a sub-DAG contraction with EMA smoothing in a $2\times2$ factorial ($240$~runs), the price CV read after a $20\%$ burn-in so that the smoother's convergence from its entry price is not read as instability. On S every cluster node reaches every leaf, the contracted region equals the true one, and the arm is \cref{prop:encapsulation}'s exact scalar interface. On T and X the cluster's leaf blocks differ, and the bottleneck scalar is an inner exposure (\cref{sec:encapsulation}), a polymatroid inside the true region that falls short of it by up to half its optimum. On X the inner exposure restores exactness on the exposed region, the shortfall incidence falling from $0.1470$ [$0.1242$, $0.1698$] to $0$ and the certificate from failing to passing, and renders the crossing and the tree indistinguishable on the advertised side. At zero translation overhead it cuts median latency at high load by $16.3\%$ [$16.0$, $16.7$] of the uncontracted median on T ($15.5$~ms) and $17.0\%$ [$16.5$, $17.5$] on X, and by $16.4\%$ and $16.8\%$ with the filter on both arms. The cut is load shed by the reservation: admitted tokens fall by $20.38$ [$19.52$, $21.25$] per round on T and $17.83$ [$16.79$, $18.87$] on X ($17.04$ and $14.39$ with the filter), and without smoothing welfare over the ex-post prefix optimum falls by $0.065$ [$0.051$, $0.079$] on T and $0.029$ [$0.015$, $0.044$] on X. On S, where \cref{prop:encapsulation} holds, the trade runs the other way: without smoothing the contraction admits $12.57$ [$12.14$, $12.99$] more tokens per round, still inside the region's $50$ at zero over-commitment, and gains $6.81$ [$6.50$, $7.11$]~a.u.\ of welfare while running $8.9\%$ slower, partly through a lower reserve, the contraction cutting S's priced nodes per path from five to three.

The smoothing half is a filter on the reported series and, at the same time, the price agents face. On the series it is an identity: over the twelve EMA cells the ratio of the agent-facing to the cleared post-burn-in CV runs $0.2994$ to $0.3382$ against the white-noise value $\sqrt{(1-\beta)/(1+\beta)} = 0.3333$ at $\beta = 0.8$. Because agents respond to that price, it also moves the allocation: at high load it admits $0.69$ to $4.13$ more tokens per round and adds $0.16$ to $3.14$~a.u.\ of welfare in every cell where it registers, while on S under the contraction no interval excludes zero (per-cell tables in \suppname{}). The smoothing lowers the post-burn-in price CV in all six instance-by-load cells, and the contraction moves it in all six, in opposite directions across them, so a pooled test of the contraction cancels ($H(1) = 0.960$, $p = 0.327$, $n = 240$). Two interface levels test the conditions of \cref{prop:encapsulation}. The conservative inner region over-commits by $0.0000$ on all three instances. The routed-composition interface advertises the sub-DAG's node-split maximum flow, which is not the bottleneck of condition~(i), and over-commits: on S it advertises $150$ tokens, of which the device node passes $100$, and admits $67.43$ tasks per round against a deliverable $50$, an over-commitment of $17.45$ [$16.97$, $17.94$] tokens per round, against $6.33$ on T and on X at the skewed mix ($1.49$ and $2.57$ at the uniform mix).

\phantomsection\label{subsec:exp3-governance}%
\paragraph*{Governance (G) and component interaction} The governance ablation ($720$~runs) sets the coordinate-wise bound $u_l$ at four levels, none, a trust threshold, a jurisdiction predicate and a role-class fraction, and adds two cross-node objects, a joint token budget shared by two leaves in different domains, the measured member of the class a data-residency rule belongs to, and that budget sliced by domain, on the three instances at two loads and two leaf mixes ($10$~paired seeds). On T and S the region stays a polymatroid under every coordinate cap and the exactness gap is zero, as \cref{prop:polymatroid} and the truncation argument of \cref{sec:mechanism-design} require. At the evaluation point neither the trust threshold nor the role cap binds: the trust arm is identical to the no-cap arm in all ten seeds on every instance, and the role cap, which reads agent roles, leaves X's shortfall at $0.1465$ [$0.1243$, $0.1687$] against $0.1470$ uncapped. Partial repair is therefore read on the jurisdiction cap and a cap-target probe on X. Zeroing $l_1$, $l_2$ or $l_3$ leaves the surviving blocks nested, so exactness returns by the leaf-block combinatorics, while zeroing $l_4$ leaves the crossing pair untouched at a shortfall of $0.7505$ [$0.7312$, $0.7698$]. The jurisdiction cap, which closes $l_3$ and $l_4$ and with them half the offered demand, repairs the crossing and cuts median latency on X at high load by $4.5\%$ [$4.2$, $4.9$] ($4.4$~ms) at a cost of $32.5$ [$32.0$, $32.9$] coverage points. Unlike the companion deployment's per-domain locality predicates~\cite{loven2026nps}, the joint budget is not coordinate-wise: on T it manufactures a crossing, a shortfall of $0.1440$ [$0.1290$, $0.1590$] where there was none and a certificate that fails in ten seeds of ten; on S it nests inside the existing blocks, and on X it deepens the crossing, to $0.1780$. Slicing by domain restores exactness and the certificate at a cost: at the uniform mix the sliced arm admits $9.87$ [$9.39$, $10.34$] fewer tokens per round than the coupled one and gives up $1.19$ [$0.97$, $1.42$]~a.u.\ of welfare, and at the skewed mix it admits $3.96$ fewer tokens and gains $0.97$ [$0.74$, $1.21$]~a.u. The volume cost follows from the capacity-split rule; the stranded cross-domain demand, $0.2275$ tokens per round in closed form, is about two percent of it at the uniform mix.

Crossed with the jurisdiction cap ($240$~runs), the contraction adds nothing to exactness on X: both reduce the region to $l_1$ and $l_2$ under $e_1$'s $50$-token block, so each alone takes the shortfall from $0.1470$ to $0$, and the joint cell sits at the cap's own ceiling, $49.30$ tokens per round against $78.74$ under the smoothed contraction alone. With smoothing the contraction's welfare gain on X of $0.78$ [$0.31$, $1.25$]~a.u.\ under no cap becomes $0.10$ [$0.07$, $0.14$] under the cap; without it the contraction loses $2.34$ [$1.55$, $3.13$]~a.u.

\phantomsection\label{subsec:exp6-mechanism}%
\paragraph*{Market mechanism (M)} The mechanism ablation ($26{,}040$~runs over its level, tuning and evaluation grids) compares the market against the ablation levels, random feasible packing and earliest-deadline-first, against the planners of Tier~B, and against the truth-free posted prices of Tier~C over a grid of markups; every tunable rule chooses its level on tuning seeds, and the tuned comparison is evaluated on twenty further seeds. Every posted price screens participation at one level per leaf, anchored at that leaf's ancestor-path reserve cost. The arrival-order posted price then admits the participants in an order drawn per round and the deadline-priority one earliest deadline first, neither reading a value. The value-ranked posted price packs them by expected value, a planner behind a participation screen. At the reported level and high load the tuned market is at parity with the static truth-free posted prices: the arrival-order one leads it in two of the six instance-by-architecture cells, by at most $0.020$ [$0.018$, $0.022$] of the ex-post prefix optimum, trails it in two and ties in two ($20$~paired seeds), and the deadline-priority one is within $0.0004$ of it in every cell, so the screen sets the outcome. At medium load the market leads both in all six cells, by $0.003$ to $0.032$. The value-ranked posted price leads the tuned market in five of the six high-load cells, by up to $0.073$ [$0.068$, $0.077$]. With the level tuned on the other load, the arrival-order posted price trails the tuned market in all six measurable cells where that level differs, by $0.99$ to $5.32$~a.u., and with one level pooled over the architectures in three of four. Measured at high load, a demand-responsive posted price, re-set each round toward a target ratio of the demand its level screens in to capacity, as cloud spot pricing adjusts to demand, reaches $0.907$ to $0.973$ of the ex-post prefix optimum against the tuned market's $0.899$ to $0.976$, leading it in four of the six cells, by up to $0.0389$ [$0.0345$, $0.0432$], and trailing it in two, by $0.009$ and $0.014$; at medium load the market leads it in all six. Both orderings move with the congestion calibration: at the steep sensitivity level the static truth-free posted prices lead the market in three high-load cells and trail it in one (\suppshort{}).

The Kubernetes-style rank trails the tuned market in all six high-load cells at the reported level, by $0.015$ to $0.099$ of the ex-post prefix optimum ($0.872$ against $0.971$ on the uncontracted tree). In the level grid, with the market at the reserve, the uncontracted market leads value-greedy at the steep sensitivity level by $11.01$ [$10.44$, $11.58$]~a.u.\ on T and $12.55$ [$11.90$, $13.21$] on X at high load and trails it at the reported level by $0.74$ [$0.57$, $0.91$] and $0.48$ [$0.17$, $0.78$].

\subsection{Incentives, Recipes, and Recorded Agentic Workloads}\label{subsec:beyond-ablation}
\phantomsection\label{subsec:exp7-dsic}%
\paragraph*{Incentive compatibility under strategic agents (DSIC)} We run the VCG mechanism with Clarke-pivot payments~\cite{ausubel2005vickrey}, under which each agent pays the externality its presence imposes on the others. The run is gated on the incentive certificate, which meets \cref{prop:mechanism}'s hypothesis and passes on T and on S. We let one agent at a time \emph{misreport} its valuations by a shading factor $\alpha$ while the others stay truthful, and read each agent's \emph{regret}, the utility it forgoes relative to truthful reporting ($10$~seeds, $30$~rounds). It is zero at $\alpha{=}1$ by construction and positive for every misreport, reaching $0.0858$ [$0.0847$, $0.0869$] on T and $0.0867$ [$0.0834$, $0.0899$] on S at $\alpha{=}0.5$. A second arm searches a joint strategy set of $7$ to $24$ distinct reports for up to four tasks and $10 + 2k$ for $k > 4$ (\suppshort{}) for the largest utility gain any member attains at true valuations, over a contention sweep of five agent counts, $8$ to $90$, by four capacity scales, $0.1$ to $1$, on each instance ($620$~runs). On T and S the worst case stays below $8\times10^{-15}$, the floating-point floor, at all twenty contention points. On X the certificate fails, and the same payments over the greedy allocation the region forces leave a positive gain at the thirteen points where the greedy packer is inexact and a gain at the floor at the seven where it is exact. At the evaluation point, $90$ agents at full capacity, the mean gain is $0.151$ [$0.057$, $0.245$] and the largest $0.364$, at a shortfall incidence of $0.157$, with a non-truthful argmax in all ten seeds. Since the payments are the same VCG payments on every instance, the incentive property rests on exactness of the allocation rule. The tuned t\^{a}tonnement market, on the same strategy set and sweep, is manipulable under contention on the laminar tree as well as on X, and never on S: a member gains at $12$ of the twenty points on T and $9$ on X, worst-point means $0.194$ [$0.169$, $0.220$] and $0.169$ [$0.148$, $0.190$], most often by over-reporting its values by $10$ to $30$ percent, and the share of profitable runs rises from $4$ percent at full capacity to $94$ at a tenth on T ($0$ to $84$ on X). Its price search runs a fixed $15$ iterations and admission then rations by reported surplus at the terminal prices, so a higher report buys priority in that rationing.

\phantomsection\label{subsec:exp11-recipes}%
\paragraph*{Recipe structure and the catalogue boundary} The recipe experiment ($120$~runs, per-tier environment) holds aggregate demand and the critical path fixed and varies whether the slices sharing an internal resource carry one recipe or two. With one recipe the feasibility system is a cardinality constraint on which descending-value greedy is optimal by construction, so its ratio of $1.0000$ is an identity; with two it is $0.974$, $0.985$ and $0.993$ at capacities $6$, $9$ and $12$. Those recipes are not proportional, outside the catalogue proposition, whose factor $2 - 1/\lambda_2$ a proportional-recipe arm attains exactly at five ratios (\suppshort{}).

\phantomsection\label{subsec:exp8-agentic}%
\paragraph*{Recorded and composed agentic workloads} We instrument multi-step LLM tool-use agents on a local model and read the dependency DAG, per-stage latencies and per-node demand from each recorded execution (\suppshort{}). The first workload is one recording, plan, two parallel tool calls, then aggregate, whose single sink makes its leaf-block family laminar. The second is a composition we build from two recordings that share a tool stage; the family that crosses is the composition's, and the two workloads also differ in leaf count, node count, capacity split, deadlines and decay rate, so the contrast is not credited to the family alone. The measured stage latencies put the zero-queue critical path at $3875$ and $7411$~ms, from which the deadlines and the value-decay rate are re-derived, and every stage's demand is counted in one token unit. At high load the market's allocative ratio to the exact optimum is $0.558$ [$0.541$, $0.575$] on the composition against $0.900$ [$0.878$, $0.923$] on the single recording; the greedy packer falls short of the round's exact optimum in $0.026$ [$0.017$, $0.035$] of the composition's rounds and in none of the recording's, and the certificate fails on the composition in all forty branches and passes on the recording in all forty.

\phantomsection\label{subsec:emul-testbed}%
\paragraph*{An emulated testbed} One container per tier replays the admitted sequence of the agentic workload in the per-tier environment under assumed cross-tier delays ($245$ tasks, $980$ stage calls, $100$ rounds; \suppshort{}). Its elasticity of median latency to offered load, $0.2597$ [$0.1777$, $0.3753$], is the target the node model's reported level is calibrated to; the per-tier simulator that produced the replayed sequence gives $0.0496$ [$0.0354$, $0.0614$], about fivefold less, with lighter tails. Median wait is zero on every tier and the load dependence is in service time: contention is in the shared backend. The testbed calibrates the relative load response only: its median latency is $2.65$ times the simulation's at low load and $3.36$ times at high load, and it misses nearly every deadline where the simulator misses one of $150$ at high load, because the deadlines derive from the measured profile and the testbed pays for real inference.
\section{Discussion and Future Directions}
\label{sec:discussion}

\paragraph*{Scope} Time-varying capacities, assumption~(v) of \suppname{}, are not exercised. Three arms keep the per-tier identical-bundle environment, the recipe experiment and two per-tier sweeps, where topology enters through three channels only (\suppshort{}). The stylised graphs carry structural effects only, and the reported magnitudes are read as direction and relative magnitude. Agents bid on power-law congestion estimates fed by an online success model, while execution follows a queueing model on the critical path. The replay testbed is one topology, one backend and one seed.

\paragraph*{Encapsulation and smoothing} The exact interface of \cref{prop:encapsulation} and the inner exposure are statements about the region, not about price dynamics: the reduction in reported price dispersion (\cref{subsec:exp4-scalability}) is the low-pass identity of the smoothed series the integrator publishes, not a property of the market behind it.

\paragraph*{Governance as a structural component} The governance experiments (\cref{subsec:exp3-governance}) show that coordinate-wise caps keep every allocated task policy-compliant by construction and act on exactness only where they reshape the leaf blocks. The cost is specific to the instrument: partitioning removes capacity from each sub-pool. Strictness has therefore to be set against structure and load~\cite{sun2019trust,huang2022trust}. The planners, the value-ranked posted price among them, assume true values and a single owner of all resources. The static truth-free posted prices read no value but need a level set in advance. The market needs neither a single owner nor a preset level, but it needs truthful reports: it is manipulable under contention, and truthfulness belongs to the VCG allocation on a laminar region (\cref{prop:mechanism}). The welfare ordering among them states what each assumption buys, and the posted price set on the other load measures what a preset level is worth. Emerging standards for service composition and network slicing \cite{foukas2017network} and agent-transaction protocols \cite{google2025ap2} should treat encapsulation as a design-time concern.

\paragraph*{Translation overhead} Charged per task on the leaves the contracted cluster exports, in execution and in the bid, the protocol-translation overhead $\delta_{\mathrm{enc}}$ spends the contraction's latency lead at the reported level between $31$ and $41$~ms on T and X, while the welfare lead is negative at every positive overhead and admitted volume falls further; at the steep sensitivity level the lead, $68$ percent of the uncontracted median ($213$ to $217$~ms), survives $50$~ms (\suppshort{}).

\paragraph*{Existence and exactness} Value-greedy is welfare-optimal for the static per-epoch allocation when the leaf-block family is laminar and demand is unit, by Edmonds' greedy theorem~\cite{edmonds1970submodular} (\cref{prop:polymatroid}). Under dynamic congestion it can still be welfare-suboptimal (\cref{subsec:exp6-mechanism}). On the totally unimodular families of \cref{sec:mechanism-design} the linear-programming (LP) integrality existence check of the crossing sweep (\suppshort{}) finds no fractional round in the $32{,}400$ rounds of its $81$ unimodular instances. Beyond those families existence can fail, and does in about $8\%$ of rounds at the uniform mix on the generated families outside them. The crossing costs exactness while existence holds: on the sixty instances whose families cross while remaining totally unimodular the existence check is clean in all $24{,}000$ rounds, while the greedy packer loses up to $12.5\%$ of the round's optimum, and across their crossing counts, one to six, its mean exactness stays between $0.986$ and $0.993$.

\paragraph*{Future directions} Three questions stay open: an integrator that procures its inputs under mixed regimes, posted, fluctuating and auctioned, which makes the exported capacity a mixed AND-OR quantity; the credibility of the operator that runs the mechanism and of an integrator that misreports its exported capacity; and a truthful ascending implementation when tasks choose among heterogeneous slice types. A federated substrate organised around the integrator layer of \cref{prop:encapsulation} on a four-domain edge--cloud testbed is reported in~\cite{loven2026nps}, which reports its own clearing-price measurements and takes the price and incentive properties from this article.
\section{Conclusion}
\label{sec:conclusion}

We presented a framework for agentic service markets across the computing continuum. Where the leaf-block family of a service-dependency graph is laminar, as it is on trees and on series--parallel compositions, the feasible allocation space is polymatroidal (\cref{prop:polymatroid}), latency-aware valuations satisfy gross substitutes under slice encapsulation (\cref{lem:gs-valuations}), and efficient, incentive-compatible mechanisms exist per epoch (\cref{prop:mechanism}). For arbitrary graphs these properties can break down, and an integrator restores polymatroidal structure at the agent-facing interface, as an exact interface where its cluster's leaf blocks coincide (\cref{prop:encapsulation}) and as an inner exposure where it must remove a crossing. The evaluation tested the exactness prediction: the greedy rule is exact on the two laminar instances and not on the crossing one, the incentive guarantee fails on the crossing instance where exactness does across contention, and coordinate-wise leaf caps act as a structural component, preserving the region on the laminar instances and, on the crossing one, restoring a laminar region at three of four cap targets by removing the capped leaf's demand. The inner exposure made the crossing instance and the tree indistinguishable to agents and cut median latency before translation overhead, at a cost in admitted volume. On truthful reports the tuned market leads the truth-free posted prices at medium load; at high load it is at parity with the static ones and trails a demand-responsive one in four of six cells, and under contention it is manipulable (\cref{subsec:exp6-mechanism,subsec:exp7-dsic}).

Cross-node coupled governance, of which a data-residency rule is an example, stays outside the coordinate-wise model, and a domain slice removes the crossing it manufactures at a cost in admitted volume (\cref{subsec:exp3-governance}). The structural properties identified here give platform designers a formal criterion for where decentralised allocation stays exact, and truthful under the VCG mechanism.

\section*{Acknowledgment}
Generative AI tools (Claude, Anthropic) were used under the authors' direction to draft and edit text throughout the article and \suppname{}, to assist in writing and testing the simulation and testbed code, and to check proofs and numerical claims against the code. The authors reviewed and verified all content and take full responsibility for it.

\appendices
\crefalias{section}{appendix}
\crefalias{subsection}{subappendix}
\crefname{appendix}{App.}{Apps.}\Crefname{appendix}{App.}{Apps.}
\crefname{subappendix}{App.}{Apps.}\Crefname{subappendix}{App.}{Apps.}
\setcounter{definition}{0}
\setcounter{proposition}{0}
\setcounter{lemma}{0}
\setcounter{theorem}{0}
\section{Background Definitions}
\label{app:background}

The formal results in \mainname{} rely on three standard concepts from combinatorial optimisation and economic theory. We collect their definitions here for self-containedness; readers familiar with these concepts may proceed directly to the proofs. Notation is provided in \cref{tab:notation-supp}.

\begin{table}[htbp]
\centering
\caption{Summary of Notation}\label{tab:notation-supp}
\renewcommand{\arraystretch}{1.15}
\setlength{\tabcolsep}{4pt}
\begin{tabular}{@{}l p{0.62\columnwidth}@{}}
\toprule
\textbf{Symbol} & \textbf{Meaning} \\
\midrule
$\mathcal{A}$ & Set of agents (consumers, providers, integrators, AI agents) \\
$t$ & Discrete time index, $t = 1,2,\dots$ \\
$\theta_i(t)$ & Type of agent $i$ at time $t$ \\
$\Theta_i$ & Type space for agent $i$ \\
$\ell_i$ & Physical location or region of agent $i$ \\
$B_i(t)$ & Budget or credit of agent $i$ \\
$\mathcal{T}_i(t)$ & Tasks issued or served by agent $i$ \\
$\kappa_i$ & Capabilities of agent $i$ \\
$\rho_i$ & Role of agent $i$ (consumer, provider, integrator) \\
$\phi_i(t)$ & Governance attributes (trust, compliance, policy flags) \\
\midrule
$\eta_k(t)$ & Parameters of task $k$ (quality, deadline, sensitivity) \\
$q_k$ & Quality or workload of task $k$ \\
$d_k^{\max}$ & Maximum acceptable latency (deadline) \\
$\lambda_k$ & Latency sensitivity parameter \\
$T_{ik}$ & Latency experienced by task $k$ served by agent $i$ \\
$v_{ik}(q)$ & Base value from completing task $k$ at quality $q$ \\
$\delta_{ik}(T)$ & Latency discount factor \\
$V_{ik}(T,q)$ & Latency-aware task value for agent $i$ \\
\midrule
$\mathcal{R}$ & Set of service types \\
$C_{r,\ell}(t)$ & Capacity of service $r$ at location $\ell$ \\
$G_{\mathrm{res}}$ & Service-dependency DAG (directed acyclic graph) \\
$C_v$ & Capacity of node $v$ of $G_{\mathrm{res}}$ within an epoch \\
$\bar{C}_j$ & Integrator $j$'s exported slice capacity \\
$L(G)$ & Set of leaf services in DAG $G$ \\
$f$ & Polymatroid rank function of $\mathcal{X}_{\mathrm{res}}$ \\
\midrule
$s_t$ & System state at time $t$ \\
$\Lambda(t)$ & Network latency conditions across the continuum \\
$D(t)$ & Exogenous demand or load \\
$R(t)$ & Reputation or trust states of all agents \\
$r_i(t)$ & Reputation score of agent $i$, $r_i(t) \in [0,1]$ \\
$G(t)$ & Governance/policy state at time $t$ \\
\midrule
$x_t$ & Allocation at time $t$ \\
$P_i(t)$ & Net payment by agent $i$ ($>0$: pays; $<0$: receives) \\
$\mathcal{X}_t$ & Feasible allocation set \\
$\mathcal{X}_{\mathrm{res}}$ & Service- and DAG-based feasibility constraints \\
$\mathcal{X}_{\mathrm{gov}}$ & Governance feasibility constraints \\
\midrule
$m_i(t)$ & Message sent by agent $i$ \\
$\mathcal{M}_i$ & Message space for agent $i$ ($= \Theta_i$ under direct mechanisms) \\
$\mathcal{M}$ & Mechanism mapping messages to $(x_t, P(t))$ \\
\midrule
$W(t)$ & Social welfare at time $t$ \\
$\mu$ & Externality trade-off parameter ($\mu \ge 0$) \\
$u_i(t)$ & Per-period utility of agent $i$ \\
$\gamma_i$ & Risk-aversion coefficient \\
$\mathrm{cost}_j(x_t)$ & Operational cost of provider $j$ under allocation $x_t$ \\
$\sigma_p$ & Price volatility: coefficient of variation (CV), sd/mean of the agent-facing per-task price induced on each leaf \\
\bottomrule
\end{tabular}
\end{table}

\subsection{Polymatroid}

\begin{definition}[Polymatroid]
\label{def:polymatroid}
Let $E$ be a finite ground set and $\rho\colon 2^E \to \mathbb{R}_{\ge 0}$ a set function that is (i)~\emph{normalised}: $\rho(\emptyset) = 0$; (ii)~\emph{monotone}: $S \subseteq T \Rightarrow \rho(S) \le \rho(T)$; and (iii)~\emph{submodular}: $\rho(S \cup \{e\}) - \rho(S) \ge \rho(T \cup \{e\}) - \rho(T)$ for all $S \subseteq T$ and $e \notin T$. The \emph{polymatroid} associated with~$\rho$ is the polytope
\[
P(\rho) = \bigl\{x \in \mathbb{R}^{E}_{\ge 0} : x(S) \le \rho(S) \;\text{for all } S \subseteq E\bigr\},
\]
where $x(S) = \sum_{e \in S} x_e$.  The function~$\rho$ is called the \emph{rank function} of the polymatroid~\cite{edmonds1970submodular,fujishige2005submodular}.
\end{definition}

\emph{Intuition.}
A polymatroid captures a system of capacity constraints with diminishing returns: the more resources already committed within a subset, the less additional throughput any new resource can contribute.  In the context of this paper, the ground set~$E$ corresponds to the leaf services consumed by agents, and the rank function~$\rho$ encodes the bottleneck capacities imposed by internal nodes of the service-dependency DAG.  This structure is important for two reasons: (i)~Edmonds'~\cite{edmonds1970submodular} greedy algorithm maximises any linear objective with non-negative weights (including social welfare) over a polymatroid in polynomial time, separable concave objectives being maximised over a polymatroid by the algorithms of Groenevelt~\cite{groenevelt1991two}, and (ii)~polymatroidal feasible sets are compatible with gross-substitutes valuations, jointly guaranteeing market equilibrium existence.

\subsection{Walrasian (Competitive) Equilibrium}

\begin{definition}[Walrasian Equilibrium]
\label{def:walrasian}
An allocation--price pair $(x^{*},\, p^{*})$ is a \emph{Walrasian} (or \emph{competitive}) \emph{equilibrium} if (i)~every agent maximises its utility at prices~$p^{*}$, and (ii)~the aggregate demand is a profit-maximising point of the supply set at $p^{*}$ (market clearing; with free disposal, any unused capacity carries price zero).
\end{definition}

\emph{Intuition.}
A Walrasian equilibrium is a vector of per-resource prices at which every agent is individually satisfied with its allocation and no resource is over- or under-demanded.  Its existence implies that fully decentralised coordination via prices alone can replicate the outcome of a benevolent central planner.  The Kelso--Crawford~\cite{kelso1982job} fixed-point argument guarantees existence under gross substitutes when supply is a box of independently available items; the polymatroid supply of this paper is covered by the discrete-convex-analysis route of Proposition~2(i).

\subsection{Gross Substitutes and Dominant-Strategy Incentive Compatibility}

\begin{definition}[Gross-Substitutes (GS) Condition]
\label{def:gs}
An agent's valuation function satisfies the \emph{gross-substitutes} condition if, whenever the price of one item increases (with all other prices unchanged), the agent's demand for every other item whose price did not change weakly increases~\cite{gul1999grosssubstitutes}.
\end{definition}

\begin{definition}[Dominant-Strategy Incentive Compatibility (DSIC)]
\label{def:dsic}
A mechanism is \emph{dominant-strategy incentive compatible} (DSIC) if truthful reporting of valuations maximises each agent's utility regardless of what other agents report~\cite{ausubel2005vickrey}.
\end{definition}

\emph{Intuition.}
The GS condition captures the idea that resources substitute for one another: raising the price of one service slice moves an agent towards alternatives and keeps it in the market.  This rules out the ``I need both or neither'' complementarity pattern that makes combinatorial auctions intractable.  When GS holds, the classical results~\cite{gul1999grosssubstitutes,kelso1982job} guarantee the existence of a Walrasian equilibrium for box supply, and on a polymatroidal feasible set existence follows by the discrete-convex-analysis route of Proposition~2(i)~\cite{murota2003discrete}, and the Vickrey--Clarke--Groves (VCG) mechanism~\cite{vickrey1961counterspeculation,clarke1971multipart,groves1973incentives} implements the welfare-maximising allocation in dominant strategies (DSIC), which rests on exactness of the allocation rule and needs no GS hypothesis: no agent can gain by misreporting, making the market robust to strategic behaviour.

\section{Proof Details}
\label{app:proofs}

\subsection{Proof of Proposition~1 (Polymatroidal Structure under Laminar Leaf Blocks)}

\begin{proposition}[Polymatroidal Structure under Laminar Leaf Blocks]
\label{prop:polymatroid-supp}
Let $G_{\mathrm{res}}$ have node capacities $C_v \ge 0$ and a laminar leaf-block family $\mathcal{L} = \{L_v\}$ under the leaf-block semantics of Definition~1 in \mainname. A rooted tree and a two-terminal series--parallel network each satisfy that hypothesis. Then:
\begin{enumerate}
\item[\emph{(i)}] the capacity function $f$ defined in Definition~1 of \mainname{} is submodular, monotone, and normalized ($f(\emptyset) = 0$);
\item[\emph{(ii)}] $\mathcal{X}_{\mathrm{res}}$ is a polymatroid with rank function $f$; and
\item[\emph{(iii)}] the base polytope $B(f) = \{x \in \mathcal{X}_{\mathrm{res}} : x(L(G)) = f(L(G))\}$ is the set of tight allocations (i.e., allocations that saturate total leaf-level throughput).
\end{enumerate}
\end{proposition}

\begin{proof}
For tree-structured DAGs, the constraint sets $\{L_v\}$ form a laminar family by the nesting property of subtrees, and $f$ is the rank function of \cref{def:polymatroid}. Edmonds~\cite{edmonds1970submodular} established that the polytope $\{x \ge 0 : x(L_v) \le C_v\}$ defined by a laminar family is a polymatroid. Submodularity of~$f$ follows once $f$ is identified as that polytope's rank function, which the min-cover formula supplies: for a cover $A$ of~$S$ by pairwise-disjoint blocks, $x(S) \le \sum_{v \in A} x(L_v) \le \sum_{v \in A} C_v$, and $A = \{v\}$ gives $f(L_v) \le C_v$, so the two descriptions define the same set and the rank function of a polymatroid is submodular~\cite{edmonds1970submodular,fujishige2005submodular}. Monotonicity is immediate from non-negative capacities and $f(\emptyset) = 0$ is trivial; the argument uses only $C_v \ge 0$, so a service node of capacity zero may be kept, which admits a contracted integrator node exporting nothing in \cref{prop:encapsulation-supp}.

For series--parallel (SP) networks we use the standard recursive definition of two-terminal SP graphs (see, e.g., \cite{fujishige2005submodular}): an SP graph is built by series and parallel composition from the primitive component $s \to l \to t$, a single service node between two virtual terminals, and the terminals of every component are virtual. We show that the leaf-block family $\mathcal{L} = \{L_v : v \text{ a service node}\}$ is laminar, inducting on the number of compositions.

\emph{Base.} The primitive component is $s \to l \to t$ with the single service node~$l$, which is a leaf. $\mathcal{L} = \{\{l\}\}$ is laminar. A service DAG presented without a decomposition is placed under the convention by prefixing a virtual source, appending a virtual terminal, and reading each remaining node as a service node; inserting or suppressing a routing placeholder of in-degree and out-degree one changes no leaf block and hence no constraint, and a component containing no service node contributes none, so the induction covers these presentations unchanged.

\emph{Parallel.} Let $G = G_1 \parallel G_2$ identify the two virtual sources with each other and the two virtual terminals with each other. Virtual nodes host no service, so $G_1$ and $G_2$ share no service node, their leaf sets $L_1, L_2$ are disjoint, and every service node keeps the leaf block it had in its own component. Hence $\mathcal{L} = \mathcal{L}_1 \cup \mathcal{L}_2$, where each $\mathcal{L}_i$ is laminar by induction and every member of $\mathcal{L}_1$ is disjoint from every member of $\mathcal{L}_2$; $\mathcal{L}$ is laminar.

\emph{Series.} Let $G = G_1 \cdot G_2$ identify the virtual terminal of $G_1$ with the virtual source of $G_2$. The identified node is virtual in both components, so it stays a routing placeholder and contributes no constraint. A leaf of $G_1$ now has the service nodes of $G_2$ downstream and is therefore no longer a leaf, so $L(G) = L(G_2)$. Every service node of a two-terminal series--parallel component lies on a source--terminal path (immediate by the same induction), so every service node of $G_1$ reaches the identified node and hence every leaf of $G_2$: $L_v = L(G_2) = L(G)$ for every service node $v$ of $G_1$. The members contributed by $G_2$ are laminar by induction and are subsets of $L(G)$, so $\mathcal{L} = \mathcal{L}_2 \cup \{L(G)\}$ is laminar.

Since $\mathcal{L}$ is laminar in both cases, the rank function $f$ is submodular by the standard laminar-family result~\cite{edmonds1970submodular,fujishige2005submodular}, and the composed feasible region is a polymatroid. Parts~(ii) and~(iii) are standard consequences of the polymatroid definition~\cite{fujishige2005submodular}.
\end{proof}

\begin{remark}[Leaf blocks are exact, and how they read against a routed composition]
\label{rem:leafblock-conservative}
For a rooted tree each leaf is reached by a unique ancestor path and the leaf-block region coincides with the node-capacitated flow region. Composing a parallel component in series with a downstream component separates the two, and under the throughput-token counting of Definition~1 in \mainname{} it is the leaf-block region that is exact: two parallel service nodes $v_1, v_2$ feeding the same downstream services satisfy $L_{v_1} = L_{v_2}$, and every token bound for those services loads both of them, so $x(L_{v_1}) \le \min(C_{v_1}, C_{v_2})$ is the capacity constraint itself. A routed composition would deliver the larger bound $x(L_{v_1}) \le C_{v_1} + C_{v_2}$, one in which each unit takes a single branch (with $C_{v_1} = C_{v_2} = 1$ and a downstream service of capacity $2$, the leaf-block region gives $x \le 1$ where a routed composition would give $2$), and the two readings coincide on rooted trees and on series chains. The exact total a sub-DAG can export is given, under either reading, by the projection of \cref{lem:export-projection} under its homogeneity hypothesis.
\end{remark}

\subsection{Proof of Lemma~1 (Agent Valuations under Slice Encapsulation)}

\begin{lemma}[Agent Valuations under Slice Encapsulation]
\label{lem:gs-valuations-supp}
Suppose integrators encapsulate multi-resource service paths into composite \emph{slices}, each a discrete, indivisible unit with fixed internal routing, deterministic latency $T_s$, and quality~$q_s$, and that (a)~each task requires at most one slice (unit demand); (b)~an agent's tasks are \emph{additively separable} (no cross-task complementarity, shared budget, or feasibility coupling beyond $\mathcal{X}_{\mathrm{res}}$); and (c)~slice attributes $(T_s, q_s)$ are fixed within each epoch by the state~$s_t$. Then each task's valuation is unit-demand and satisfies GS by Gul--Stacchetti~\cite[Sec.~2 and Thm.~1]{gul1999grosssubstitutes}, and agent~$i$'s aggregate valuation over slice bundles is their convolution, an assignment (OXS) valuation~\cite{lehmann2006combinatorial} that satisfies GS as well and is M$^{\natural}$-concave~\cite{murota2003discrete}. GS is closed under convolution, though not under summation across agents, so an agent's tasks form one gross-substitutes bidder while distinct agents do not.
\end{lemma}

\begin{proof}
Under encapsulation, agent~$i$'s valuation for assigning task~$k$ to slice~$s$ is $V_{ik}(T_s, q_s)$, where $(T_s, q_s)$ are fixed within the current epoch by condition~(c). Since each task requires exactly one slice, the bidder representing task $(i,k)$ selects whichever slice maximises $V_{ik}(T_s, q_s) - p_s$, where $p_s$ is the slice price. This defines a unit-demand valuation over slices. Gul and Stacchetti show that every unit-demand utility function satisfies their strong-no-complementarities condition, which implies no complementarities and hence, by their Theorem~1, gross substitutes in the sense of \cref{def:gs}~\cite{gul1999grosssubstitutes}; so unit-demand valuations satisfy GS: raising the price of one slice causes the bidder either to retain it or to switch to another, but never to drop demand for an unrelated slice whose price has not changed.

By additive separability (condition~(b)) and the absence of shared agent-side budgets or cross-task feasibility coupling, agent~$i$'s value for a bundle $z \in \mathbb{Z}_{\ge 0}^K$ of slice units is $V_i(z) = \max \textstyle\sum_{k \in \mathcal{T}_i(t)} V_{ik}(T_{s_k}, q_{s_k})$, the maximum taken over the assignments that give each task~$k$ at most one slice type~$s_k$ and use at most $z_s$ units of each type~$s$, an unserved task contributing zero. This is the \emph{convolution} of the $|\mathcal{T}_i(t)|$ unit-demand valuations and differs from their pointwise sum. A convolution of unit-demand valuations is an assignment valuation, called OXS (OR-of-XOR-of-singletons) by Lehmann, Lehmann and Nisan~\cite{lehmann2006combinatorial}, and assignment valuations satisfy GS. Each task's unit-demand valuation is M$^{\natural}$-concave on $\mathbb{Z}_{\ge 0}^K$ and M$^{\natural}$-concavity is preserved under convolution~\cite{murota2003discrete}, so $V_i$ is M$^{\natural}$-concave on multi-unit bundles; for $0$--$1$ bundles this is the Fujishige--Yang equivalence between gross substitutes and M$^{\natural}$-concavity~\cite{fujishige2003note}. This is the form the mechanism-design results of Proposition~2 use. The level matters: GS is closed under convolution, which is how an agent's tasks aggregate, but not under summation of distinct bidders' valuations into a single set-valuation, which is why the market keeps the agents separate.

Shared agent-side budgets, per-agent latency coupling, or shared token caps across tasks would break the additive-separability premise of condition~(b) and, with it, the convolution structure. Slice capacities (the number of available units of each type) are determined by the polymatroid $\mathcal{X}_{\mathrm{res}}$; this is a feasibility condition relevant to Proposition~2, and GS itself does not require it.
\end{proof}

\subsection{Proof Sketch of Proposition~2 (Efficient Mechanism Design under Structured DAGs)}

\emph{Note.} The following proof sketch reduces each part to standard results from auction theory and discrete convex analysis; the underlying theorems are cited explicitly at each step.

\begin{proposition}[Efficient Mechanism Design under Structured DAGs]
Suppose $\mathcal{X}_t$ is polymatroidal (\cref{prop:polymatroid-supp}); agent valuations are as in \cref{lem:gs-valuations-supp} (its additive-separability hypothesis~(b) included); agents have quasilinear utility; and slices are discrete items with integral capacities. Then:
\begin{enumerate}
\item[\emph{(i)}] a Walrasian equilibrium exists with per-type prices supporting a welfare-maximising allocation, by the discrete-convex-analysis route for M$^{\natural}$-concave valuations (\cref{lem:gs-valuations-supp}) on M$^{\natural}$-convex polymatroid supply~\cite{fujishige2003note,murota2003discrete,fujishige2005submodular}, which specialise to the classical box-supply results~\cite{kelso1982job,gul1999grosssubstitutes};
\item[\emph{(ii)}] the exact welfare maximiser is computable by the M$^{\natural}$-convex intersection algorithm of discrete convex analysis~\cite{murota2003discrete}, polynomial in the agent count~$n$, the slice-type count~$K$ and $\log V_{\max}$, the largest reported task value, given the per-agent value oracles of \cref{lem:gs-valuations-supp}, the aggregate oracle being obtained by successive convolution of the per-agent oracles. The salary-adjustment auction of Kelso and Crawford~\cite{kelso1982job} covers the box-supply special case, and its output at a finite price granularity is not the object part~\emph{(iii)} applies to; and
\item[\emph{(iii)}] the efficient allocation, computed as the exact maximiser of reported welfare, is implementable in DSIC \emph{at the agent level} by the VCG mechanism whose Clarke pivot removes an agent in full and recomputes welfare without it~\cite{ausubel2005vickrey}: truthful reporting of an agent's entire type across all of its tasks is a dominant strategy. Truthfulness needs no hypothesis on the valuation class. VCG is individually rational, and every Clarke-pivot payment is nonnegative under down-closed supply, so the mechanism is weakly budget balanced. The truthfulness guarantee does not extend to coalitions across agents.
\end{enumerate}
\end{proposition}

\begin{proof}
Part~(i): Under slice encapsulation (Lemma~1), agents treat each unit of each slice type as an indivisible item. Each agent has unit demand per task (one slice per task). The supply of~$K$ slice types is governed by the integer polymatroid $P_{\mathbb{Z}}(f) = \{x \in \mathbb{Z}_{\ge 0}^K : x(S) \le f(S) \text{ for all } S\}$, where $f$ is the rank function from Proposition~1. This imposes not only per-type upper bounds $x_s \le f(\{s\})$ but also cross-type bottleneck constraints $x(S) \le f(S)$ for subsets~$S$ of types sharing internal DAG resources (integral node capacities yield an integral rank function, ensuring integer vertices~\cite{fujishige2005submodular}). With per-type prices~$p_s$ and quasilinear payoffs, each task's valuation is unit-demand and satisfies the GS condition~\cite{gul1999grosssubstitutes}, and each agent's aggregate valuation is their convolution and satisfies it as well (Lemma~1). Each agent's aggregate valuation is M$^{\natural}$-concave on multi-unit bundles (Lemma~1), so the demand side is M$^{\natural}$-concave at the agent level, and the aggregate demand $V(z) = \max\{\sum_i V_i(z_i) : \sum_i z_i = z\}$ is again M$^{\natural}$-concave by closure under convolution~\cite{murota2003discrete}; the integer polymatroid supply set is down-closed and M$^{\natural}$-convex~\cite{murota2003discrete}. Existence of a Walrasian equilibrium (\cref{def:walrasian}) then follows from discrete convex analysis: for M$^{\natural}$-concave demand over an M$^{\natural}$-convex supply set, the M$^{\natural}$-convex intersection theorem (equivalently, discrete Fenchel duality)~\cite{murota2003discrete} yields integral primal and dual optima, a welfare-maximising integer allocation supported by per-type prices. The classical existence theorems of Kelso and Crawford~\cite{kelso1982job} and of Gul and Stacchetti~\cite{gul1999grosssubstitutes} take supply to be a box $\{x \le c\}$ of independently available items; they are the special case in which no cross-type bottleneck binds, and the discrete-convex-analysis route covers the polymatroid supply of this paper.

Part~(ii): Under GS valuations, the Walrasian equilibrium of Part~(i) is welfare-maximising~\cite{gul1999grosssubstitutes}. The exact maximiser is obtained by the M$^{\natural}$-convex intersection algorithm of discrete convex analysis~\cite{murota2003discrete}, which is polynomial in $n$, $K$ and $\log V_{\max}$ given the per-agent value oracles of Lemma~1. The salary-adjustment auction of Kelso and Crawford~\cite{kelso1982job} covers the box-supply special case, and its output at a finite price granularity is not the object Part~(iii) is applied to.

Part~(iii): Within the static-epoch setting of Parts~(i)--(ii), take the allocation rule to be the exact maximiser of reported welfare, and let each agent's payment be its Clarke pivot, computed by removing that agent in full, with all of its tasks, and recomputing the welfare-maximising allocation on the remaining agents. With quasilinear utilities this is the VCG mechanism with agents as the reporting parties, so truthful reporting of an agent's entire type is a dominant strategy, which is DSIC in the sense of \cref{def:dsic}, and the mechanism is individually rational~\cite{ausubel2005vickrey}. The argument quantifies over an agent's whole report and so covers joint deviations across its own tasks; beyond quasilinearity it needs no hypothesis on the valuation class. (This is a per-epoch result; it carries no dynamic repeated-game guarantee and does not extend to coalitions across distinct agents.) Gross substitutes contributes what Parts~(i)--(ii) use: the exact welfare maximiser exists as an integral allocation, is computable, and is supported by per-type prices~\cite{gul1999grosssubstitutes}; truthfulness itself does not rest on it. Exactness of the allocation rule, by contrast, is essential: pairing VCG payments with an allocation that is only Walrasian, or with an approximate maximiser, breaks the dominant-strategy property. Weak budget balance also holds: restricting the optimal allocation with agent~$i$ present to the other agents leaves a feasible allocation of the problem without~$i$, because the supply set is down-closed, so the others' optimum without~$i$ is at least their welfare with~$i$ present; agent~$i$'s Clarke-pivot payment is the difference between the two and is therefore nonnegative, and summing over agents gives $\sum_i P_i \ge 0$, which is weak budget balance as \mainname{} defines it. The simulation study of this paper charges this same full-agent pivot, computing an agent's payment by removing all of its tasks and re-solving, so the mechanism evaluated empirically is the agent-level mechanism stated here; what discharges the exactness hypothesis on the reported arms is \cref{prop:polymatroid-supp}, since on a laminar leaf-block region the value-greedy allocation rule is the exact maximiser over that region, and the boundary of the discharge is measured: on the crossing instance X the greedy packer falls short of the round's exact leaf-block optimum in $14.7$ [$12.4$, $17.0$] per cent of rounds at high load, and where it does the hypothesis is not met and the guarantee fails with it.

\end{proof}

\begin{remark}[Clinching under one task per agent]
\label{rem:clinching}
When each agent holds one task per epoch and that task admits a unique admissible slice type, the agent is a single-parameter bidder: its private information reduces to the scalar willingness-to-pay $V_{ik}(T_{s(k)}, q_{s(k)})$ for its designated type, and the type-level polymatroid pushes forward to a polymatroid in agent coordinates, with rank $g(T) = f(\{s(i) : i \in T\})$, which is the form the clinching auction takes as input. The polymatroid clinching auction of Goel et al.~\cite{goel2015polyhedral}, which extends the ascending clinching auction of Ausubel~\cite{ausubel2004ascending} to polymatroidal environments (their result is stated for budget-constrained bidders; the unbudgeted case of this paper is the specialisation in which the Pareto-optimal outcome is the efficient one), is then an alternative DSIC implementation in which prices ascend from zero and each unit is charged the price at which it clinches, so payments are determined along an ascending price path. With several tasks per agent the construction is single-parameter only per task, and truthfulness at that pseudo-bidder level does not imply truthfulness at the agent level, so the clinching implementation is claimed only in the one-task case.
\end{remark}

The single-type-per-task condition holds naturally in the proposed architecture: Proposition~3 assumes each integrator exports a single homogeneous slice type, and when each task class is served by exactly one integrator, each task's demand is one-dimensional. When tasks may choose among \emph{multiple heterogeneous} slice types (different $(T_s, q_s)$), a scalar willingness-to-pay $\max_s V_{ik}(T_s, q_s)$ does not faithfully encode item-specific preferences, and the mechanism may allocate a non-utility-maximising slice, breaking the single-parameter equivalence. In this general heterogeneous case, VCG remains the appropriate DSIC mechanism.

\subsection{Proof of Proposition~3 (Exact Scalar Interface Where Leaf Blocks Coincide)}

\begin{proposition}[Exact Scalar Interface Where Leaf Blocks Coincide]
\label{prop:encapsulation-supp}
Let $G_{\mathrm{res}}$ be an arbitrary service-dependency DAG. Suppose integrators partition the non-leaf nodes into disjoint clusters, each managing a connected sub-DAG internally. Assume:
\begin{enumerate}
\item[\emph{(i)}] each integrator~$j$ exposes a single composite service (slice) with scalar capacity $\bar{C}_j = \min_{v \in P_j} C_v$, the least capacity among the internal service nodes $P_j$ that carry the export, each of which every admitted external unit loads exactly once;
\item[\emph{(ii)}] the units integrator~$j$ exports are interchangeable in its internal schedule: any internally feasible total throughput can be split arbitrarily across the coordinates it exposes, so the internal schedule constrains an external demand vector only through its total (a single homogeneous slice type is the leading case);
\item[\emph{(iii)}] no external feasibility constraints couple different integrators beyond the shared quotient-graph structure, and in particular agent-facing feasibility is exhausted by the quotient graph's leaf-block constraints and the coordinate-wise governance bounds of the preliminaries in \mainname; and
\item[\emph{(iv)}] each integrator presents a single-node interface: one entry node and one exit node, so that only a scalar throughput crosses each quotient arc.
\end{enumerate}
Let $G'$ be the quotient graph from contracting each cluster to a single node. If $G'$ has rooted-tree or series--parallel structure, the agent-facing feasible region is polymatroidal.
\end{proposition}

\begin{lemma}[Exported capacity is the projection of the internal region]
\label{lem:export-projection}
Let integrator~$j$ manage a connected sub-DAG $G_j$ with integral node capacities $C_v > 0$, entry set $I_j$ and exit set $O_j$ (the nodes of $G_j$ at which external demand enters, that is, those carrying an in-arc from outside the cluster or from the virtual source of $G_{\mathrm{res}}$, respectively those carrying an out-arc to outside or to the virtual terminal), and suppose the exported slice type is \emph{homogeneous}: writing $P_j$ for the set of internal service nodes that carry the export, each loaded exactly once by every exported unit, and the units are interchangeable, so the internal schedule constrains an external demand vector~$x$ only through its total $x(L'_j)$ over the slice set $L'_j$ that integrator~$j$ exposes. Call $x$ \emph{internally realisable} if integrator~$j$ can carry $x(L'_j)$ units of $I_j$-to-$O_j$ throughput simultaneously. Then the projection of the internally realisable set onto the agent-facing coordinates is exactly $\{x \ge 0 : x(L'_j) \le \bar{C}_j\}$ with $\bar{C}_j = \min_{v \in P_j} C_v$, the least capacity among the nodes of $P_j$ (equivalently, its integer points are exactly the integrally realisable demand vectors). Where a unit is instead carried on one of several alternative internal routes, the realisable set need not be a sublevel set of $x(L'_j)$ at all, and no scalar summarises it; that regime is outside the hypothesis and is taken up by \cref{prop:inner-exposure}.
\end{lemma}

\begin{proof}
Every admitted unit loads each node of $P_j$ exactly once, so the load at $v \in P_j$ is $x(L'_j)$ and $x$ is internally realisable exactly when $x(L'_j) \le C_v$ for every $v \in P_j$, that is, exactly when $x(L'_j) \le \min_{v \in P_j} C_v = \bar{C}_j$. Homogeneity makes the units interchangeable, so a realisable total realises any split of itself across the agent-facing coordinates, and the projection is $\{x \ge 0 : x(L'_j) \le \bar{C}_j\}$; capacities are integral by hypothesis, so the integer points of that set are exactly the integrally realisable demand vectors.
\end{proof}

The lemma makes the quotient graph's constraints leaf-block constraints in the sense of the terminals-and-leaves convention: a contracted integrator node carries the scalar $\bar{C}_j$ and the slices it exposes are the leaves below it, so Proposition~1 applies to the quotient graph under that same convention. The projection is exact for integrator~$j$ taken in isolation, its entries fed and its exits drained without constraint; Proposition~3 adds the hypothesis under which the per-integrator summaries compose. Homogeneity is the hinge: when several types are exported the internal region no longer depends on~$x$ through its total, the projection is that region itself, and \cref{app:catalogue} takes up the question of when that region is a polymatroid.

\begin{proof}
By \cref{lem:export-projection}, the sub-DAG managed by integrator~$j$ is summarised by the single scalar $\bar{C}_j$ without loss: the agent-facing constraint $x(L'_j) \le \bar{C}_j$ is exactly the projection of the internally realisable set onto the slices that integrator~$j$ exposes, so the contracted node does not over-commit its sub-DAG. Computing $\bar{C}_j$ is a minimum over the nodes of $P_j$, which the integrator recomputes in time linear in $\lvert P_j \rvert$ as internal conditions change.

Complete $G'$ with a virtual source preceding its sources and a virtual terminal following its sinks, as in the terminals-and-leaves convention of \mainname; every other node~$v$ of $G'$ is then a service node carrying a well-defined capacity: for contracted integrator nodes, the capacity is $\bar{C}_v = \bar{C}_j$ (the bottleneck of integrator~$j$'s sub-DAG; an integrator with $\bar{C}_j = 0$ exports nothing, and the laminar-family argument of Proposition~1 uses only $C_v \ge 0$, so such a node may be kept with capacity zero); for the leaves, which no integrator cluster absorbs, the capacity is the original $\bar{C}_v = C_v$. The agent-facing constraints are $x(L'_v) \le \bar{C}_v$ for each node~$v$ in~$G'$, where $L'_v$ is the set of slices in $v$'s sub-DAG. Condition~(iii) makes that list complete: agent-facing feasibility is exhausted by the quotient graph's leaf-block constraints and the coordinate-wise governance bounds, so no further external constraint couples the clusters. Since $G'$ is tree or series--parallel by assumption, the constraint family $\{L'_v\}$ is laminar and Proposition~1 applies directly. The exported throughput token is a single commodity, measured in the same units (e.g., requests per second) as the agent-facing allocation variable~$x_l$, which is the homogeneity hypothesis of \cref{lem:export-projection}. Internal optimisation within each integrator (scheduling, routing, and capacity allocation across the sub-DAG) is a separate, lower-level problem solved by platform-specific orchestration. Under the single-node interface assumption only a scalar crosses each quotient arc, so the per-integrator summaries compose and $\{x \ge 0 : x(L'_v) \le \bar{C}_v\ \text{for every node } v \text{ of } G'\}$ is the whole of the agent-facing constraint system; without it the contraction discards which exits of a cluster feed which entries of the next, and the summaries can over-commit the composed graph. An integrator that cannot present a single-node interface may instead expose any polymatroid contained in its realisable set (\cref{prop:inner-exposure}), at a welfare cost and with every market guarantee intact.
\end{proof}

\begin{remark}[Scope of the scalar summary]
\label{rem:scalar-summary}
The scalar summary $\bar{C}_j$ is valid when each integrator exports a single, homogeneous slice type so that feasible throughput does not depend on the mix of external requests. In realistic service graphs, feasible throughput may depend on request mix, route contention, and internal compatibility constraints; a single scalar is then insufficient. When an integrator exports multiple service types, $\bar{C}_j$ generalises to a vector-valued capacity, and the quotient-graph polymatroid argument requires each integrator's external feasibility set to itself be polymatroidal. \cref{app:catalogue} lifts the scalar to a catalogue of exported types and locates the boundary: the lifted interface is exact exactly when the recipes of slices sharing internal resources are identical, over-commits by a computable factor when those recipes are merely proportional, and admits no polymatroidal interface at all otherwise, the integrator then exposing a conservative region inside its feasible set. Additionally, if leaf nodes are shared across integrators (e.g., common cloud GPUs used by multiple domains), external coupling can re-enter through leaf capacities unless explicitly captured in the quotient-graph structure.
\end{remark}

\subsection{Multi-Resource Capacity: The Catalogue Interface and Its Boundary}
\label{app:catalogue}

The scalar $\bar{C}_j$ of Proposition~3 assumes a single homogeneous exported slice. We now lift it to the multi-dimensional case (compute, memory, bandwidth) and characterise \emph{exactly} when the polymatroidal guarantee survives. Instead of one scalar, integrator~$j$ publishes a finite \emph{catalogue} $\mathcal{C}_j=\{c_{j,1},\dots,c_{j,K}\}$, $c_{j,k}=(T_k,Q_k,\bar{C}_{j,k})$, where each composite slice~$k$ bundles a fixed \emph{recipe} of the $m$ internal resources, encoded as a column $A_{\cdot k}\in\mathbb{R}^{m}_{\ge 0}$ of a recipe matrix $A\in\mathbb{R}^{m\times K}_{\ge 0}$; internal capacities are $C\in\mathbb{R}^{m}_{\ge 0}$. The agent-facing object is the \emph{catalogue rank}, defined for each subset $S\subseteq[K]$ of slices by
\begin{multline*}
  f(S)=\max\Big\{\textstyle\sum_{k\in S} y_k :\\
  A y \le C,\ y\ge 0,\ y_k=0\ \ \forall k\notin S\Big\},
\end{multline*}
the best simultaneous throughput using only slices in~$S$. The interface is polymatroidal, so the chain of Proposition~1 applies, iff $f$ is submodular.

\paragraph*{When the catalogue rank is submodular} Form the \emph{column-sharing graph} on $[K]$: slices $k,\ell$ are adjacent iff their recipes share a resource ($\operatorname{supp} A_{\cdot k}\cap\operatorname{supp} A_{\cdot\ell}\ne\emptyset$). Call $A$ \emph{resource-decomposable} if within every connected component all columns are mutually proportional ($A_{\cdot k}=\lambda_k u_c$ for a common direction $u_c\ne 0$ and sizes $\lambda_k>0$, a Leontief bundle). Operationally, each group of slices that contend for a shared internal resource consumes it in \emph{fixed proportions}, the discipline real cloud SKUs already follow.

\begin{proposition}[Exactness of the catalogue interface]
\label{prop:catalogue-exact}
Let $A$ be resource-decomposable, write $B_1,\dots,B_t$ for the components of the column-sharing graph, $K_c=|B_c|$ and $\kappa_c=\min_{i:(u_c)_i>0} C_i/(u_c)_i$ for the scalar budget of component~$c$, which we take to be positive, a component with $\kappa_c=0$ having $P_{f_c}=F_c=\{0\}$ and nothing to separate, and let $F=\{y\ge 0: Ay\le C\}$ be the internally realisable set, and write $f_c$ for the restriction of the catalogue rank to $B_c$. Then
\begin{enumerate}
\item[\emph{(i)}] $f$ is normalised, monotone and submodular, so $P_f=\{x\ge 0: x(S)\le f(S)\ \forall S\}$ is a polymatroid, and $F\subseteq P_f$;
\item[\emph{(ii)}] $P_f=F$ if and only if the columns within every component are identical, that is, $\lambda_k$ is constant on each $B_c$. In that case component~$c$ is a single scalar slice of capacity $\kappa_c/\lambda_c$ offered under $K_c$ attribute labels $(T_k,Q_k)$, and the interface is Proposition~3 applied per component;
\item[\emph{(iii)}] otherwise the inclusion is strict and the advertised region over-commits the internal budget: with the sizes of component~$c$ indexed in ascending order $\lambda_{(1)}\le\dots\le\lambda_{(K_c)}$,
\[
  \max_{x\in P_{f_c}}\ \frac{1}{\kappa_c}\sum_{k\in B_c}\lambda_k x_k
  \;=\; K_c-\sum_{j=1}^{K_c-1}\frac{\lambda_{(j)}}{\lambda_{(j+1)}}\;\ge\;1,
\]
with equality to one exactly when $\lambda$ is constant on $B_c$.
\end{enumerate}
\end{proposition}
\begin{proof}
We prove~(i), then~(iii), then~(ii), whose converse uses~(iii).

\emph{(i)}~Distinct components use disjoint resources, so the defining linear program is block-diagonal and $f=\bigoplus_c f_c$. Within a component, proportionality collapses the $m$ rows to the single budget $\sum_k \lambda_k y_k\le\kappa_c$, and routing that budget to the slice of least size gives $f_c(S)=\kappa_c\max_{k\in S}(1/\lambda_k)$, a positive multiple of a maximum-of-weights function and hence submodular; a direct sum of submodular functions on disjoint ground sets is submodular. For $F\subseteq P_f$, restrict a feasible $y$ to $S$: it stays feasible and is supported in~$S$, so $y(S)\le f(S)$.

\emph{(iii)}~The objective $x\mapsto\sum_k\lambda_k x_k$ is linear with non-negative weights, so Edmonds' greedy maximises it over the polymatroid $P_{f_c}$~\cite{edmonds1970submodular}. Sorting by weight descending, greedy assigns $x_{(K_c)}=f_c(\{(K_c)\})=\kappa_c/\lambda_{(K_c)}$ and, for $j<K_c$, $x_{(j)}=f_c(\{(j),\dots,(K_c)\})-f_c(\{(j+1),\dots,(K_c)\})=\kappa_c(1/\lambda_{(j)}-1/\lambda_{(j+1)})$. Summing $\lambda_{(j)}x_{(j)}$ telescopes to $\kappa_c(K_c-\sum_{j<K_c}\lambda_{(j)}/\lambda_{(j+1)})$. Each ratio lies in $(0,1]$, so the bracket is at least one, with equality if and only if every ratio equals one.

\emph{(ii)}~If $\lambda_k\equiv\lambda_c$ on $B_c$ then $f_c(S)=\kappa_c/\lambda_c$ for every non-empty $S\subseteq B_c$, so the only binding constraint of $P_{f_c}$ is $x(B_c)\le\kappa_c/\lambda_c$, which is $F_c$. Since $f=\bigoplus_c f_c$ and distinct components use disjoint resources, $P_f=\prod_c P_{f_c}$ and $F=\prod_c F_c$, so the componentwise equality lifts. Conversely, if two sizes in a component differ then the bracket in~(iii) exceeds one, so the greedy maximiser lies in $P_{f_c}\setminus F_c$.
\end{proof}

\begin{remark}[Resource units]
\label{rem:resource-units}
In resource units $z_k=\lambda_k y_k$ the realisable set of a component is the uniform polymatroid $\{z\ge 0: z(B_c)\le\kappa_c\}$, so the supply side is exact under resource-decomposability alone; exactness fails in slice units precisely because a non-constant diagonal rescaling of that uniform polymatroid is not a polymatroid. The demand side does not follow: a task wanting one unit of slice~$k$ wants $\lambda_k$ tokens all or nothing, so unit demand and gross substitutes are lost and welfare maximisation becomes a knapsack problem; divisible slices would restore the chain in continuous form with one price per component, whereas the tasks of this article are indivisible and deadline-bound.
\end{remark}

\paragraph*{The condition cannot be dropped} Resource-decomposability is sufficient without being necessary, and it cannot simply be dropped: some structural condition on how slices share internal resources is required. With $m=2$ resources, $K=3$ slices,
\[
  A=\begin{bmatrix}1&0&1\\0&1&1\end{bmatrix},\qquad C=(1,1)
\]
(slices $1,2$ private, slice $3$ sharing both non-proportionally) one computes $f(\{3\})=f(\{1,3\})=f(\{2,3\})=1$ but $f(\{1,2,3\})=2$, so the marginal of slice~$2$ \emph{rises} from $0$ to $1$ as the conditioning set grows from $\{3\}$ to $\{1,3\}$: $f$ is not submodular and $P_f$ is not a polymatroid. Non-proportional sharing of a multi-dimensional bottleneck is thus the obstruction this instance exhibits, and it is minimal in size: a single resource ($m=1$) gives a maximum-of-weights rank, and at most two slices ($K\le 2$) reduces submodularity to subadditivity of $f$, which restriction supplies; both are always submodular.

\begin{example}[Unequal sizes: the advertised region is not deliverable]
\label{ex:unequal-recipe}
One internal resource of capacity $\kappa$ and two slices with $\lambda_1=1$, $\lambda_2=2$. The catalogue rank gives $x_1\le\kappa$, $x_2\le\kappa/2$ and $x_1+x_2\le\kappa$, which the point $(\kappa/2,\kappa/2)$ satisfies while consuming $1.5\kappa$ of the resource. The factor of \cref{prop:catalogue-exact}(iii) is here $2-\tfrac12=\tfrac32$.
\end{example}

\begin{proposition}[Inner exposure is always available]
\label{prop:inner-exposure}
For any $A$ and $C$ and any integral $b\ge 0$ with $Ab\le C$ (replacing $b$ by $\lfloor b\rfloor$ is without loss, as $A\ge 0$), the box $\{y:0\le y\le b\}$ is a polymatroid contained in $F$. Exposing any polymatroid $P\subseteq F$ with an integer-valued rank function as the agent-facing region keeps every admitted allocation deliverable, supplies the polymatroidal-supply hypothesis that Proposition~1 establishes for tree and series--parallel graphs, and keeps Lemma~1 and Proposition~2 in force on the exposed catalogue, part~(ii) given an evaluation oracle for $P$'s rank, which a reservation box supplies trivially, for any internal topology and any resource dimension. The cost is welfare, measured against the maximum over $F$; for a fixed $P$ it is zero for every non-negative welfare weighting if and only if $P=F$ as polytopes (over integral allocations a strictly smaller $P$ can still be lossless when the integer points coincide), which requires $F$ itself to be a polymatroid and hence, since $f$ is by definition the rank of $F$, that $F=P_f$; under resource-decomposability \cref{prop:catalogue-exact}(ii) makes this exactly the case of identical recipes within each component.
\end{proposition}
\begin{proof}
A box is the polymatroid of the modular rank $S\mapsto\sum_{k\in S}b_k$, and $Ab\le C$ with $A\ge 0$ gives $Ay\le Ab\le C$ for every $y\le b$, so the box lies in $F$. Both $F$ and $P$ are down-closed polyhedra, hence intersections of the non-negative orthant with the half-spaces $w\cdot y\le h(w)$ over $w\ge 0$, with $h$ the respective support function. If $P$ is a proper subset of $F$ then some $w\ge 0$ separates a point of $F$ from $P$, and that $w$ is a welfare weighting on which the exposed region loses; if $P=F$ no weighting does. Only down-closedness of $F$ is used, so the argument covers every internally realisable set and is not confined to $\{y\ge 0: Ay\le C\}$.
\end{proof}

\paragraph*{Reduction: supply-side structure is the whole question} Under the model's single-slice-per-task demand, each task's valuation over the catalogue is unit-demand, so an agent's aggregate valuation is their convolution and satisfies gross substitutes (Lemma~1), and VCG is efficient and DSIC given any exact welfare-maximising oracle; neither needs a further hypothesis. The multi-dimensionality question therefore reduces entirely to the supply side: to which polymatroid inside $F$ the integrator exposes. When $f$ is \emph{modular} (resource-disjoint recipes, the cleanest decomposable case) the $K$ slices price as independent goods and the exposed region is exact slice by slice. When the recipes within a contending group are proportional but unequal, $P_f$ is still a polymatroid, so a single combined market with per-slice Walrasian prices replaces $K$ independent auctions, but by \cref{prop:catalogue-exact}(iii) that region over-commits the internal budget; rescaling to resource units removes the over-commitment on the supply side but not on the demand side (\cref{rem:resource-units}), and the integrator may then expose a polymatroid inside $F$ (\cref{prop:inner-exposure}), at a welfare cost.

\paragraph*{Internal tractability and graceful degradation} The lift also bounds the integrator's own burden beyond the agent-facing interface, and leaves the internal problem tractable. \emph{First}, the decomposability condition that makes $f$ submodular is the same condition that makes the internal schedule trivial: within a component the Leontief collapse reduces routing to a single scalar bottleneck, so realising the advertised throughput is a one-dimensional allocation, and no multi-resource optimisation arises. The interface and the internal problem are easy under one and the same condition. \emph{Second}, and without that condition, market correctness never requires the integrator to schedule optimally: by \cref{prop:inner-exposure} it need only expose some polymatroid contained in its true feasible region $F$ and deliver allocations inside it, which a conservative reservation box always permits, so gross substitutes, efficiency on the exposed region and DSIC hold for any internal topology or resource dimension. Delivering a committed allocation $y^{\star}\in F$ is a feasible-flow computation (polynomial) in the flow model of Proposition~3, and more generally the integrator exposes only a region it already knows how to schedule; only \emph{maximising} realised throughput, by choosing a welfare-best exposable polymatroid for the realised demand and the best internal schedule, can be hard, and that difficulty is paid in reduced throughput, never in lost market guarantees. When the recipes within each component are identical the gap closes and the exact catalogue is exposed (\cref{prop:catalogue-exact}(ii)). \emph{Finally}, encapsulation composes: an integrator's leaves may themselves be integrators (\cref{prop:encapsulation-supp} applied to the quotient graph), so per-integrator size is a design knob, one integrator per natural domain boundary, trading encapsulation depth against per-node simplicity. The complexity an integrator absorbs is therefore paid in efficiency, leaving tractability and correctness intact.

\paragraph*{Local marketplaces and layered governance in the urban-monitoring scenario}
To fulfil their slice commitments, the integrators draw on two local marketplaces. \emph{Local Market A} (city-infrastructure domain) coordinates fungible edge services and resources (GPU cycles at street-level compute nodes, bandwidth to camera feeds, frame-extraction endpoints) via lightweight auctions and enforces municipal data-locality policies. \emph{Local Market B} (cloud-provider domain) coordinates fungible cloud services and resources (model-hosting endpoints, GPU instances for inference, result-caching storage) where multiple integrators and other workloads bid for capacity under the provider's terms. Governance is handled at the appropriate layer: Integrator 1 enforces municipal data-locality; Integrator 2 enforces cross-domain data-handling compliance; Local Market A enforces city-domain admission control; Local Market B enforces cloud-provider terms. No single layer requires global knowledge of the entire system.

\paragraph*{Operational coordination signals for the urban-monitoring scenario}
Operational coordination between the integrators and local marketplaces of the illustrative example in \mainname{} relies on a small set of coarse-grained signals. At the inter-market level, Market~A broadcasts an edge-GPU congestion indicator to Market~B so that the cloud domain can pre-scale inference-endpoint capacity for incoming feature vectors; Market~B reports model-endpoint queue depth back, enabling Market~A's local price adjustment to anticipate downstream conditions beyond its local state. At the market--integrator level, Market~A reports aggregate feature-extraction demand to Integrator~1, which responds with its current slice price and residual capacity. Before admitting a cross-domain slice, Integrator~2 verifies that Integrator~1's output (feature vectors via Market~A) and cloud capacity (at Market~B) are both available; this is a flow dependency through the quotient-graph series connection, and no shared-resource coupling between integrators arises. When forwarding a slice request, Integrator~2 attaches trust scores and data-handling policy tags from the city domain so that Market~B can enforce its own policies without direct visibility into Market~A's operations. Each of these signals is a scalar or small vector, e.g., a utilisation ratio, a price, a capacity count, or a trust score. The signal layer is part of the design and is characterised here only in prose: the evaluation does not run it, and no reported number depends on it.

\section{Related Work Comparison}
\label{app:related-work}

\cref{tab:related-work-comparison-supp} positions this article relative to prior work across five research areas, identifying the specific gaps each area leaves and how the present framework addresses them. This section first surveys the five lineages the table compares.

\subsection{Lineages}
\label{app:related-work-lineages}

\paragraph*{Edge, fog and cloud platforms} Edge, fog and cloud computing have been extensively studied as platforms for low-latency, resource-aware service delivery~\cite{meuser2024edgeai}. Prior work investigates workload placement~\cite{hu2023intelligent,liu2026data}, mobility-aware orchestration~\cite{tang2025collaborative} and capacity management~\cite{li2025tiered} within edge-cloud collaborative networks, highlighting the challenges of heterogeneous resources, dynamic network conditions and real-time QoS requirements~\cite{taleb2025survey}. That literature formulates edge-cloud coordination as centralized system-level optimization~\cite{xie2025coedge,shang2024joint}, without addressing decentralized autonomous agent interaction or market-mediated coordination.

\paragraph*{Service composition and service function chaining} Service composition has a long engineering heritage in which a composite service is specified as an explicit orchestration or choreography and executed by an engine that enforces the specified control flow~\cite{oasis2007bpel}, and a research heritage of automated and QoS-aware composition over that substrate~\cite{jatoth2017survey}. Both fix the control flow at design time and leave the allocation of the underlying capacity outside the composition language. Service function chaining and distributed execution paradigms formalize the inter-dependencies among computational tasks and the implications for placement and performance optimization~\cite{li2020service,fan2023drl}. The resulting service graphs couple compute, data and communication resources, complicating scheduling, latency management and capacity allocation, and prior work models these dependencies for performance optimization~\cite{oskoui2026distributed}.

\paragraph*{Autonomous and agentic systems} Autonomous and agentic behaviours in distributed systems, in which AI agents initiate tasks, adapt preferences and negotiate resources, are studied in~\cite{park2023generative,deng2025agenticservicescomputing}, and Sedlak et al.~\cite{sedlak2026service} envision an autonomous service-orchestration framework for the computing continuum grounded in active inference within a multi-agent coordination context.

\paragraph*{Governance, trust and norms} Governance constraints including privacy, locality and compliance requirements have been studied in multi-agent systems and distributed computing~\cite{fornara2008institutions}. Trust and reputation mechanisms are widely used to evaluate service providers, enforce reliability and mediate access~\cite{sun2019trust,huang2022trust}.

\paragraph*{Market-based resource allocation} Pricing, auctions and credit mechanisms have long been explored as tools for resource coordination in distributed systems~\cite{weinhardt2009cloud}. Market-based control has been proposed for cloud scheduling, bandwidth allocation and multi-tenant resource sharing, with emphasis on fairness and efficiency under load. Recent work has advanced auction-based resource coordination in mobile edge computing: repeated auction mechanisms for load-aware dynamic resource allocation under multi-tenant demand~\cite{habiba2023repeated}, and double-auction pricing schemes that jointly address allocation and network pricing~\cite{zheng2023resource}. These approaches demonstrate the practical viability of economic coordination in distributed edge environments, but typically assume independently traded resources. Classical mechanism-design results such as double auctions and equilibrium conditions~\cite{mcafee1992dominant} assume independent resources or substitutable goods, whereas deadline-bound AI services rely on interdependent resource bundles whose dependency structure can alter stability, efficiency and incentive properties.

\begin{table*}[tp]
\centering
\caption{Comparison of Related Work and Positioning of This Article}
\label{tab:related-work-comparison-supp}
\renewcommand{\arraystretch}{1.2}
\begin{tabular}{p{2.55cm} p{4.2cm} p{5.0cm} p{5.0cm}}
\toprule
\textbf{Category} & \textbf{Strengths} & \textbf{Limitations} & \textbf{How This Article Addresses the Gap} \\
\midrule

\textbf{Edge/Fog/ Cloud Service Mgmt}
\cite{hu2023intelligent,liu2026data,tang2025collaborative,li2025tiered,xie2025coedge,shang2024joint}
&
Low-latency orchestration; mobility- and multi-resource-aware management.
&
Assume centralized control; Limited modeling of multi-stage service dependencies; absence of explicit agent-driven decisions; and restricted adaptability under dynamic conditions.
&
Introduce an agentic layer; unify resource management with autonomous task generation and price-mediated allocation; incorporate service-DAG constraints.\\
\midrule

\textbf{Service Function Chaining}
\cite{li2020service,fan2023drl,oskoui2026distributed}
&
Optimize dependency-aware service caching; capture placement constraints.
&
Do not consider strategic agents; lack integrated economic or governance constraints; ignore market or negotiation dynamics.
&
Model service-dependency DAGs formally; analyse laminar leaf-block families, for which trees and series--parallel DAGs are sufficient conditions, as the regime enabling stable coordination; embed into an economic/management framework.\\
\midrule

\textbf{Agentic/ Autonomous Systems}
\cite{park2023generative,deng2025agenticservicescomputing,sedlak2026service}
&
Identify autonomous AI agents as active system participants; highlight emergent coordination challenges.
&
No resource-DAG model; no governance-aware allocation; no management-plane integration or tractable orchestration framework.
&
Provide a full agent--service--resource--governance model; link agent behavior to system feasibility and market outcomes.\\
\midrule

\textbf{Governance, Trust, Norms}
\cite{fornara2008institutions,sun2019trust,huang2022trust}
&
Formalize policies, trust, locality, compliance; mature models of norm enforcement.
&
Not integrated with latency and quality-of-service constraints, dependency-aware placement, or market-based coordination.
&
Embed governance directly into feasibility ($\mathcal{X}_{\mathrm{gov}}$); model reputation-dependent access (a trust cap that does not bind at the reported operating point); evaluate governance--performance trade-offs.\\
\midrule

\textbf{Market-Based Resource Allocation}
\cite{mcafee1992dominant,weinhardt2009cloud,habiba2023repeated,zheng2023resource}
&
Provide incentive-compatible mechanisms; efficient pricing and allocation for substitutable resources.
&
Assume independent or substitutable items; break down under interdependent resource bundles or service graphs; no governance.
&
Identify structural conditions where market stability is preserved (polymatroid regimes); propose hybrid slice-based architecture insulating the market from deep complementarities.\\

\bottomrule
\end{tabular}
\end{table*}

\subsection{Theoretical Contributions: Detailed Positioning vs.\ Amin et al.\ (2026)}
\label{app:novelty-detailed}

The novelty table in \mainname{} states each result against the closest contemporaneous parallel work, Amin~\textit{et~al.}~\cite{amin2026market}, and against what this article does not claim; two deltas distinguish the two settings.

\paragraph*{Two deltas summarised}
The closest parallel work~\cite{amin2026market} proves equilibrium existence on series--parallel networks under homogeneous capacity-sharing disutilities (its Theorem~3.2), whose proof also yields a two-step polynomial-time computation (its Sec.~3.2), gross substitutes of the augmented value function under the same homogeneity condition (Lemma~3.8), and a VCG-equivalent strategyproof equilibrium (Theorem~3.10), shows by the Wheatstone counterexample that equilibrium can fail outside the series--parallel class (Example~3.3) and turns to differentiated path-based pricing there (its Sec.~4.2), and treats a multi-period model (Lemma~4.1).
\emph{(a)~Domain primitives.} Service-dependency DAGs with multi-output, capacity-typed compute stages (compute / storage / bandwidth / inference) versus single-source--single-sink flow networks with edge capacities on transport infrastructure.
\emph{(b)~Route through non-SP regimes.} Integrator encapsulation (main-paper Proposition~3) versus path-based pricing on the underlying network: encapsulation produces an agent-facing polymatroidal interface even when the underlying DAG is entangled, exact where the exported cluster's leaf blocks coincide and an inner exposure (\cref{prop:inner-exposure}) otherwise, decoupling cross-domain coordination from internal sub-DAG complexity.

\section{Simulation Details}
\label{app:simulation}

\subsection{Baseline Parameter Configuration}

Resources are the service nodes of the dependency graph. Every instance carries the same eight nodes, a device node~$d$, three edge nodes $e_1$, $e_2$ and $e_3$, and four leaves $l_1$ to $l_4$, and the same capacity vector, $200$~units at $d$, $100$ at each edge node and $150$, $100$, $150$ and $100$ at $l_1$ to $l_4$, so the per-node capacities sum to the tier totals $\{200, 300, 500\}$ of the device, edge and cloud tiers. A node inherits the base processing delay of the physical tier it sits in, $5$~ms at the device node, $15$~ms at the edge nodes and $50$~ms at the leaves, which puts the zero-queue critical path of every instance at $70$~ms. A task is one throughput token at one leaf and consumes $2$~capacity units at each node on that leaf's path, so in tokens the node capacities are $100$ at $d$, $50$ at each edge node and $75$, $50$, $75$ and $50$ at $l_1$ to $l_4$, and a task's recipe is the token weight times the indicator of its leaf's ancestors. The recipe matrix is therefore the leaf-block incidence matrix, which is the object the theory names.

The three instances differ only in which leaves each internal node reaches. T is a rooted tree of seven arcs, $d$ feeding $e_1$, $e_2$ and $e_3$, $e_1$ feeding $l_1$ and $l_2$, $e_2$ feeding $l_3$ and $e_3$ feeding $l_4$, so the internal leaf blocks are $\{l_1, l_2\}$, $\{l_3\}$ and $\{l_4\}$ and the family is laminar. X is T plus the single arc from $e_2$ to $l_2$, eight arcs in all, so $e_2$'s block becomes $\{l_2, l_3\}$ and crosses $e_1$'s $\{l_1, l_2\}$ without nesting. S is the parallel-fan instance, fifteen arcs in which each of $e_1$, $e_2$ and $e_3$ reaches every leaf, so the three internal blocks coincide at $\{l_1, l_2, l_3, l_4\}$: the family is laminar again, and with every internal block equal to the whole leaf set the region is the uniform-matroid case. Arc count therefore rises from T to X to S while laminarity goes yes, no, yes, so arc density is not what separates the instances.

The agent population calibrates the operating point. The leaf-block capacity $K_c$ of the effective region is the smallest ratio of a node's capacity to its demand weight, which at the uniform leaf mix is $100$~tokens per round on T, where $d$ and $e_1$ bind, $100$ on X, where $d$, $e_1$ and $e_2$ bind, and $50$ on S, where the three edge nodes bind. Populations of $90$ on T, $90$ on X and $45$ on S therefore put all three at the same offered load $\rho = \lambda N / K_c$, which is $1.35$ at high load. Between T and X the leaf-block capacity and the offered load are matched, their $K_c$ being equal by construction. The added arc changes two quantities. Every $l_2$ token also loads $e_2$, which on X carries the arrivals of $l_2$ and $l_3$, twice its share on T, and binds. The arc also adds $e_2$ to the priced nodes of $l_2$'s tokens, so their path at reserve costs $0.32$ against $0.24$. The crossing sweep (\cref{app:crossing-sweep}) varies the family alone and is the structure-only control. Against S the match is a rate match, S's $K_c$ being half theirs, and no contrast involving S leans on it. At the skewed leaf mix $e_1$ alone binds on T and on X, $K_c$ falls to $83.3$~tokens per round and those two instances run at $\rho = 1.62$ at high load, while S is unchanged at $50$ and $1.35$; every result row carries its mix and no contrast crosses mixes.

Task arrivals are independent Poisson draws per agent per round, at $\lambda = 0.5$, $1.0$ and $1.5$~tasks per agent per round for the low, medium and high load levels. Each task draws the leaf it is destined for from the cell's leaf mix, uniform at $0.25$ on each leaf or skewed at $0.55$, $0.05$, $0.20$ and $0.20$ on $l_1$ to $l_4$; the skewed mix exists because at uniform shares the charge on $e_1$ meets its capacity exactly, an over-commitment measured there would be a knife-edge null, and the skewed mix puts that charge at $60$~tokens against a capacity of $50$. Task values are drawn from $\mathrm{Uniform}[1,2]$ and discounted by the exponential latency factor $\delta(T) = e^{-\lambda_l T}$. The decay rate is rescaled to the substrate: the nominal $0.005$~per~ms belongs to an environment whose zero-queue critical path is $135$~ms and every node instance's is $70$~ms, so the rate is $0.005 \times 135/70 = 9.643\times10^{-3}$~per~ms, which holds the share of a task's value surviving its own pipeline where it was and is matched across instances. Deadlines are drawn per task from $\{500, 750, 1000\}$~ms, chosen to match realistic cross-continuum agentic round-trips, such as large-language-model (LLM) calls and multi-hop tool use, on the order of $0.5$ to $1$~s, and a task that misses its deadline receives zero value (salvage fraction~$=0$). The instrumented agentic workload carries its own deadlines and its own decay rate, re-derived from each recording's measured zero-queue critical path by the same rule: $\{4800, 5800, 6800\}$~ms at $\lambda_l = 1.742\times10^{-4}$~per~ms on the single recorded agent and $\{9300, 11100, 13000\}$~ms at $\lambda_l = 9.108\times10^{-5}$~per~ms on the composition of two recordings (\cref{app:exp8-agentic}).

Two latency models are used. Agents form bids against a power-law congestion estimate $\hat{L} = L_\mathrm{base} + \alpha\,\hat\rho^{\,p}$ with $\alpha = 50$ and $p = 1.2$, where $L_\mathrm{base}$ is the leaf's own zero-queue path and $\hat\rho$ is the mean utilisation of that leaf's ancestor nodes in the previous round, read off the region the market advertises and saturating at unit capacity, so that excess load registers as queue saturation and drops. The scale $\alpha$ is chosen so that the worst-case congestion latency $L_\mathrm{base} + \alpha$ is commensurate with the deadline range, and the exponent $p > 1$ produces a convex over-estimate relative to the queueing model below, instantiating the intended bid-time information asymmetry. Execution follows an M/M/1-inspired queueing model~\cite{kleinrock1975queueing} at each node, $L = L_\mathrm{base} + \min\!\bigl(\lambda_\mathrm{load}\cdot\tfrac{\rho}{1-\rho}\cdot c,\; 500\bigr)$~ms, where $\rho$ is the node's realised utilisation clamped below unity and the $500$~ms cap prevents divergence near saturation. A token's latency is the longest source-to-leaf path through its own ancestors at those per-node latencies, each node's congestion computed from the round's own admitted mix while the bid-time estimate above reads the previous round's, and the per-task realisation is drawn from $\mathcal{N}(L_\mathrm{crit},\, 0.1\,L_\mathrm{crit})$, a coefficient of variation of $10\%$. The absolute scale of these latencies is set by assumption: the base delays of $5$, $15$ and $50$~ms by tier are assumed, the congestion term's load response is the part calibrated against the emulated testbed, and the market's clearing computation is not charged to task latency, so the prose states latency effects as shares of the uncontracted arm's median as well as in milliseconds.

The clamp on $\rho$ and the coefficient $c$ together set the congestion parametrisation, which is run at two levels. The reported level is the calibrated one, clamp $0.95$ and coefficient $0.75$, whose elasticity of median latency to offered load is $0.267$ against the $0.260$ measured on the emulated testbed; the steep sensitivity level, clamp $0.99$ and coefficient $2.0$, sits at $1.825$. All three figures come from the congestion-calibration sweep, a released pipeline target over $35$ clamp-and-coefficient pairs that writes its table to the results directory and checks that the reported level is its closest match to the testbed, the two node elasticities read over that grid's own medium-to-high offered step of $89.2$ to $133.8$ tasks per round on three seeds and ten rounds, and the testbed figure over the replay's own low-to-high step (\cref{app:emul-testbed}). The mechanism comparison, the eight-setting parameter sweep and the node-level overhead sweep run at both levels with the same arms and the same seeds, so every mechanism table exists at both and each row is labelled with the level it is read at; every other block runs at the reported level.

The reserve price is $0.04$ per node and sets the marginal-cost floor: a task's path at reserve costs $0.24$ on T, $0.24$ or $0.32$ on X according to its leaf, and $0.40$ on S, that is $16$ to $27$ per cent of the mean task value of $1.5$. Market clearing is t\^{a}tonnement on per-node excess demand, $15$~iterations per round at price step $\eta = 0.15$ normalised by node capacity, with prices floored at the reserve, or at the reserve times the arm's markup where the arm carries one, and capped at $1000$, so a node's price leaves the floor only in whole increments. The integrator adjusts its slice prices with the same step, so no arm runs a faster or slower price clock than another. The online success model is a logistic sigmoid $p(\text{success}\mid\hat\rho) = \sigma(a + b\,\hat\rho)$ initialised at $a = 2.0$ and $b = -2.0$ and updated by stochastic gradient descent with learning rate~$0.3$. Welfare is realised latency-aware value less a congestion penalty $\gamma_c \sum_r [\max(\rho_r - 1, 0)]^2$ with $\gamma_c = 0.05$, which realises the main-text welfare metric $W(t)$ with its externality term $\mathrm{Ext}$ instantiated as that penalty and risk-neutral agents ($\gamma_i = 0$); payments are transfers and net out of social surplus, and the per-provider cost term of $W(t)$ is not modelled, so welfare is value minus congestion penalty. Trust scores initialise at $0.8$ and update asymmetrically, $+0.03$ per successful service-level-agreement event and $-0.08$ per violation, clipped to $[0,1]$, a failure-dominant reputation model. The reputation score $r_i(t) \in [0,1]$ is the state this update maintains, $r_i(t+1) = \Phi(r_i(t), o_i(t), \text{logs}(t))$, where $o_i(t)$ records service-level-agreement compliance or violations~\cite{fornara2008institutions}. A governance policy is a coordinate cap on the effective region and closes the capped leaf outright, so tasks destined for it are refused.

The integrator's two components are separable by construction and are crossed in the Architecture experiment's $2\times2$ factorial, whose arms are uncontracted, uncontracted~$+$~EMA, contracted without EMA, and contracted~$+$~EMA (\cref{app:exp4-ablation}). The contraction is the structural component: the integrator's exported cluster, $e_1$ and $e_2$ on T and X and all three edge nodes on S, is replaced by a single advertised node, the market clears against that quotient region and delivery is evaluated against the true instance. The advertised scalar is the least capacity in the cluster, and which proposition licenses it depends on the instance. On S every cluster node reaches every exported leaf, so the cluster's leaf blocks coincide and the scalar is the exact interface of \cref{prop:encapsulation-supp}. On T and X the cluster's blocks do not coincide, and the scalar is an inner exposure in the sense of \cref{prop:inner-exposure}, deliverable and lossy by construction: the quotient's single bound on $l_1$, $l_2$ and $l_3$ together excludes the true region's point that loads $l_1$ and $l_3$ at $50$ each. The pipeline computes the relation for every arm, and it reports the contracted T and X arms as inner exposures, with a largest gap of $0.5$ of the true region's optimum at the weight vector $(1,0,1,0)$, and the contracted S arm and every uncontracted arm as exact. Smoothing is the price filter, an exponential moving average~\cite{roberts1959control} with coefficient $\beta_{\mathrm{ema}} = 0.8$: within a round the packing runs at the raw clearing prices, and the smoothed vector is both the price the agent is charged and the entry price of the next round, so it reaches the allocation through what the agents bid against it.

Every run is $200$~rounds and the burn-in fraction is $0.2$, so tail statistics are taken on the last $80\%$ of rounds. Price series are summarised twice, over the full run and over the tail, and the factorial is contrasted on the tail, because the full series still contains the opening transient in which a smoothed price climbs from its entry level. Every cell is run on ten seeds, $1$ to $10$, except where a block states otherwise (the eight-setting parameter sweep and the per-tier sensitivity run on five), and the round's task stream is reseeded on the run's seed and the round index, so that at one population and one seed two instances see identical task counts, values, deadlines and leaf labels and differ in supply alone. The mechanism block is the exception: each tunable arm's knob, the posted markup over the range $0.5$ to $4$ or the market's reserve markup over the same range, is chosen on the tuning seeds $1$ to $5$ and the tuned arm is then reported on the disjoint evaluation seeds $11$ to $30$, so no number in the tuned comparison was seen while the knob was being chosen.

The recorded metrics are the ones the tables print. The results grid of \mainname{} carries seven of them per instance and load level: \textit{Lat.}, the mean over rounds of the per-round median realised latency in ms; \textit{Drop}, the fraction of generated tasks that miss their deadline or are never admitted, a control matched on capacity only; \textit{Tok.}, admitted tokens per round; \textit{Welf.}, realised value less the congestion penalty, in arbitrary units; \textit{Alloc.}, the allocative ratio, the arm's admitted set scored on true zero-queue values against the exact allocative optimum on those same values; \textit{Exact.}, the exactness shortfall, the share of rounds in which the greedy packer falls short of the round's exact leaf-block optimum; and $\sigma_p$, the across-rounds coefficient of variation of the demand-weighted mean leaf price, zero below the dispersion onset where the reserve pins every leaf price. \cref{tab:exp1-results} prints the same seven per instance and load level. Beside these, each run records the two reference optima and the ratios to them, the worst round of the greedy-versus-exact instrument, the residual excess demand at the t\^{a}tonnement's terminal prices, the over-commitment of the admitted mix against true node capacities, and the incentive certificate computed on the region the cell actually cleared over; the four per-arm columns of the comparator table, realised welfare over the ex-post prefix optimum, the allocative ratio to the exact optimum, the frontier point and the incentive certificate, are read off them. \cref{tab:sim-baseline-params} consolidates the baseline parameters with their values and justifications. The blocks of \cref{tab:run-accounting} map to the released pipeline's targets as follows: structure \texttt{node\_exp1}, population sweep \texttt{node\_exp2}, governance \texttt{node\_exp3} and the cap-target probe \texttt{node\_exp3b}, architecture \texttt{node\_exp4} and the node overhead sweep \texttt{node\_exp4\_overhead}, architecture $\times$ governance \texttt{node\_exp5}, mechanism \texttt{node\_exp6} with its parameter sweep \texttt{node\_exp6\_sensitivity}, shading \texttt{node\_exp7a}, joint misreport \texttt{node\_exp7b}, agentic workloads \texttt{node\_exp9} and the faithfulness probe \texttt{node\_exp10}; in the per-tier environment, recipes \texttt{exp11}, the per-tier overhead sweep \texttt{exp12} and the per-tier sensitivity sweep \texttt{exp14}.

Three arms keep the per-tier identical-bundle environment and carry no node-level claim: the recipe-heterogeneity experiment (\cref{app:exp11-heterogeneity}), the per-tier encapsulation-overhead sweep (\cref{tab:exp12-overhead}) and the architecture parameter sensitivity of \cref{tab:exp4-sensitivity}. The node-level overhead sweep of \cref{tab:node-overhead} is a separate experiment on the node substrate. That environment carries one capacity row per physical tier, $\{200, 300, 500\}$~units at device, edge and cloud, with base delays $\{5, 15, 50\}$~ms, the nominal decay rate $\lambda_l = 0.005$~per~ms against deadlines $\{500, 750, 1000\}$~ms, and the same reserve, t\^{a}tonnement, success model, trust update and welfare definition as above. An arm of that environment is identified by the per-tier demand profile it carries, the three-entry vector of device, edge and cloud demand weights; the overhead sweep and the sensitivity sweep use the profiles $(2,4,5)$ and $(5,4,3)$, at populations $N = 55$ and $35$. The tables of these three arms label their rows by that profile, because the profile, the tier set and the critical-path tier sequence are the only channels through which a dependency graph reaches allocation or price there, and no effect reported in them is attributed to a topology. The recipe-heterogeneity experiment runs the tree's tier set with its demand weights replaced by the recipe mix-average, $(1.5, 1.5, 1.5)$, at $N = 8$ over $100$~rounds at medium load, with per-tier capacity set to the catalogue cap and each round's instance truncated at fourteen tasks, the exact packer's enumeration limit. The overhead sweep and the architecture sensitivity sweep run the other two profiles at $N = 55$ and $N = 35$ over $200$~rounds, the first on ten seeds and the second on five.

\begin{table*}[tp]
\centering
\caption{Baseline Simulation Parameters}
\label{tab:sim-baseline-params}
\renewcommand{\arraystretch}{1.05}
\footnotesize
\begin{tabular}{p{3.8cm} l l p{6.0cm}}
\toprule
\textbf{Parameter} & \textbf{Symbol} & \textbf{Value} & \textbf{Justification} \\
\midrule
\multicolumn{4}{l}{\textit{Instances}} \\
Node capacities & $C_v$ & 200 / 100 / 100 / 100 / 150 / 100 / 150 / 100 & Units at $d$, $e_1$, $e_2$, $e_3$, $l_1$ to $l_4$; sums to the tier totals $\{200, 300, 500\}$ \\
Token weight & $w$ & 2 & Capacity units one throughput token consumes at each node on its path \\
Nodes and arcs & --- & 8 nodes; 7 / 8 / 15 arcs & One device node, three edge nodes and four leaves, shared by T, X and S; X is T plus the arc $e_2 \to l_2$, which makes the leaf-block family cross \\
Base processing delays & $L_\mathrm{base}$ & $\{5, 15, 50\}$~ms & By physical tier; zero-queue critical path $70$~ms on every instance \\
Flow bound ratio & --- & 1.0 / 1.0 / 0.5 & T / X / S; the max-flow bound of a routing count over the leaf-block capacity of Definition~1. On S the routing bound is twice the fork-join capacity, because every token there loads all three parallel edge nodes \\
Reserve price per node & --- & 0.04 & A task's path at reserve costs $0.24$ to $0.40$, $16$ to $27$ per cent of mean task value \\
\midrule
\multicolumn{4}{l}{\textit{Operating point}} \\
Agent count & $N$ & 90 / 90 / 45 & T / X / S; puts all three at the same offered load; varied in Exp~2 \\
Leaf-block capacity & $K_c$ & 100 / 100 / 50 tokens per round & Uniform mix; $83.3$ on T and X at the skewed mix, where $e_1$ alone binds \\
Offered load, high & $\rho$ & 1.35 & $\lambda N / K_c$, equal on all three at the uniform mix; $1.62$ on T and X at the skewed mix \\
Rounds and burn-in & --- & 200, 0.2 & Allows t\^{a}tonnement and learning convergence; tail statistics discard the opening $20\%$ of rounds \\
Monte Carlo seeds & --- & 10 & Seeds $1$ to $10$; the mechanism block tunes on $1$ to $5$ and reports its tuned comparison on $11$ to $30$ \\
\midrule
\multicolumn{4}{l}{\textit{Task model}} \\
Arrival rate & $\lambda$ & 1.0 tasks/round/agent & Poisson; $0.5$ at low and $1.5$ at high load \\
Leaf mix & --- & $0.25$ each; $0.55 / 0.05 / 0.20 / 0.20$ & Uniform and skewed; the skewed mix lifts $e_1$'s charge above its capacity \\
Task value & $v$ & $\mathrm{Uniform}[1,2]$ & Normalised; focuses analysis on structural effects \\
Latency decay rate & $\lambda_l$ & $9.643\times10^{-3}$ /ms & Nominal $0.005$ rescaled by $135/70$, the two zero-queue critical paths; re-derived on the agentic workload \\
Deadlines & $D$ & $\{500, 750, 1000\}$~ms & Realistic cross-continuum agentic round-trips (LLM calls, multi-hop tool use, $0.5$--$1$~s); re-derived on the agentic workload \\
Salvage fraction & --- & 0 & Hard deadline: missed $\Rightarrow$ zero value \\
\midrule
\multicolumn{4}{l}{\textit{Bidding and execution models}} \\
Bidding congestion & $\hat{L}$ & $L_\mathrm{base} + 50\,\hat\rho^{1.2}$, $\hat\rho \le 1$ & Power law on the mean utilisation of the leaf's ancestors in the previous round, saturating so excess load registers as drops \\
Execution latency & $L$ & M/M/1 with 500~ms cap, CV $10\%$ & Per node, along the leaf's longest path; the cap prevents divergence and the noise is Gaussian around the critical-path latency \\
Congestion level & (clamp, $c$) & $(0.95, 0.75)$ reported, $(0.99, 2.0)$ sensitivity & The reported level is calibrated so the model's latency elasticity matches the emulated testbed's; the mechanism comparison, the node-level overhead sweep and the parameter sweep run at both, every other block at the reported level \\
\midrule
\multicolumn{4}{l}{\textit{Trust}} \\
Trust initialisation & $\tau_0$ & 0.8 & Moderate prior trust \\
Trust update & --- & $+0.03$ / $-0.08$ & Asymmetric and failure-dominant per~\cite{sun2019trust} \\
\midrule
\multicolumn{4}{l}{\textit{T\^{a}tonnement and learning}} \\
Iterations per round & --- & 15 & The pipeline's iteration budget \\
Price step & $\eta$ & 0.15 & Normalised by node capacity; shared by the node and the slice t\^{a}tonnement \\
Price floor / cap & --- & reserve / 1000 & Floor at the reserve times the arm's markup; cap prevents runaway prices \\
Tuned knob grids & --- & $0.5$ to $4$ & The market arms' reserve markup and the posted arms' markups, chosen on the tuning seeds; the frontier is read at $1$ to $4$ \\
Online success model & --- & $\sigma(2.0 - 2.0\,\hat\rho)$ & Conservative initial prior; learns from outcomes \\
Learning rate & --- & 0.3 & Balances responsiveness and stability \\
Congestion penalty & $\gamma_c$ & 0.05 & Mild welfare penalty for over-utilisation \\
\bottomrule
\end{tabular}
\end{table*}

\subsection{Experiment-Specific Parameter Settings}

\cref{tab:sim-experiment-params-supp} summarises the factors varied in each experiment and the parameters held fixed.

\begin{table*}[tp]
\centering
\caption{Experiment-Specific Parameter Settings}
\label{tab:sim-experiment-params-supp}
\renewcommand{\arraystretch}{1.1}
\begin{tabular}{p{0.5cm} p{3.7cm} p{7.0cm} p{5.0cm}}
\toprule
\textbf{Exp.} & \textbf{Varied Aspect} & \textbf{Settings} & \textbf{Fixed} \\
\midrule

1 &
Leaf-block family $\times$ load &
T (rooted tree), X (T plus one arc, whose leaf-block family crosses), S (parallel fan, whose internal leaf blocks coincide); low, medium, high ($\lambda \in \{0.5, 1.0, 1.5\}$) &
$N = 90$ (T), $90$ (X), $45$ (S); governance: none; architecture: uncontracted; uniform leaf mix; $90$~runs \\

2  &
Agent population $\times$ instance &
$N$ over a $28$-point grid from $10$ to $200$; T, X and S; medium and high load &
Governance: none; architecture: uncontracted; $1{,}680$~runs \\

3 &
Governance policy &
None, two coordinate caps (a provider trust threshold and a jurisdiction predicate), a role-class cap enforced at admission, and a cross-node coupling, a joint token budget over two leaves in different domains, the measured member of the class a data-residency rule belongs to, with that coupling sliced by domain &
T, X and S; load: medium and high; both leaf mixes; $N = 90$ / $90$ / $45$; $720$~runs \\

4 &
Architecture factorial &
Uncontracted, uncontracted~$+$~EMA, contracted without EMA, contracted~$+$~EMA; the interface is fixed by the arm, none uncontracted and the cluster's bottleneck scalar contracted, and the advertised routed composition is the faithfulness probe's &
T, X and S; load: medium and high; $2 \times 2 \times 3 \times 2 \times 10 = 240$~runs \\

5 &
Architecture $\times$ governance &
Contracted vs.\ uncontracted; caps vs.\ none; T, X and S &
$N = 90$ / $90$ / $45$; load: medium and high; $2 \times 2 \times 3 \times 2 \times 10 = 240$~runs \\

6 &
Allocation mechanism, grouped by comparator tier &
Tier~A reference optima as reported columns; Tier~B random, earliest-deadline-first~\cite{liu1973scheduling}, value-greedy, the Kubernetes-style rank~\cite{k8sdocs2026priority} and the value-ranked posted price; Tier~C the market, the congestion-consistent market, the arrival-order and the deadline-priority posted prices over a markup grid re-levelled around the interior optimum, a demand-responsive posted price tuned over its step and target, and the mixed arm; uncontracted and contracted architecture &
T, X and S; load: medium and high; $N = 90$ / $90$ / $45$; no governance; both congestion levels; $9{,}120$~runs on the level grid, $13{,}080$ tuning and $3{,}840$ evaluation runs \\

7 &
Valuation shading factor $\alpha$ &
$\alpha \in \{0.5, 0.7, 0.9, 1.0, 1.1, 1.3\}$; one agent misreports at a time, the others truthful &
The two certified instances T and S and the uncertified instance X; the code checks the precondition, a polymatroidal effective region with unit-demand rows, before the arm runs; $N = 90$, $45$ and $8$, the last at a capacity scale of $0.1$ so that capacity binds; $30$~rounds; $30$~runs \\

8 &
The recorded leaf-block family &
One recorded agent, whose single sink makes its family laminar, and a composition, built by us, of two recordings sharing a tool stage, whose family crosses by that sharing; medium and high load &
$N=200$; deadlines and value decay re-derived from each recording; demand weights on one token unit; $80$~runs \\

9 &
Recipe heterogeneity $\times$ catalogue cap &
Homogeneous and heterogeneous recipes, each with per-resource pricing and encapsulated; cap $\in \{6, 9, 12\}$ &
Tree DAG; $N = 8$; $100$~rounds; medium load; per-round instance truncated at $14$ tasks (the exact packer's enumeration cap); aggregate per-tier demand matched across arms; $120$~runs \\

\bottomrule
\end{tabular}
\end{table*}

A run is one per-seed simulation of one arm of a block: one seed of one cell of the block's grid, with each of the two arms counted once in the per-tier sensitivity sweep, which compares a pair of arms per cell. \cref{tab:run-accounting} gives the counts, read off the released pipeline's grids; the experiments without a number are specified where they are reported.

\begin{table*}[!t]
\centering
\caption{Run accounting: one run is one per-seed simulation of one arm of a block. Counts are read off the released pipeline's grids.}
\label{tab:run-accounting}
\footnotesize
\setlength{\tabcolsep}{3pt}
\begin{tabular}{@{}p{4.2cm} p{10.2cm} r@{}}
\toprule
\textbf{Block} & \textbf{Grid} & \textbf{Runs} \\
\midrule
\multicolumn{3}{@{}l}{\emph{Ablation experiments}} \\
Structure & 3 instances $\times$ 3 loads $\times$ 10 seeds & $90$ \\
Governance & 6 levels $\times$ 3 instances $\times$ 2 loads $\times$ 2 mixes $\times$ 10 & $720$ \\
Cap-target probe & 4 targets $\times$ 2 loads $\times$ 10, on X & $80$ \\
Architecture & 4 arms $\times$ 3 instances $\times$ 2 loads $\times$ 10 & $240$ \\
Architecture $\times$ governance & 2 $\times$ 2 $\times$ 3 instances $\times$ 2 loads $\times$ 10 & $240$ \\
Mechanism, level grid & 37 arm-levels uncontracted, 39 contracted, $\times$ 3 $\times$ 2 loads $\times$ 2 congestion levels $\times$ 10 & $9{,}120$ \\
Mechanism, tuning & 109 knob levels per cell $\times$ 24 cells $\times$ seeds 1--5 & $13{,}080$ \\
Mechanism, evaluation & 8 tuned arms $\times$ 24 cells $\times$ seeds 11--30 & $3{,}840$ \\
\midrule
Population sweep & 28 populations $\times$ 3 instances $\times$ 2 loads $\times$ 10 & $1{,}680$ \\
\midrule
\multicolumn{3}{@{}l}{\emph{Further node-level blocks}} \\
Shading & 3 instances $\times$ 10 & $30$ \\
Joint misreport, contention sweep & 5 populations $\times$ 4 capacity scales $\times$ 3 instances $\times$ 10 for VCG, the two laminar evaluation points run twice ($620$), and for the market ($600$) & $1{,}220$ \\
Agentic workloads & 2 workloads $\times$ 2 arms $\times$ 2 loads $\times$ 10 & $80$ \\
Faithfulness probe & 3 interfaces $\times$ 3 instances $\times$ 2 mixes $\times$ 10, and 2 interfaces $\times$ 2 mixes $\times$ 10 on S with two leaf sizes & $220$ \\
Node overhead sweep & 7 arm-levels $\times$ 3 instances $\times$ 2 congestion levels $\times$ 10 & $420$ \\
Parameter sweep & 8 settings $\times$ 4 arms $\times$ 3 instances $\times$ 2 architectures $\times$ 5 & $960$ \\
\midrule
\multicolumn{3}{@{}l}{\emph{Per-tier arms}} \\
Recipes & 4 arms $\times$ 3 caps $\times$ 10 & $120$ \\
Per-tier overhead sweep & 2 profiles $\times$ 3 levels $\times$ 10 & $60$ \\
Per-tier sensitivity & 13 levels $\times$ 2 profiles $\times$ 5 seeds $\times$ 2 arms & $260$ \\
\midrule
\multicolumn{3}{@{}l}{\emph{Stability battery, outside the total}} \\
Convergence, determinacy & 360 and 90 & $450$ \\
Shock, report stability & 360 and 60 & $420$ \\
\multicolumn{3}{@{}l}{\emph{Congestion calibration, outside the total}} \\
Calibration sweep & 35 clamp-and-coefficient pairs $\times$ 2 loads $\times$ 3 seeds, 10 rounds each & $210$ \\
\bottomrule
\end{tabular}
\end{table*}

The five ablation experiments with the cap-target probe total $27{,}410$~runs, the population sweep $1{,}680$, the further node-level blocks $2{,}930$ and the per-tier arms $440$, so the evaluation as a whole comprises $32{,}460$~runs, $32{,}020$ of them on the node substrate. The stability battery of \cref{app:crossing-sweep} adds $870$ runs and the congestion calibration $210$ short runs outside that total, and the existence certificate and the crossing sweep run no market.

\subsection{Component Ablation Summary}
\label{app:ablation-summary}

\cref{tab:ablation-summary-supp} consolidates the effect of removing or toggling each design component under its most informative condition; \cref{tab:exp1-results} gives the per-instance$\times$load view. $\sigma_p$ is the coefficient of variation (CV) of the demand-weighted mean leaf price, over the full run unless the row says tail, and all values are $10$-seed means at the matched operating point and the reported congestion level, except the mechanism rows, which are means of the tuned arms over the twenty evaluation seeds. The structural component is reported on allocative exactness, because that is where the leaf-block family enters, with the drop rate beside it as the capacity-matched control. The mechanism block is reported at both congestion parametrisations, because its ordering against a truth-assuming planner depends on which one is used.

\begin{table}[htbp]
\centering
\caption{Component ablation summary. Each block toggles one design component under its most informative condition. The arrow in a block heading gives the direction of the Before and After columns. In the $-$M block both arms are read on the twenty evaluation seeds, the market at its tuned reserve markup ($2$ at the reported level, $3$ at the steep level).}
\label{tab:ablation-summary-supp}
\footnotesize
\setlength{\tabcolsep}{4pt}
\begin{tabular}{@{}llcc@{}}
\toprule
\textbf{Component} & \textbf{Metric} & \textbf{Before} & \textbf{After} \\
\midrule
\multicolumn{4}{@{}l}{\textit{$-$S: T $\to$ X (high load)}} \\
 & Exactness shortfall & $0.0000$ & $0.1470$ \\
 & Incentive certificate & Passes & Fails \\
 & Drop rate (control) & $0.290$ & $0.308$ \\
\addlinespace[2pt]
\multicolumn{4}{@{}l}{\textit{$-$H: Contracted $\to$ Uncontracted (T, high)}} \\
 & Median latency & $80$\,ms & $95$\,ms \\
 & Admitted tokens & $78.74$ & $94.99$ \\
 & Welfare & $59.92$ & $61.65$ \\
\addlinespace[2pt]
\multicolumn{4}{@{}l}{\textit{$-$G: Jurisdiction cap $\to$ None (X, high)}} \\
 & Exactness shortfall & $0.0000$ & $0.1470$ \\
 & Incentive certificate & Passes & Fails \\
 & Admitted tokens & $49.30$ & $92.44$ \\
\addlinespace[2pt]
\multicolumn{4}{@{}l}{\textit{H$\times$G: Either remedy $\to$ Neither (X, high)}} \\
 & Exactness shortfall & $0.0000$ & $0.1470$ \\
 & \multicolumn{3}{@{}l}{\quad substitutes; together they do no more} \\
\addlinespace[2pt]
\multicolumn{4}{@{}l}{\textit{$-$M: Tuned market $\to$ Value-greedy (T, high, uncontr.)}} \\
 & Welfare over opt., reported & $0.971$ & $0.954$ \\
 & Welfare over opt., steep & $0.707$ & $0.083$ \\
 & \multicolumn{3}{@{}l}{\quad calibration-dependent, no sign} \\
\bottomrule
\end{tabular}
\end{table}

\subsection{Per-Tier Parameter Sensitivity of the Smoothing Filter's Reporting Identity}
\label{app:exp4-sensitivity}

This sweep is run in the per-tier identical-bundle environment and is reported as a parameter sensitivity of that environment alone; it carries no node-level claim. The reduction it measures is the EMA filter's low-pass identity on the reported price series (\cref{app:exp4-ablation}), and no claim of the article rests on it; the sweep shows how the size of that identity moves with the per-tier parameters. The integrator's efficiency factor is held at one in every reported arm, and its sensitivity is not reported. Holding all other parameters at their defaults, we vary each of three parameters that plausibly drive the result, the per-tier capacity scale, the integrator slice-price step~$\eta$, and the latency-decay rate~$\lambda_l$, and recompute the median slice-versus-na\"{i}ve reduction in the agent-facing price CV (full series) over the volatile cells (the $(2,4,5)$ and $(5,4,3)$ demand profiles, high load, each at its own operating point, $N = 55$ and $N = 35$, $5$~seeds on each of the two profiles for ten cells, full-series $\sigma_p$, volatility floor $0.02$). \cref{tab:exp4-sensitivity} reports the result.

\begin{table}[htbp]
\centering
\caption{Architecture sensitivity: one-at-a-time parameter sensitivity of the median price-volatility reduction (slice$+$EMA vs.\ na\"{i}ve, fair agent-facing metric) over the volatile cells. $n_{\mathrm{vol}}$ is the number of the ten cells that clear the volatility floor of $0.02$; the median, minimum and maximum are taken over those cells. Baseline levels in bold. The $(2,4,5)$ and $(5,4,3)$ demand profiles, high load, each at its own operating point, $N = 55$ and $N = 35$, $5$~seeds.}
\label{tab:exp4-sensitivity}
\footnotesize
\begin{tabular}{llcccc}
\toprule
\textbf{Parameter} & \textbf{Value} & \textbf{$n_{\mathrm{vol}}$} & \textbf{Median} & \textbf{Min} & \textbf{Max} \\
\midrule
Capacity scale            & $0.70$            & 3  & $0.906$          & $0.895$ & $0.923$ \\
                          & $\mathbf{1.00}$   & 10 & $\mathbf{0.840}$ & $0.636$ & $0.928$ \\
                          & $1.30$            & 5  & $0.850$          & $0.838$ & $0.895$ \\
\midrule
Slice-price step $\eta$   & $0.10$            & 10 & $0.851$          & $0.638$ & $0.921$ \\
                          & $\mathbf{0.15}$   & 10 & $\mathbf{0.840}$ & $0.636$ & $0.928$ \\
                          & $0.20$            & 10 & $0.841$          & $0.632$ & $0.915$ \\
\midrule
Latency decay $\lambda_l$ & $0.0025$          & 10 & $0.824$          & $0.696$ & $0.881$ \\
                          & $\mathbf{0.0050}$ & 10 & $\mathbf{0.840}$ & $0.636$ & $0.928$ \\
                          & $0.0075$          & 0  & \multicolumn{3}{c}{no volatile cell} \\
\midrule
\bottomrule
\end{tabular}
\end{table}

Eight of the nine levels have at least one volatile cell. Across those eight the median reduction spans $0.824$ to $0.906$, and the smallest per-cell reduction anywhere in the sweep is $0.632$. The baseline level of every parameter sits inside that band.

At $\lambda_l = 0.0075$ no cell clears the volatility floor, so no reduction is defined at that level: value decays fast enough that the market does not run hot enough to oscillate. The three capacity-scale levels rest on bases of $3$, $10$ and $5$ volatile cells, so their medians are not directly comparable with one another.

\subsection{Result Summary Tables}

\cref{tab:exp1-results} opens this subsection and consolidates the Structure experiment's numeric results; the per-cell values of the scaling, governance, architecture and architecture$\times$governance experiments are reported in the sections that discuss them. All values are means over $10$~random seeds and latency is in milliseconds. Bracketed figures are Student-$t$ $95\%$ intervals over the seeds. Where a column carries no brackets, its caption gives the widest interval in the column.

\begin{table*}[tp]
\centering
\caption{Structure: impact of the leaf-block family and the load level. Uncontracted architecture, no governance, uniform leaf mix, $N = 90$ (T), $90$ (X), $45$ (S); $10$~seeds per cell, $90$~runs, at the reported congestion level. Student-$t$ $95\%$ intervals are at most $\pm0.35$~ms wide on median latency, $\pm0.009$ on drop rate and $\pm0.012$ on $\sigma_p$.}
\label{tab:exp1-results}
\renewcommand{\arraystretch}{1.1}
\footnotesize
\setlength{\tabcolsep}{4pt}
\begin{tabular}{ll rrrrrrr}
\toprule
\textbf{Instance} & \textbf{Load} & \textbf{Med.\ Lat.\ (ms)} & \textbf{Drop Rate} & \textbf{Tokens} & \textbf{Welfare} & \textbf{Alloc.} & \textbf{Shortfall} & \textbf{$\sigma_p$} \\
\midrule
T & Low    & 70.68 & 0.0000 & 45.09 & 34.27 & 1.0000 & 0.0000 & 0.0000 \\
T & Medium & 82.27 & 0.0132 & 88.72 & 60.43 & 0.9993 & 0.0000 & 0.1697 \\
T & High   & 95.11 & 0.2900 & 94.99 & 61.65 & 0.9541 & 0.0000 & 0.1260 \\
\midrule
X & Low    & 70.75 & 0.0000 & 45.09 & 34.25 & 1.0000 & 0.0000 & 0.0000 \\
X & Medium & 83.96 & 0.0198 & 88.10 & 59.18 & 0.9969 & 0.0130 & 0.1144 \\
X & High   & 96.32 & 0.3078 & 92.44 & 58.89 & 0.9253 & 0.1470 & 0.1666 \\
\midrule
S & Low    & 70.48 & 0.0052 & 22.41 & 17.07 & 0.9962 & 0.0000 & 0.0000 \\
S & Medium & 75.99 & 0.0969 & 40.41 & 29.95 & 0.9383 & 0.0000 & 0.0723 \\
S & High   & 84.19 & 0.4300 & 36.84 & 26.46 & 0.7541 & 0.0000 & 0.2679 \\
\bottomrule
\end{tabular}
\\[4pt]
{\footnotesize \textit{Alloc.} is the allocative ratio, the arm's admitted set scored on true zero-queue values against the exact allocative optimum. \textit{Shortfall} is the share of rounds in which the greedy packer falls short of the round's exact leaf-block optimum; it is zero on the two laminar instances at every load and positive on X alone. The drop rate is the capacity-matched control, and $\sigma_p$ is zero at low load on every instance because the reserve pins every leaf price below the dispersion onset.}
\end{table*}

\subsection{Architecture factorial: smoothing and encapsulation}
\label{app:exp4-ablation}

The integrator applies two changes relative to the uncontracted baseline: EMA price smoothing ($\beta_{\mathrm{ema}} = 0.8$), and a sub-DAG contraction, which replaces per-node bidding inside the cluster with a single quotient node. Crossing the two gives four arms, uncontracted, uncontracted~$+$~EMA, contracted without EMA, and contracted~$+$~EMA, over three instances and two loads: $2 \times 2 \times 3 \times 2 \times 10 = 240$~runs. The paragraphs below decompose the volatility reduction and report the four arms.

\paragraph*{Price volatility} The smoothing half of the integrator moves the price series and the contraction half does not. Pooled over the grid, the contraction's main effect on the post-burn-in price CV is not resolved ($H(1) = 0.9604$, $p = 0.3271$, $n = 240$), while smoothing's is ($H(1) = 150.99$, $p = 1.1\times10^{-34}$). Smoothing is a low-pass identity of the reporting filter: the ratio of the agent-facing to the cleared post-burn-in CV runs $0.2994$ to $0.3382$ over the twelve EMA cells against the white-noise value $\sqrt{(1-\beta)/(1+\beta)} = 0.3333$ at $\beta = 0.8$, and its effect on the cleared series itself is null in four of six factorial cells and positive in two. Per cell the two factors are sub-additive in three of six cells and super-additive in three, with every interval excluding zero. The pooled null for the contraction is a cancellation: per cell its effect on the tail CV is resolved in all six cells, lowering it in three and raising it in three. Read on the full series instead of the tail, the same grid would report the uncontracted~$+$~EMA arm as the most volatile of the four, which is the opening transient of a smoothed price.

\paragraph*{Latency, welfare and throughput} The contraction is reported in \mainname{} as the component that moves the allocation, and the pooled statistics behind that reading are these. Its pooled main effects resolve on median latency ($H(1) = 66.94$, $p = 2.8\times10^{-16}$), on welfare ($H(1) = 9.21$, $p = 0.0024$) and on the drop rate ($H(1) = 7.00$, $p = 0.0081$), while smoothing resolves none of the three ($p \ge 0.13$, $n = 240$). The direction differs by instance. Paired by seed at high load, the contraction cuts median latency by $16.3\%$ [$16.0$, $16.7$] of the uncontracted arm's on T ($15.5$~ms) and $17.0\%$ [$16.5$, $17.5$] on X, and by $16.4\%$ [$15.8$, $17.0$] and $16.8\%$ [$16.3$, $17.4$] with the filter on both arms. With smoothing it takes welfare on T from $61.65$ [$61.45$, $61.86$] to $59.92$ [$59.65$, $60.19$] and admitted tokens from $94.99$ [$94.63$, $95.35$] to $78.74$ [$78.17$, $79.30$] per round. Averaged over the two filter levels and paired by seed, the contraction costs $18.7$ [$18.1$, $19.3$] and $16.1$ [$15.3$, $16.9$] admitted tokens per round and $3.38$ [$3.04$, $3.73$] and $0.89$ [$0.44$, $1.34$] of welfare on T and X. On S it runs the other way: averaged over the two filter levels, the contraction recovers $10.63$ [$10.25$, $11.01$] tokens per round that the uncontracted market leaves unadmitted and gains $5.84$ [$5.57$, $6.11$] of welfare and $0.181$ [$0.173$, $0.190$] of the ex-post prefix optimum ($12.57$ tokens and $6.81$ of welfare without the filter, $8.69$ and $4.87$ with it), and is $8.9\%$ [$8.5$, $9.3$] slower than the uncontracted arm ($7.5$~ms), $5.4\%$ [$5.0$, $5.7$] with the filter on both arms, because the recovered tokens queue. The recovered tokens stay inside the region, the contracted arms admitting $49.35$ and $49.40$ per round at high load against the region's $50$, and the over-commitment column is zero on every arm of the factorial.

\paragraph*{Interpretation} The factorial separates the integrator into a price filter and a structural change, which are independent switches by construction. The filter makes the agent-facing series a low-pass image of the cleared one. The contraction is the structural change, and which proposition licenses it depends on the instance (\cref{app:simulation}). On S the cluster's leaf blocks coincide and the contraction is the exact scalar interface of \cref{prop:encapsulation-supp}: it admits more, earns more and is slower, and part of that gain comes from the lower reserve floor of the quotient, whose paths carry three priced nodes where the instance's carry five, so that a path's reserve cost falls from $0.40$ to $0.24$. On T and X the contraction is an inner exposure (\cref{prop:inner-exposure}), and its volume cost is the measured price of that exposure: one bound on three leaves excludes mixes the true region admits, so the contracted market sheds load, and that shedding buys the latency cut on T and X; T is laminar already, so the cut there does not come from restoring laminarity. On X the quotient is a rooted tree, so allocative exactness and the incentive certificate, both read on the quotient region the agents face, hold by construction: they certify that region, not the true instance. The contraction renders the crossing and the tree indistinguishable on the advertised side, and its welfare stays within a narrower band across the parameter sweep on T and X (\cref{app:exp6}). Neither component buys price stability. The remaining parameters of the per-tier environment are swept in \cref{tab:exp4-sensitivity}.

\paragraph*{The filter moves the allocation as well} The smoothed vector is the price the agent is charged and the entry price of the next round, so the filter reaches the allocation through what the agents bid against it, and its effect there is measured. At high load, paired by seed over the ten seeds with Student-$t$ intervals, switching the filter on without the contraction moves welfare by $+0.16$ [$+0.04$, $+0.27$] on T, $+0.22$ [$+0.00$, $+0.44$] on X and $+1.94$ [$+1.66$, $+2.23$] on S, admitted tokens by $+0.78$ [$+0.35$, $+1.22$], $+0.69$ [$+0.20$, $+1.18$] and $+3.82$ [$+3.40$, $+4.24$] per round, and median latency by $+0.6\%$ [$+0.1$, $+1.1$], $+0.4\%$ [$-0.1$, $+0.8$] and $+3.3\%$ [$+2.9$, $+3.6$] of the unfiltered arm's ($+0.6$, $+0.4$ and $+2.7$~ms). Switching it on under the contraction moves welfare by $+3.14$ [$+2.33$, $+3.95$], $+3.12$ [$+2.32$, $+3.93$] and $+0.01$ [$-0.03$, $+0.05$], and admitted tokens by $+4.13$ [$+3.11$, $+5.15$] on T and on X and $-0.05$ [$-0.14$, $+0.03$] on S. The effect is small on the uncontracted tree and crossing instance and absent on the contracted fan, and the pooled split of the grid, which neither pairs by seed nor holds the instance fixed, resolves the filter on none of the three allocation responses; the paired per-cell effects are the reading. The two switches are separable by construction; their effects are not divided cleanly between a reporting side and an allocative one.

\section{Detailed Experiment Results}
\label{app:detailed-results}

This section presents figures and extended discussion for each experiment. Condensed ablation results appear in \mainname; here we provide the per-experiment visualisations and condition-by-condition analysis.

\subsection{Structure}
\label{app:exp1-detailed}

\cref{tab:exp1-results} carries median latency, drop rate, admitted volume and price volatility across the three instances and three load levels, and the structure figure of \mainname{} plots the first three of them. The metric that links the simulation to Proposition~1 is allocative exactness: where the leaf-block family is laminar, Edmonds' greedy attains the exact optimum over the region, and where the family crosses it need not. The price-side instrument is residual excess demand at the terminal price vector, because price dispersion does not order the instances.

At low load ($\lambda = 0.5$) the three instances run between $70.5$ and $70.8$~ms of median latency and $\sigma_p$ is zero on each: the reserve pins every leaf price and no price moves. The exactness shortfall is zero on T and on S at every load, and the instances separate on it alone. On a laminar family that zero is an identity of the instrument: the exact reference there is the same descending-value walk the packer runs, which Edmonds' theorem makes the optimum.

Under load the drop rate separates the instances, but it is a control. It decomposes into a capacity term matched by construction and two unmatched terms, a price-rationing term and a deadline term, so a difference between instances on it is not a structural measurement. At high load it runs $0.2900$ [$0.2864$, $0.2937$] on T, $0.3078$ [$0.3021$, $0.3134$] on X and $0.4300$ [$0.4215$, $0.4386$] on S.

Price dispersion orders differently at each load, and the instances do not order it. At high load $\sigma_p$ runs S $0.2679$, X $0.1666$, T $0.1260$; at medium it runs T $0.1697$, X $0.1144$, S $0.0723$; at low load it is zero on all three, a structural zero, because the reserve pins every leaf price below the dispersion onset. T and S are both laminar and sit at opposite ends of the high-load ordering, so laminarity orders none of it. The structure orders exactness: the greedy packer falls short of the round's exact leaf-block optimum in $0.1470$ [$0.1242$, $0.1698$] of X's rounds at high load and in none of T's or S's, and with that shortfall goes the incentive certificate.

\subsection{Governance}
\label{app:exp3-detailed}

Six governance levels are compared: none; two coordinate caps, a provider trust threshold and a jurisdiction predicate; a role-class cap, enforced at admission on each agent's own role and so a class cap; and a cross-node coupling, with that coupling sliced by domain. The experiment runs both leaf mixes at both loads; the cap-level numbers of this subsection are read at the uniform mix, the skewed mix is the control, and the coupling pair below is reported at both.

\begin{figure}[!t]
    \centering
    \includegraphics[width=1\linewidth]{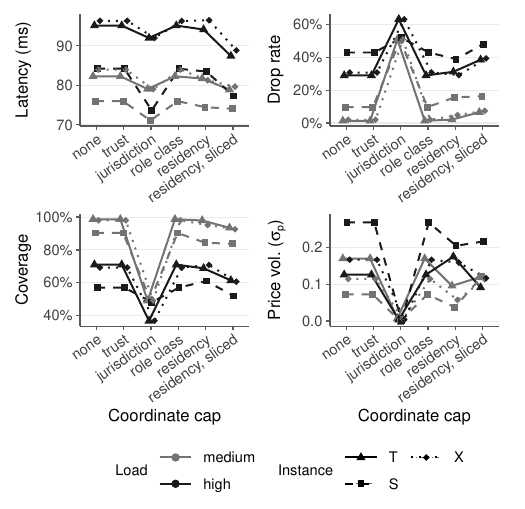}
    \caption{Governance: effect of the governance level (none, two coordinate caps, a role-class cap, and the cross-node coupling with its domain slice) on performance for T, X and S under medium and high load, at the uniform leaf mix. Panels: median latency, drop rate, service coverage, price volatility. Coverage here is the share of offered tasks served, one minus the drop rate, so the coverage panel mirrors the drop panel.}
    \label{fig:exp3-governance-supp}
\end{figure}

\cref{fig:exp3-governance-supp} reports median latency, drop rate, service coverage (the share of offered tasks served) and price volatility. All allocated tasks are policy-compliant by construction, so the throughput-for-quality trade appears in coverage. The structural result is separate: on T and on S the region stays a polymatroid at every coordinate-wise cap level and the exactness shortfall is identically zero under each of them, as \cref{prop:polymatroid-supp} and the truncation argument require.

On X the cap's effect depends on its target. Zeroing $l_1$, $l_2$ or $l_3$ leaves the surviving blocks nested, so the greedy packer is exact by Edmonds' theorem, the shortfall is $0.0000$ in all $60$ runs and the certificate passes at both loads; zeroing $l_4$ leaves the crossing pair untouched, at $0.7505$ [$0.7312$, $0.7698$] at high load and $0.0165$ [$0.0104$, $0.0226$] at medium, with a worst-round ratio of $0.9012$ [$0.8952$, $0.9072$] and a certificate that fails. Three of the four targets repair and the fourth is a null.

The two partial-dose caps do not bind at this operating point, so neither illustrates a partial repair. The role cap admits agents of the capped role to at most half of each leaf's token capacity; on X at high load it leaves the shortfall at $0.1465$ [$0.1243$, $0.1687$] against $0.1470$ [$0.1242$, $0.1698$] uncapped. The statement about which coordinate caps repair a crossing therefore rests on the jurisdiction cap, which removes it outright (shortfall $0.0000$, $49.30$ tokens per round against $92.44$ uncapped), and on the cap-target probe above. The trust level is a threshold on the previous round's provider reputation, so it binds only where deadlines are missed, and at the reported congestion level too few are missed for it to bind anywhere: on all three instances each of the ten seeds reproduces the no-cap run to the last digit on welfare, admitted tokens, median latency, drop rate and the shortfall, at both loads and both mixes.

The cross-node coupling is the one governance object the coordinate-wise model does not cover. The evaluation imposes a joint token budget across two leaves that sit in different domains, a cross-node constraint that is not coordinate-wise and manufactures a crossing block on a tree; a data-residency rule is the example of that class the architecture names. It differs from a per-domain locality predicate, which is a coordinate cap and which the companion deployment runs as one~\cite{loven2026nps}. On T at high load and the uniform mix the coupling puts a shortfall of $0.1440$ [$0.1290$, $0.1590$] where there was none and fails the certificate in ten seeds of ten, and the domain slice restores both.

The slice's cost is measured on the two markets. Paired by seed over the ten seeds with Student-$t$ intervals, slicing the coupled market by domain on T at high load moves welfare by $-1.19$ [$-1.42$, $-0.97$]~a.u.\ and admitted tokens by $-9.87$ [$-10.34$, $-9.39$] per round at the uniform mix, and by $+0.97$ [$+0.74$, $+1.21$] and $-3.96$ [$-4.34$, $-3.58$] at the skewed mix; measured instead against the uncoupled market the slice stands at $-3.30$ [$-3.50$, $-3.11$] of welfare and $-13.06$ [$-13.56$, $-12.55$] of admitted tokens at the uniform mix and $+0.53$ [$+0.24$, $+0.83$] and $-4.76$ [$-5.29$, $-4.22$] at the skewed one, and the coupling itself costs $2.11$ [$1.82$, $2.40$] of welfare and $3.19$ [$2.52$, $3.86$] tokens at the uniform mix and $0.44$ [$0.25$, $0.63$] and $0.80$ [$0.28$, $1.32$] at the skewed one. The slice therefore buys exactness and the certificate at a cost in volume at both mixes; its welfare cost is $1.19$ at the uniform mix, and at the skewed mix it gains welfare. The volume cost follows from the slice's capacity-split rule, which gives each domain a fixed share of the joint budget ($25$ of $50$ tokens). The stranded cross-domain demand, the offered demand the two disjoint markets cannot place between them, is one component of it: $0.2275$ [$0.1978$, $0.2572$] tokens per round at the uniform mix, $2.3$ percent of the volume loss against the coupled market, and $3.2840$ [$3.1041$, $3.4639$] at the skewed mix, $83$ percent of the loss there.

\paragraph*{Governance design-space considerations} The coordinate-cap model tested above captures the structural effect of policy constraints on feasibility, and every level is a built instrument: a jurisdiction predicate constraining which leaves are admissible, a role fraction that reads agent roles, a trust threshold that responds to SLA compliance, and a cross-node coupling imposed as a joint token budget over a leaf pair in two domains. The measured result is that coordinate-wise caps preserve the region on a laminar instance at every level and repair a crossing one exactly when they delete a leaf whose removal leaves the surviving blocks nested, while a coupling that is not coordinate-wise breaks the region on an instance that was laminar and is repaired by slicing per domain. This complements findings in trustworthy distributed systems~\cite{sun2019trust,huang2022trust} by providing a resource- and dependency-aware perspective.

\subsection{Architecture}
\label{app:exp4-detailed}

Four arms are compared: \emph{uncontracted} (per-node pricing), \emph{uncontracted~$+$~EMA} (per-node pricing with EMA smoothing), \emph{contracted without EMA} (a sub-DAG contraction, unsmoothed), and \emph{contracted~$+$~EMA} (both). All four price four leaves. The contraction does change the priced nodes on a path, and with them the reserve floor, which falls from $0.40$ to $0.24$ per task on S; that is part of S's gain.

\begin{figure}[!t]
    \centering
    \includegraphics[width=1\linewidth]{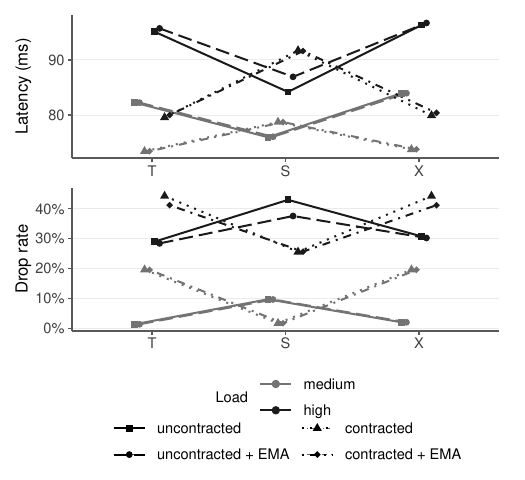}
    \caption{Architecture (operational): the four factorial arms across the three instances and two loads. Panels: median latency (top) and drop rate (bottom). Line shade encodes load, grey for medium and black for high.}
    \label{fig:exp4-operational-supp}
\end{figure}

\begin{figure}[!t]
    \centering
    \includegraphics[width=1\linewidth]{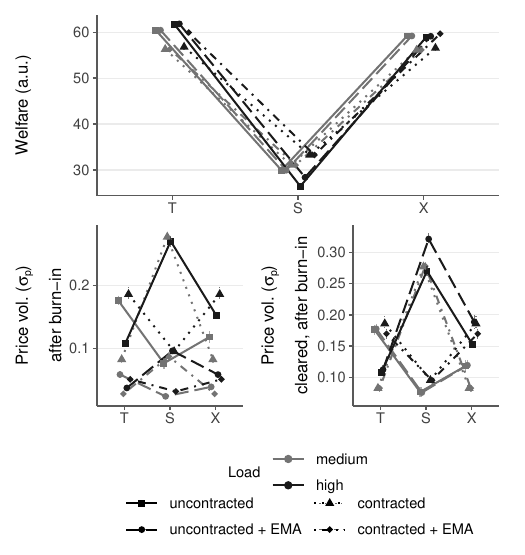}
    \caption{Architecture (economic): welfare (top row) and post-burn-in agent-facing price volatility $\sigma_p$ (bottom row), over the same runs as \cref{fig:exp4-operational-supp}, with the cleared pre-smoothing series beside the agent-facing one.}
    \label{fig:exp4-economic-supp}
\end{figure}

\cref{fig:exp4-operational-supp,fig:exp4-economic-supp} present the factorial. The volatility panels report a filter identity rather than a market effect: the agent-facing series is the cleared series through a low-pass filter, and the ratio of their post-burn-in CVs sits at that filter's white-noise value, $0.2994$ to $0.3382$ against $\sqrt{(1-\beta)/(1+\beta)} = 0.3333$, in every one of the twelve EMA cells. The contraction's own main effect on the same response is not resolved (\cref{app:stat-exp4}).

The welfare and latency panels are not flat. The contraction's pooled main effects resolve on median latency, welfare and the drop rate while smoothing resolves on none of the three (\cref{app:stat-exp4}), and the direction of the contraction's effect differs by instance: it cuts latency on T and X at a cost in admitted volume and a little welfare, and on S it raises volume and welfare and adds latency.

The factorial therefore separates the two integrator components. EMA smoothing is the price filter: the agent-facing series is the cleared series through it, and because the smoothed vector is also the next round's entry price the filter moves the allocation too, in the per-cell directions \cref{app:exp4-ablation} reports. The contraction moves latency, welfare, admitted volume and exactness, and it is not a price-stability mechanism.

\subsection{Architecture $\times$ Governance}
\label{app:exp5-detailed}

This experiment crosses architecture (uncontracted / contracted) $\times$ governance (none / a zero cap) $\times$ instance (T, X, S) $\times$ load (medium / high), yielding $240$~runs.

\begin{figure}[!t]
    \centering
    \includegraphics[width=1\linewidth]{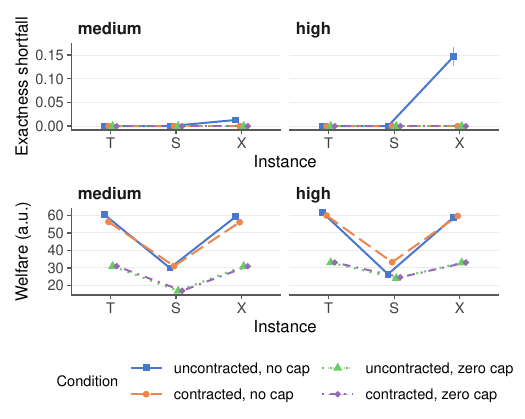}
    \caption{Architecture $\times$ governance interaction: the exactness shortfall (top row) and welfare (bottom row) against the instance, faceted by load with medium on the left and high on the right.}
    \label{fig:exp5-interaction-supp}
\end{figure}

\cref{fig:exp5-interaction-supp} presents exactness and welfare across the four condition combinations, and the admitted volumes follow. On X the contraction and the caps are substitutes for exactness: each alone takes the shortfall from $0.1470$ [$0.1242$, $0.1698$] to $0.0000$, and the two together do no more.

The two remedies do not compound on throughput either. At X under high load the four cells admit $92.44$ [$91.81$, $93.06$], $78.74$ [$78.17$, $79.30$], $49.30$ [$49.21$, $49.38$] and $49.30$ [$49.23$, $49.37$] tokens per round, so the cap sets the volume wherever it applies, and the contraction's welfare gain of $0.78$ [$0.31$, $1.25$]~a.u.\ under no cap falls to $0.10$ [$0.07$, $0.14$] under the cap.

\subsection{Mechanism}
\label{app:exp6-detailed}

Allocation rules are compared in three tiers. Tier~A supplies the reference optima the rest are scored against: the exact allocative optimum on true values at zero queue, and the ex-post prefix optimum, which is the best value-prefix admission through the realised congestion model and so a lower bound on the attainable optimum. Tier~B holds the truth-assuming single-domain planners: value-greedy, which ranks tasks by their bid-time expected value computed on the agents' true values; a Kubernetes-style priority rank, a deployed single-domain scheduler that allocates on declared priorities and resource requests and is in this setting a truth-assuming planner; and the value-ranked posted price, a planner behind a participation screen, which screens at a posted level and packs the participants by expected value under the node capacities. At its tuned level that screen does not bind in five of the six high-load cells, where the arm's allocation coincides with value-greedy's, bit-identical in two of them. Random feasible packing and earliest-deadline-first are the tier's two controls. Tier~C holds the cross-domain mechanisms, which need no single owner of capacity: the t\^{a}tonnement market, evaluated on truthful reports (its joint-misreport test is in \cref{app:exp7b-joint}), two static truth-free posted prices behind the same participation screen, a demand-responsive posted price behind that screen, and the mixed arm that posts the integrator's slice price inside a discovered-price market. The congestion-consistent market is reported beside them as a diagnostic variant; as built it is an undamped cobweb. The arrival-order posted price admits those who clear the price in an arrival order drawn per round from the round's seed, and the deadline-priority posted price admits them earliest deadline first, deadlines being verifiable contract terms; both are tuned over the same eleven levels on the tuning seeds. The demand-responsive posted price starts at the reserve and re-sets its markup $k$ over the reserve each round by $k \leftarrow \max(1, k(1 + \eta(\rho - \rho^{\ast})))$, where $\rho$ is the screened demand, the tokens of the round's tasks whose expected value clears their posted price, over the capacity of the most loaded node; it admits those screened in in arrival order and reads no value. Its step $\eta \in \{0.01, 0.02, 0.05, 0.1, 0.2, 0.4\}$ and target $\rho^{\ast} \in \{0.5, 0.6, 0.7, 0.8, 0.9, 1.0, 1.1, 1.25, 1.5, 2.0\}$, $60$ pairs, are tuned on the same seeds. At four times the reserve the screen admits nothing at all, a consequence of the value support and the decay rate, so that level is live for no posted arm. \cref{fig:exp6-mechanism-supp} shows the ablation levels and \cref{tab:exp6-naive} carries the comparator table.

\begin{table*}[!t]
\centering
\caption{Comparator table: every arm with its four numbers and the condition under which it is feasible. Uncontracted architecture, high load, the reported congestion level, each arm at its own knob tuned on seeds $1$ to $5$ and reported on the twenty evaluation seeds $11$ to $30$. Numeric cells read T / X / S. \emph{W/opt.} is realised welfare over the ex-post prefix optimum, \emph{Alloc.} the allocative ratio to the exact allocative optimum on true values at zero queue, \emph{Frontier} the arm's point on the posted family's welfare-versus-admission curve at its own realised volume, and \emph{Cert.} the incentive certificate. The EDF, random and mixed rows and the \emph{Frontier} column are read off the level grid, seeds $1$ to $10$; the demand-responsive arm has no single level and so no \emph{Frontier} entry, and on S the planners' volume lies outside the range of volumes the posted family covers. The three posted prices share one participation screen at the posted level: the value-ranked one packs those it screens in by expected value, the arrival-order one admits them in arrival order and the deadline-priority one earliest deadline first.}
\label{tab:exp6-naive}
\renewcommand{\arraystretch}{1.1}
\footnotesize
\setlength{\tabcolsep}{4pt}
\begin{tabular}{@{}p{2.1cm} p{3.7cm} p{2.4cm} p{2.4cm} p{1.8cm} p{1.8cm}@{}}
\toprule
\textbf{Tier and arm} & \textbf{Feasible where} & \textbf{W/opt.} & \textbf{Alloc.} & \textbf{Frontier} & \textbf{Cert.} \\
\midrule
\multicolumn{6}{@{}l}{\emph{Tier A, reference optima, unobtainable in deployment, yardsticks only}} \\
Exact allocative optimum & Zero queue and true values, neither available to any arm & --- & $1.000$ & --- & --- \\
Ex-post prefix optimum & The best value-prefix admission through the realised congestion model, a lower bound & $1.000$ & --- & --- & --- \\
\midrule
\multicolumn{6}{@{}l}{\emph{Tier B, truth-assuming single-domain planners}} \\
Value-greedy & Inside one operator's domain, with its own agents, and only where it can read true values & 0.954 / 0.931 / 0.979 & 0.999 / 0.999 / 1.000 & 1.00 / 1.00 / --- & Yes / No / Yes \\
Kubernetes rank & Inside one operator's domain, on declared priorities and resource requests & 0.872 / 0.854 / 0.905 & 0.924 / 0.924 / 0.925 & 0.91 / 0.91 / --- & Yes / No / Yes \\
Posted prices, value-ranked & A planner behind a participation screen: a posted level screens and the participants are packed by value, which needs true values & 0.954 / 0.972 / 0.979 & 0.999 / 0.933 / 1.000 & Reference family & Yes / No / Yes \\
EDF, random & Ablation controls only & 0.871 / 0.852 / 0.905 & 0.924 / 0.923 / 0.925 (random 0.926) & 0.91 / 0.92 / --- & Yes / No / Yes \\
\midrule
\multicolumn{6}{@{}l}{\emph{Tier C, cross-domain mechanisms, no single owner of capacity}} \\
Discovered prices & Cross-domain composition where no operator cedes control; the reserve markup is tuned on the cell's own load & 0.971 / 0.945 / 0.925 & 0.907 / 0.940 / 0.825 & 0.99 / 0.97 / 0.92 & Yes / No / Yes \\
Congestion-consistent prices & Diagnostic variant: as built, the arm is an undamped cobweb & 0.848 / 0.712 / 0.691 & 0.671 / 0.537 / 0.626 & 0.74 / 0.72 / 0.84 & Yes / No / Yes \\
Posted prices, arrival-order & A local marketplace with third-party agents; the level is tuned on the cell's own load (\cref{tab:misset} carries it over from another) & 0.964 / 0.966 / 0.939 & 0.907 / 0.924 / 0.948 & 0.99 / 0.99 / 0.96 & Yes / No / Yes \\
Posted prices, deadline-priority & As above, with deadlines as verifiable contract terms & 0.964 / 0.965 / 0.939 & 0.907 / 0.924 / 0.948 & 0.99 / 0.99 / 0.96 & Yes / No / Yes \\
Posted prices, demand-responsive & A local marketplace as above, the level re-set each round from the demand it screens in; step and target tuned on the cell's own load & 0.962 / 0.973 / 0.964 & 0.900 / 0.897 / 0.933 & --- & Yes / No / Yes \\
Mixed & Cross-domain composition with a compliant integrator; contracted architecture only & 0.972 / 0.992 / 0.979 & 0.830 / 0.831 / 1.000 & 1.01 / 1.01 / 1.00 & Yes / Yes / Yes \\
\midrule
\multicolumn{6}{@{}l}{\emph{Outside the tiers}} \\
Ascending-price market & Diagnostic only: prices nodes, truncated at $15$ iterations, untuned, no incentive run (see text) & 0.980 / 0.974 / 0.974 & 0.990 / 0.977 / 0.954 & --- & --- \\
\bottomrule
\end{tabular}
\\[4pt]
{\footnotesize Median latency at high load and the reported congestion level runs $88.9$ / $96.0$ / $80.0$~ms for the market on T / X / S, against $99.2$ / $104.9$ / $92.3$~ms for value-greedy, $99.2$ / $94.1$ / $92.3$~ms for the tuned value-ranked posted price, $89.9$ / $93.7$ / $91.0$~ms for the tuned arrival-order posted price and $89.8$ / $93.8$ / $91.0$~ms for the tuned deadline-priority posted price and $87.9$ / $88.9$ / $87.1$~ms for the tuned demand-responsive posted price. The mixed arm's row is the contracted cell at its better posted level; every other row is uncontracted. A posted price at four times the reserve admits nothing on any instance, in all $120$ branches of each discipline at this congestion level, one event shared by the three disciplines, so that level is excluded from the comparison. The ascending-price row is the level grid's untuned arm at reserve markup $1$. The certificate column is the test the code checks, a polymatroidal effective region with unit-demand rows, which the crossing instance fails on the uncontracted architecture and passes under the contraction; it is a property of the region, and the incentive guarantee it licenses is the VCG allocation's (\cref{app:exp7b-joint}).}
\end{table*}

\begin{figure}[!t]
    \centering
    \includegraphics[width=1\linewidth]{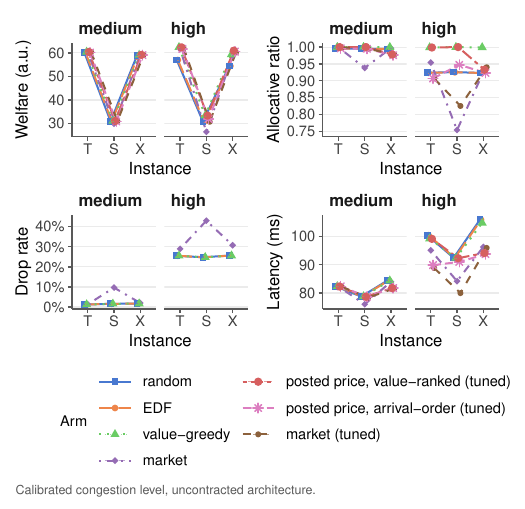}
    \caption{Mechanism: seven arms across the three instances and two load levels, at the reported congestion level and the uncontracted architecture. The four ablation levels are random, EDF, value-greedy and the untuned market; the tuned market, the tuned value-ranked posted price and the tuned arrival-order posted price are read from the tuned table, whose knob was chosen on seeds disjoint from the ablation's and which carries no drop rate, so those three arms have no point in the drop panel. The tuned deadline-priority posted price is not drawn: it lies within $0.0004$ of the arrival-order one in every cell and would overprint it. Panels: welfare (top-left), allocative ratio (top-right), drop rate (bottom-left), median latency (bottom-right). The comparator table of \cref{tab:exp6-naive} carries every arm with its four numbers.}
    \label{fig:exp6-mechanism-supp}
\end{figure}

\cref{fig:exp6-mechanism-supp} presents the four ablation levels under the uncontracted architecture. Each level adds one component: scheduling awareness, then a value model, then price signals.

Prices separate the market from the three rank-based levels only where capacity binds, and how far they separate it is a property of the congestion parametrisation. At the steep sensitivity level and high load the market reaches $11.43$ [$11.16$, $11.85$]~a.u.\ on S against $3.26$ [$3.22$, $3.31$] for value-greedy, $3.03$ [$2.99$, $3.07$] for random and $3.03$ [$2.99$, $3.07$] for EDF. At the reported level that separation closes and reverses on the same instance, value-greedy reaching $33.37$ [$33.33$, $33.42$]~a.u.\ against the market's $26.46$ [$26.23$, $26.69$]. On the allocative side the planner leads at the reported level of \cref{fig:exp6-mechanism-supp} as well: against the exact allocative optimum on true values, value-greedy attains $0.9990$ [$0.9990$, $0.9990$] on T, the market $0.9541$ [$0.9518$, $0.9579$] and the value-ranked posted price at $k = 2$ $0.9125$ [$0.9116$, $0.9136$]. The two objectives disagree in sign, and the distance between them is the unpriced congestion externality.

The frontier ratios place the untuned market against the posted family. At the reported congestion level and high load the market sits below the posted family's welfare frontier at its own realised volume in every congested cell, at $0.930$ and $0.931$ of it under the contraction on T and X and $0.998$ on S, and at $0.994$, $0.966$ and $0.922$ uncontracted; at matched latency the same ratios are $0.898$, $0.899$ and $0.997$ contracted and $0.994$, $0.969$ and $0.844$ uncontracted ($10$~seeds per cell). \cref{app:exp6} adds the cell-by-cell contrasts.

\section{Incentive, Recipe, Workload and Overhead Experiments}
\label{app:round2}

Six experiments are detailed here: incentive compatibility under scalar shading, incentive compatibility under a joint strategy set, the measured agentic workload, recipe heterogeneity against the catalogue boundary, capacity faithfulness, and encapsulation overhead.

\subsection{Incentives: strategic misreporting}
\label{app:exp7-dsic}
We run the VCG mechanism with Clarke-pivot payments gated on a certificate the code checks before the arm runs: the effective region is a polymatroid (its leaf-block rank is submodular) and demand is unit, so value-greedy is the exact welfare maximiser over the region and the mechanism's hypothesis is met. The certificate passes on T and on S and fails on X. One agent at a time misreports its valuations by a shading factor~$\alpha$ while the others stay truthful, and per-agent \emph{regret} (the utility forgone relative to truthful reporting, evaluated at the agent's true valuation net of payment) is averaged over agents and rounds ($10$~seeds, $30$~rounds). \cref{tab:exp7-regret} reports the full curve. The uncertified arm runs at $N = 8$ and a capacity scale of $0.1$, deliberately off the operating point of the other experiments, because Clarke-pivot payments vanish when capacity is slack and the guarantee then has nothing to test.

\begin{table}[htbp]
\centering
\caption{Incentives: mean per-agent regret under valuation shading by factor~$\alpha$, with the incentive certificate each instance carries. Regret is zero at $\alpha=1$ (truthful) by construction and strictly positive for every misreport, the empirical content of the DSIC guarantee (Proposition~2). $10$~seeds and $30$~rounds per instance, $30$~runs.}
\label{tab:exp7-regret}
\footnotesize
\setlength{\tabcolsep}{4pt}
\begin{tabular}{@{}llcccccc@{}}
\toprule
Instance & Cert. & $\alpha{=}0.5$ & $0.7$ & $0.9$ & $1.0$ & $1.1$ & $1.3$ \\
\midrule
T & Passes & 0.0858 & 0.0440 & 0.0030 & 0 & 0.0020 & 0.0102 \\
S & Passes & 0.0867 & 0.0421 & 0.0029 & 0 & 0.0019 & 0.0099 \\
X & Fails  & 0.0473 & 0.0226 & 0.0017 & 0 & 0.0014 & 0.0073 \\
\bottomrule
\end{tabular}
\end{table}

Mean regret grows with the magnitude of the misreport on every instance: averaged over agents and rounds, no scalar shading improves an agent's outcome. Mean payments are $19.82$ on T, $9.76$ on S and $1.02$ on the uncertified arm, whose capacity scale is a tenth of the others'. The average hides the worst case, and the worst case separates the instances: in the joint search below, a uniform shading is the best response in $84$ of the $115$ seed-level best responses on X that pay, so the separation comes from reading the maximum over seeds, and a uniform shading suffices for it.

\subsection{Joint misreport over a strategy set}
\label{app:exp7b-joint}
The shading sweep of \cref{app:exp7-dsic} moves one scalar at a time and averages. A second experiment searches a set of joint misreports and takes the worst case: an agent may report truthfully, shade uniformly at five levels, swap its highest and lowest declared values, perturb each task independently at two levels, or withhold part of its tasks, giving a de-duplicated strategy set of $7$, $10$, $16$ and $24$ distinct reports for one to four tasks and $10 + 2k$ for $k > 4$ tasks, where the independent block falls back to one-at-a-time perturbations, and we take its best response over the whole set. The experiment is swept over contention: the agent count $N \in \{8, 15, 30, 45, 90\}$ crossed with the capacity scale $\{0.1, 0.25, 0.5, 1\}$ on all three instances, ten seeds per point, $620$~runs, the shading arm's operating point ($N = 8$ at capacity scale $0.1$ on X) and the evaluation operating point (full capacity, $N = 90$ on T and X and $45$ on S) both among them. \cref{tab:exp7b-joint} reports it beside the greedy packer's inexactness at each point.

\begin{table*}[!t]
\centering
\caption{Joint misreport against contention: best-response gain over the joint strategy set, the gain being the utility a deviating agent obtains over truthful reporting at its true valuation net of payment, under the VCG allocation with Clarke payments and under the discovered-price market. For each point, each seed's worst case over agents and rounds is taken; \emph{mean} and \emph{max} are the mean and maximum of those over the ten seeds. \emph{Inexact} is the share of rounds in which the greedy packer falls short of the exact leaf-block optimum. The certificate fails on X at every point and passes on T and S at every point; on T and S the VCG inexactness is zero throughout and only the maximum is shown. The market's gain on S is $0$ at every point and is not shown; its zeros are exact, where VCG's floor values are floating-point residue. Of the $171$ market runs with a positive gain, the best deviation is a uniform over-report in $137$ ($87$ at $\times1.1$, $50$ at $\times1.3$), a uniform under-report in $4$ (all at $\times0.9$) and a mixed or per-task report in $30$. $620$~VCG runs and $600$~market runs.}
\label{tab:exp7b-joint}
\scriptsize
\setlength{\tabcolsep}{3pt}
\begin{tabular}{@{}rr ccccc cccc@{}}
\toprule
 & & \multicolumn{5}{c}{\textbf{VCG}} & \multicolumn{4}{c}{\textbf{Market}} \\
\cmidrule(lr){3-7}\cmidrule(l){8-11}
$N$ & \textbf{Cap.} & \textbf{X mean} & \textbf{X max} & \textbf{X inexact} & \textbf{T max} & \textbf{S max} & \textbf{T mean} & \textbf{T max} & \textbf{X mean} & \textbf{X max} \\
\midrule
8 & 0.10 & $0.154$ & $0.244$ & $0.060$ & $5.6\times10^{-16}$ & $4.4\times10^{-16}$ & $0.014$ & $0.043$ & $0.004$ & $0.019$ \\
8 & 0.25 & $5.3\times10^{-16}$ & $8.9\times10^{-16}$ & $0.000$ & $8.9\times10^{-16}$ & $5.6\times10^{-16}$ & $0$ & $0$ & $0$ & $0$ \\
8 & 0.50 & $5.3\times10^{-16}$ & $8.9\times10^{-16}$ & $0.000$ & $8.9\times10^{-16}$ & $8.9\times10^{-16}$ & $0$ & $0$ & $0$ & $0$ \\
8 & 1.00 & $5.3\times10^{-16}$ & $8.9\times10^{-16}$ & $0.000$ & $8.9\times10^{-16}$ & $8.9\times10^{-16}$ & $0$ & $0$ & $0$ & $0$ \\
15 & 0.10 & $0.111$ & $0.207$ & $0.113$ & $9.2\times10^{-16}$ & $4.4\times10^{-16}$ & $0.071$ & $0.135$ & $0.042$ & $0.064$ \\
15 & 0.25 & $0.096$ & $0.199$ & $0.030$ & $1.8\times10^{-15}$ & $8.9\times10^{-16}$ & $0.0008$ & $0.008$ & $0$ & $0$ \\
15 & 0.50 & $9.8\times10^{-16}$ & $1.8\times10^{-15}$ & $0.000$ & $1.8\times10^{-15}$ & $1.8\times10^{-15}$ & $0$ & $0$ & $0$ & $0$ \\
15 & 1.00 & $9.8\times10^{-16}$ & $1.8\times10^{-15}$ & $0.000$ & $1.8\times10^{-15}$ & $1.8\times10^{-15}$ & $0$ & $0$ & $0$ & $0$ \\
30 & 0.10 & $0.030$ & $0.112$ & $0.083$ & $9.2\times10^{-16}$ & $4.4\times10^{-16}$ & $0.116$ & $0.169$ & $0.091$ & $0.120$ \\
30 & 0.25 & $0.148$ & $0.329$ & $0.177$ & $1.7\times10^{-15}$ & $8.9\times10^{-16}$ & $0.037$ & $0.067$ & $0.008$ & $0.034$ \\
30 & 0.50 & $0.093$ & $0.160$ & $0.023$ & $1.9\times10^{-15}$ & $1.8\times10^{-15}$ & $0$ & $0$ & $0$ & $0$ \\
30 & 1.00 & $1.8\times10^{-15}$ & $1.8\times10^{-15}$ & $0.000$ & $1.8\times10^{-15}$ & $1.9\times10^{-15}$ & $0$ & $0$ & $0$ & $0$ \\
45 & 0.10 & $0.020$ & $0.043$ & $0.103$ & $8.9\times10^{-16}$ & $4.4\times10^{-16}$ & $0.156$ & $0.211$ & $0.121$ & $0.189$ \\
45 & 0.25 & $0.059$ & $0.147$ & $0.150$ & $1.8\times10^{-15}$ & $8.9\times10^{-16}$ & $0.082$ & $0.121$ & $0.047$ & $0.096$ \\
45 & 0.50 & $0.223$ & $0.328$ & $0.150$ & $3.5\times10^{-15}$ & $1.8\times10^{-15}$ & $0.006$ & $0.032$ & $0$ & $0$ \\
45 & 1.00 & $3.6\times10^{-15}$ & $3.6\times10^{-15}$ & $0.000$ & $3.6\times10^{-15}$ & $3.6\times10^{-15}$ & $0$ & $0$ & $0$ & $0$ \\
90 & 0.10 & $0.012$ & $0.026$ & $0.097$ & $8.9\times10^{-16}$ & $4.4\times10^{-16}$ & $0.194$ & $0.261$ & $0.169$ & $0.214$ \\
90 & 0.25 & $0.025$ & $0.055$ & $0.160$ & $1.8\times10^{-15}$ & $9.4\times10^{-16}$ & $0.171$ & $0.231$ & $0.084$ & $0.145$ \\
90 & 0.50 & $0.034$ & $0.115$ & $0.133$ & $2.0\times10^{-15}$ & $1.8\times10^{-15}$ & $0.062$ & $0.071$ & $0.029$ & $0.036$ \\
90 & 1.00 & $0.151$ & $0.364$ & $0.157$ & $7.1\times10^{-15}$ & $3.6\times10^{-15}$ & $0.001$ & $0.008$ & $0$ & $0$ \\
\bottomrule
\end{tabular}
\end{table*}

On the laminar instances T and S the worst-case gain stays at the floating-point floor of the utility arithmetic at every one of the twenty points, at most $7.1\times10^{-15}$, roughly fifteen orders of magnitude below a single payment, with the certificate passing in every run. On X it is positive at the thirteen points where the greedy packer is inexact and at the floor at the seven where it is exact. The incentive guarantee therefore fails on the crossing instance where exactness fails, across contention. At the evaluation point the seed-mean gain is $0.151$ [$0.057$, $0.245$], the maximum over seeds $0.364$, the inexactness $0.157$, and a non-truthful member is the argmax in ten seeds of ten; at the shading arm's point it is $0.1538$ [$0.0922$, $0.2155$], $0.244$ and $0.060$, with a non-truthful argmax in eight seeds of ten. The Clarke clamp never acted, the count of negative externalities being zero in all $620$~runs. Payments are nonzero, the code's reading of a binding payment constraint, in every seed at $43$ of the $60$ instance-points, and the independent fallback is exercised on $2.47\%$ of agent-rounds.

\paragraph*{The discovered-price market} The same strategy set, points and seeds are run against the t\^{a}tonnement market of the mechanism comparison, which clears on the reported values and charges each winner the node-price cost of its bundle ($600$~runs). Under contention its agents gain by misreporting, on the laminar tree and on the crossing instance alike. On T the worst-case gain reaches $0.261$, at $N = 90$ and capacity scale $0.1$, where its seed mean is $0.194$ [$0.169$, $0.220$]; it rises with the agent count and falls with capacity, and is positive at capacity scale $0.1$ from $N = 8$, at $0.25$ from $N = 15$, at $0.5$ from $N = 45$ and at full capacity at $N = 90$, twelve of the twenty points. On X it reaches $0.214$ at the same point, seed mean $0.169$ [$0.148$, $0.190$], and is positive at nine points. The market is manipulable most often by over-reporting values by $10$ to $30$ percent, and more often as capacity tightens: $171$ of its $600$ runs carry a positive gain, $89$, $57$, $23$ and $2$ of the $150$ at capacity scales $0.1$, $0.25$, $0.5$ and $1$, and in $137$ of the $171$ the best deviation is a uniform over-report. On S the gain is $0$ at every point: the fan's clearing prices end at their floor, where winners pay nothing. Under the VCG allocation the same strategy set finds the floor on T and S at every point and a positive gain on X where the packer is inexact. Truthfulness is therefore a property of the VCG allocation on a laminar region; the price-discovering market of the mechanism comparison is manipulable under contention, on laminar and crossing instances alike, and the comparison evaluates it on truthful reports.
\subsection{Agentic Workload}
\label{app:exp8-agentic}
Two multi-step, tool-using LLM agents are instrumented on a local model, \texttt{mistral:7b-instruct-q4\_K\_M}. The first is a single recorded agent, plan, two parallel tool calls, then aggregate; its one sink makes its recorded leaf-block family laminar. The second is a composition, built by us, of two recordings that share a tool stage; its family crosses only through that sharing, so the family the contrast varies is the composition's. Per-stage latencies, per-node demand and the dependency structure are read from each execution. The measured token counts, expressed on one token unit ($58.2$ tokens, the first recording's smallest stage) for both workloads, give per-task demand weights of $1.43$, $1.00$, $1.65$ and $3.41$ on the single recording and $1.43$, $1.00$, $1.54$, $3.41$, $1.59$, $3.75$, $5.97$ and $5.25$ on the composition and the measured stage latencies put the zero-queue critical path at $3875$ and $7411$~ms, from which the deadline set and the latency-decay rate are re-derived so that the fraction of task value surviving at the zero-queue critical path matches the synthetic configuration's. The two workloads differ in more than the family: one leaf against three, four service nodes against eight, a different split of each tier's capacity across the stages, and the deadlines and decay rate each critical path implies. The contrast below is therefore not read at a matched operating point, and the family is not credited with it alone. The allocation is then driven by each measured profile ($N=200$): the demands are measured and the market is simulated. \cref{tab:exp8-agentic} reports the two workloads.

\begin{table}[htbp]
\centering
\caption{Agentic workload: the single recorded agent, whose family is laminar, and the composition of two recordings built by us, whose family crosses by the shared tool stage, at medium and high load, uncontracted architecture without the filter. $N=200$, $10$~seeds per cell; the grid also runs the filtered arm, $80$~runs in all. \textit{Rec.} is the recording (laminar or crossing), \textit{Lat.} the median latency, \textit{Drop} the drop rate, \textit{Welf.} welfare (a.u.) and \textit{Cert.} the incentive certificate; \textit{Alloc.} is the allocative ratio to the exact optimum on true values and \textit{Short.} is the shortfall, the share of rounds in which the greedy packer falls short of the round's exact leaf-block optimum.}
\label{tab:exp8-agentic}
\scriptsize
\setlength{\tabcolsep}{3pt}
\begin{tabular}{@{}llcccccc@{}}
\toprule
Rec. & Load & Lat.\ (ms) & Drop & Welf. & Alloc. & Short. & Cert. \\
\midrule
Lam.\ & medium & 3890.8 & 0.5526 & 80.53 & 0.9940 & 0.0000 & Passes \\
Lam.\ & high   & 3896.7 & 0.7301 & 75.88 & 0.9001 & 0.0000 & Passes \\
Cross. & medium & 3929.9 & 0.7324 & 60.18 & 0.7433 & 0.2290 & Fails \\
Cross. & high   & 4183.6 & 0.8670 & 47.70 & 0.5582 & 0.0260 & Fails \\
\bottomrule
\end{tabular}
\end{table}

Both workloads contend, and the offered loads are measured on each: on the composition at high load the utilisation is $0.3166$ and the drop rate $0.8670$, and $99.92$ and $99.97$ percent of admitted tasks meet their deadlines without and with the filter.

The structural measurement carries across. The market's allocative ratio to the exact optimum is $0.5582$ [$0.5409$, $0.5754$] on the composition against $0.9001$ [$0.8776$, $0.9226$] on the single recording at high load. The greedy packer falls short of the round's exact leaf-block optimum in $0.2290$ [$0.2139$, $0.2441$] of the composition's rounds at medium load and $0.0260$ [$0.0174$, $0.0346$] at high load, where admission rations below the crossing more often, and in none of the single recording's; where it falls short it loses little, its mean ratio to that optimum being $0.995$ and $0.999$.

The incentive certificate passes on the single recording in all forty branches and fails on the composition in all forty, with each family computed from the call graphs. The structural and incentive results therefore hold on measured demand, under the matching caveat above.

\subsection{Recipe Heterogeneity and the Catalogue Boundary}
\label{app:exp11-heterogeneity}
This is the recipe experiment of \mainname. \cref{prop:catalogue-exact} makes the exactness of the advertised catalogue turn on whether slices that share an internal resource consume it in fixed proportions. This experiment instantiates that boundary. Four arms cross recipe heterogeneity with encapsulation: homogeneous and heterogeneous recipes, each run with per-resource pricing and with an encapsulated slice price, at catalogue caps $6$, $9$ and $12$. Aggregate per-tier demand is matched across arms, so only heterogeneity changes between a matched pair: realised aggregate demand is $11.9685$ on both homogeneous arms and $11.9670$ to $11.9700$ on the heterogeneous ones, and the truncation rate averages $0.022$ in every one of the twelve arm$\times$cap groups, with a per-branch range of $0$ to $0.05$.

\begin{table*}[tp]
\centering
\caption{Recipe heterogeneity: four arms at three catalogue caps, $10$~seeds per cell, $120$~runs. \textit{Cap} is the catalogue cap and \textit{Arm} the recipe-by-pricing arm. \textit{Price CV} is the coefficient of variation of the agent-facing price series over rounds. \textit{Greedy exact} is the greedy allocation's welfare against the exact optimum, and \textit{Adm.\ exact} the admitted set's welfare against that same exact optimum. \textit{Welf.\ ratio} is realised welfare over the ex-post prefix optimum, a different denominator from the two exactness ratios and so a different quantity from \textit{Greedy exact}. \textit{Served/adm.} is the share of admitted tasks actually served, and \textit{Clearing} the share of offered tasks admitted and served, one minus the drop rate.}
\label{tab:exp11-hetero}
\footnotesize
\setlength{\tabcolsep}{3pt}
\begin{tabular}{@{}llcccccc@{}}
\toprule
Cap & Arm & Price CV & Greedy exact & Adm.\ exact & Welf.\ ratio & Served/adm. & Clearing \\
\midrule
6  & Homogeneous       & 0.722 & 1.0000 & 0.229 & 0.211 & 0.417 & 0.136 \\
6  & Homog., encaps.   & 0.247 & 1.0000 & 0.318 & 0.272 & 0.421 & 0.184 \\
6  & Heterogeneous     & 0.802 & 0.9741 & 0.285 & 0.245 & 0.493 & 0.164 \\
6  & Hetero., encaps.  & 0.231 & 0.9741 & 0.404 & 0.245 & 0.462 & 0.224 \\
\midrule
9  & Homogeneous       & 0.670 & 1.0000 & 0.626 & 0.604 & 0.556 & 0.518 \\
9  & Homog., encaps.   & 0.192 & 1.0000 & 0.752 & 0.686 & 0.498 & 0.577 \\
9  & Heterogeneous     & 0.655 & 0.9846 & 0.666 & 0.598 & 0.586 & 0.548 \\
9  & Hetero., encaps.  & 0.174 & 0.9846 & 0.860 & 0.646 & 0.538 & 0.659 \\
\midrule
12 & Homogeneous       & 0.554 & 1.0000 & 0.880 & 0.892 & 0.679 & 0.815 \\
12 & Homog., encaps.   & 0.170 & 1.0000 & 0.949 & 0.942 & 0.615 & 0.854 \\
12 & Heterogeneous     & 0.558 & 0.9931 & 0.869 & 0.856 & 0.697 & 0.807 \\
12 & Hetero., encaps.  & 0.162 & 0.9931 & 0.974 & 0.875 & 0.657 & 0.877 \\
\bottomrule
\end{tabular}
\end{table*}

Two things separate in \cref{tab:exp11-hetero}. At matched aggregate demand, recipe heterogeneity on its own does not move price dispersion: heterogeneous minus homogeneous price CV is $+0.080$ [$-0.010$, $+0.211$] at cap $6$, $-0.015$ [$-0.034$, $+0.005$] at cap $9$ and $+0.004$ [$-0.003$, $+0.011$] at cap $12$, three intervals that all span zero. The encapsulated arm, which carries the EMA filter, does move it, and by a wide margin: the encapsulated minus per-resource contrast is $-0.571$ [$-0.623$, $-0.528$], $-0.481$ [$-0.510$, $-0.451$] and $-0.396$ [$-0.420$, $-0.371$] at the three caps on the heterogeneous arms, and $-0.476$, $-0.477$ and $-0.383$ on the homogeneous ones. That reduction is the filter's reporting identity (\cref{app:exp4-ablation}).

Where heterogeneity does show up is in allocative exactness, value-greedy's welfare against the enumerated integer optimum. With one recipe the feasibility system collapses to a cardinality constraint, on which descending-value greedy is the optimum by construction, so the $1.0000$ the two one-recipe arms carry in all $60$ seed-cap cells is an identity. The measured side is the heterogeneous one, whose two recipes are not proportional and so fall outside \cref{prop:catalogue-exact}'s hypothesis: the ratio of the greedy allocation's welfare to the exact optimum is strictly below $1$ in all $60$ of its cells, $0.9741$ at cap $6$, $0.9846$ at cap $9$ and $0.9931$ at cap $12$, with per-cell ranges $0.9665$--$0.9852$, $0.9758$--$0.9897$ and $0.9887$--$0.9969$. The boundary is therefore visible, small, and closing as the catalogue cap rises.

On the heterogeneous arms the encapsulated arm also admits more than it serves: its clearing fraction is higher than the matched per-resource arm's at every cap ($0.224$ against $0.164$, $0.659$ against $0.548$, $0.877$ against $0.807$) while its served-among-admitted share is lower ($0.462$ against $0.493$, $0.538$ against $0.586$, $0.657$ against $0.697$). No over-commitment is measured in this experiment, and for these non-proportional recipes no polymatroidal catalogue region exists: the encapsulated arm admits against a mean-recipe slice capacity, and admitting against that capacity produces the pair of movements.

The over-commitment factor the advertised region carries is predicted in closed form, $\rho = 2 - 1/\lambda_2$ for a two-slice catalogue with $\lambda_2$ the ratio of the two recipes' internal intensities, which gives $1.000$, $1.200$, $1.333$, $1.500$ and $1.750$ at $\lambda_2 = 1.00$, $1.25$, $1.50$, $2.00$ and $4.00$ (\cref{ex:unequal-recipe} works the $\lambda_2 = 2$ case through). A simulated arm measures it. Two proportional recipes share one component, $B = \lambda_2 A$ on every tier, at an internal capacity of $\kappa = 60$ light-recipe units; each round offers a Poisson number of tasks with mean $2\kappa$, draws its heavy share uniformly on $[0, 1]$, and admits value-greedily on the catalogue interface's advertised region, and a round's factor is the admitted recipes' delivered load over $\kappa$ in the rounds where the joint count binds. Over $10$ seeds of $200$ rounds the largest measured factor equals the prediction at every level and never exceeds it; the mean over binding rounds is $1.000$, $1.111$, $1.200$, $1.334$ and $1.596$, over $2{,}000$, $1{,}812$, $1{,}672$, $1{,}514$ and $1{,}262$ binding rounds. The advertised extreme point is thus reached in simulation, and for $\lambda_2 > 1$ it is infeasible in the deliverable set. The delivery order of the closed-form construction, light slice first, decides which slices go unserved once the advertised point is delivered: the unserved fraction over those five levels runs $0.000$, $0.160$, $0.222$, $0.250$ and $0.188$, and among the heavy slices alone $0.000$, $0.200$, $0.333$, $0.500$ and $0.750$. Neither column is a cost curve. The unserved fraction is not monotone in $\lambda_2$, because a heavier slice draws fewer heavy commitments in the first place, and a different delivery order would move the incidence elsewhere.

The conservative inner region of \cref{prop:inner-exposure} prices the trade on the non-proportional two-resource, three-slice instance of \cref{app:catalogue}, which the pipeline's control runs at internal capacities of $2$, where every quantity is an integer and the proposition's integrality hypothesis is met. The per-slice reading, each slice at its own singleton catalogue rank as though the slices were independent goods, advertises $6$ slice units against a deliverable integral optimum of $4$, found by enumeration, and the point it advertises consumes twice each internal resource's capacity, an over-commitment factor of $2$. The largest equal-sided integral box, $(1, 1, 1)$, advertises $3$, consumes each resource exactly to capacity, a factor of $1$, and forgoes $1$ unit of deliverable welfare. A certificate independent of that arithmetic confirms the containment: for each exposed box and each internal resource a linear programme maximises the resource's load over the box, and the inner box's maxima stay within capacity, so it lies inside the realisable set, while the per-slice reading's reach $4$ against a capacity of $2$ on both resources, so it does not. The trade is a factor-two over-commitment removed for $1$ of the $4$ deliverable units.

\subsection{Faithfulness under Capacity Over-Advertisement}
\label{app:exp10-faithfulness}
This experiment varies the scalar the integrator advertises at a fixed cluster: no contraction, the cluster's bottleneck scalar, and the cluster's node-split maximum flow, the advertised routed composition. The first two expose a scalar the cluster can deliver, the bottleneck being exact on S and an inner exposure on T and X; the third exposes one it cannot, because under the fork-join counting every token loads every cluster node on its path, so a routed reading of the cluster's capacity is not deliverable. The probe therefore exercises condition~\emph{(i)} of \cref{prop:encapsulation-supp}, the advertised scalar, and not condition~\emph{(ii)}, which holds on S and fails on T and X in every arm. We measure the over-commitment each construction incurs against the true instance, and what it costs. \cref{tab:exp10-faithfulness} reports the outcome.

\begin{table}[htbp]
\centering
\caption{Capacity-faithfulness probe: market outcomes under the three interface constructions, no contraction, the cluster's bottleneck scalar (\emph{Inner}: exact on S, an inner exposure on T and X) and the advertised routed composition (\emph{Maxflow}). Uniform leaf mix, high load, the reported congestion level, $10$~seeds per cell; the probe runs both leaf mixes, $180$~runs, and the skewed half is reported in the text. \textit{Over-comm.} is the per-round mean excess of admitted tokens over what the true instance can deliver, clipped at zero, and \textit{Serv/adm} is the share of admitted tasks actually served.}
\label{tab:exp10-faithfulness}
\scriptsize
\setlength{\tabcolsep}{3pt}
\begin{tabular}{@{}llcccccc@{}}
\toprule
Instance & Interface & Over-comm. & Serv/adm & Tokens & Welfare & Lat.\ (ms) \\
\midrule
T & None    & 0.0000 & 1.0000 & 94.99 & 61.65 & 95.1 \\
T & Inner   & 0.0000 & 1.0000 & 74.61 & 56.78 & 79.6 \\
T & Maxflow & 1.4890 & 1.0000 & 99.04 & 61.70 & 99.9 \\
X & None    & 0.0000 & 1.0000 & 92.44 & 58.89 & 96.3 \\
X & Inner   & 0.0000 & 1.0000 & 74.61 & 56.55 & 79.9 \\
X & Maxflow & 2.5685 & 1.0000 & 99.04 & 59.03 & 105.2 \\
S & None    & 0.0000 & 1.0000 & 36.84 & 26.46 & 84.2 \\
S & Inner   & 0.0000 & 1.0000 & 49.40 & 33.26 & 91.7 \\
S & Maxflow & 17.4540 & 1.0000 & 67.43 & 40.89 & 94.1 \\
\bottomrule
\end{tabular}
\end{table}

The bottleneck scalar over-commits by $0.0000$ [$0.0000$, $0.0000$] on all three instances and at both leaf mixes, as its construction guarantees. The advertised routed composition over-commits on every one: at the skewed mix, the design point at which $e_1$'s charge exceeds its capacity (\cref{app:simulation}), by $6.327$ [$6.156$, $6.498$] tokens per round on both T and X, where $e_1$ carries the same share on each; at uniform leaf shares, where that charge meets capacity, by $1.489$ [$1.392$, $1.586$] on T and $2.569$ [$2.434$, $2.703$] on X, a ratio of $1.73$ with non-overlapping intervals, because X's crossing gives it two half-share internal nodes where T has one; and $17.454$ [$16.966$, $17.942$] on S, where the advertised $150$ tokens per round, of which the device node passes $100$, draw $67.43$ [$66.94$, $67.91$] admitted tasks against a deliverable $50$.

At the reported congestion level the over-commitment is paid in latency. Every admitted task is served, the share served among admitted being $1.0000$ in every cell of the probe, and the excess load queues: median latency rises from $95.1$ [$94.8$, $95.4$]~ms with no contraction and $79.6$ [$79.2$, $80.0$] under the inner region to $99.9$ [$99.8$, $100.1$] under the routed composition on T, from $96.3$ [$96.0$, $96.7$] and $79.9$ [$79.6$, $80.3$] to $105.2$ [$105.0$, $105.4$] on X, and at the skewed mix from $91.9$ [$91.6$, $92.2$] and $76.2$ [$76.1$, $76.4$] to $103.8$ [$103.7$, $104.0$] on T. Welfare does not fall with it: the routed composition reaches $61.70$ [$61.62$, $61.78$] on T against $61.65$ [$61.45$, $61.86$] with no contraction and $56.78$ [$56.04$, $57.51$] under the inner region, and $40.89$ [$40.65$, $41.13$] on S against $26.46$ [$26.17$, $26.75$] and $33.26$ [$33.19$, $33.34$], because the extra admitted tasks still meet their deadlines under this queue term. The probe therefore separates the construction that never over-commits from the ones that do, and at this level an over-stated scalar costs latency.

\subsection{Encapsulation-Overhead Sweep}
\label{app:exp12-overhead}
This sweep runs in the per-tier identical-bundle environment and its numbers are scoped to it. The integrator's cross-layer abstraction adds an additive per-slice protocol-translation latency $\delta_{\mathrm{enc}}$ on the slice path only; the na\"{i}ve baseline has no integrator and incurs none. We sweep $\delta_{\mathrm{enc}} \in \{0, 25, 50\}$~ms and \cref{tab:exp12-overhead} reports the effect.

\begin{table}[htbp]
\centering
\caption{Encapsulation-overhead sweep: additive per-slice protocol-translation latency $\delta_{\mathrm{enc}}$ on the slice path. High load, $10$~seeds per cell, $60$~runs. Rows are labelled by the per-tier demand profile, the vector of device, edge and cloud demand weights that is the channel through which the dependency graph enters this environment.}
\label{tab:exp12-overhead}
\footnotesize
\setlength{\tabcolsep}{4pt}
\begin{tabular}{@{}llcccc@{}}
\toprule
Demand profile & $\delta_{\mathrm{enc}}$ (ms) & Lat.\ (ms) & p95 (ms) & Drop & Welfare \\
\midrule
$(2,4,5)$ & 0  & 301.9 & 349.4 & 0.2914 & 29.12 \\
$(2,4,5)$ & 25 & 320.2 & 370.9 & 0.2985 & 26.19 \\
$(2,4,5)$ & 50 & 341.1 & 394.9 & 0.3073 & 23.36 \\
$(5,4,3)$ & 0  & 330.4 & 380.9 & 0.4269 & 14.83 \\
$(5,4,3)$ & 25 & 346.7 & 399.5 & 0.4337 & 13.34 \\
$(5,4,3)$ & 50 & 365.5 & 421.1 & 0.4438 & 11.99 \\
\bottomrule
\end{tabular}
\end{table}

The overhead lands on latency and welfare, leaving coverage untouched. At $50$~ms against $0$, median latency under the $(2,4,5)$ profile rises by $39.24$~ms [$+37.25$, $+41.13$] and welfare falls by $5.760$ [$-5.926$, $-5.588$], from $29.115$ to $23.355$, a fall of $19.8\%$. The drop rate rises by $0.0159$ [$+0.0132$, $+0.0188$], under two points. The $(5,4,3)$ profile behaves the same way: latency $+35.12$~ms [$+32.11$, $+38.28$], welfare $-2.841$ [$-2.986$, $-2.709$], drop rate $+0.0169$ [$+0.0128$, $+0.0217$].

These rows are per-tier and are not differenced against the node-level lead. The node substrate carries its own sweep, \cref{tab:node-overhead}: the overhead is charged on each task whose leaf the contracted cluster exports ($l_1$ to $l_3$ on T and X, all four leaves on S), both in execution and in the bid-time latency estimate, and the uncontracted arm pays none; high load, both congestion levels, each contracted arm paired by seed against the uncontracted one over ten seeds with Student-$t$ intervals. At the reported level the contraction's median-latency lead on T and X is $16$ to $17$ percent of the uncontracted median at zero overhead ($15.1$ to $16.4$~ms), $4$ to $7$ percent at $25$~ms ($3.7$ to $6.3$~ms), and negative at $50$~ms; by linear interpolation between the bracketing levels it reaches zero at $31.3$ to $38.9$~ms on T and $32.6$ to $41.0$~ms on X. The welfare lead is negative at every positive overhead, $-11.40$ to $-16.89$ at $25$~ms and $-17.32$ to $-23.85$ at $50$~ms, and admitted volume moves with the overhead, since the charge enters the bid-time estimate: the token lead falls from $13.70$ to $20.38$ fewer tokens per round at zero overhead to $21.94$ to $32.71$ fewer at $50$~ms. On S the contraction is slower at every level, by $9$ to $65$ percent of the uncontracted median ($7.4$ to $54.5$~ms); it gains $6.8$ of welfare at zero overhead, trails the uncontracted arm by $0.38$ to $0.41$ at $25$~ms and by $6.39$ to $6.59$ at $50$~ms, and admits $9.0$ to $12.6$ more tokens per round. At the steep sensitivity level the lead is larger by an order of magnitude and falls from $68$ percent of the uncontracted median at zero overhead ($212.7$ to $217.4$~ms) to $60$ to $63$ percent at $50$~ms, with no break-even inside the swept range; the welfare lead falls from $31.9$ to $35.3$ to $15.9$ to $19.6$, and on S the contraction trails at every level, by $40$ to $45$ percent of the uncontracted median ($85$ to $95$~ms) and $4.4$ to $5.0$ of welfare.

\begin{table*}[!t]
\centering
\caption{Encapsulation overhead on the node substrate: the lead of each contracted arm over the uncontracted one at $0$, $25$ and $50$~ms of translation overhead, charged on each task whose leaf the contracted cluster exports, in execution and in the bid-time estimate, at the reported congestion level (clamp $0.95$, coefficient $0.75$) and at the steep sensitivity level (clamp $0.99$, coefficient $2$). A positive lead is better for the contracted arm: the latency lead is the uncontracted median latency less the contracted one, the welfare and token leads the contracted value less the uncontracted one. \emph{Zero at} is the overhead at which the latency lead reaches zero, by linear interpolation between the swept levels; \emph{none} means no crossing inside $0$ to $50$~ms. High load, paired by seed, Student-$t$ $95\%$ intervals, $n = 10$ per row, $420$~runs.}
\label{tab:node-overhead}
\footnotesize
\setlength{\tabcolsep}{3pt}
\begin{tabular}{@{}lllrlllr@{}}
\toprule
\textbf{Level} & \textbf{Instance} & \textbf{Arm} & $\delta_{\mathrm{enc}}$ (ms) & \textbf{Latency lead (ms)} & \textbf{Welfare lead} & \textbf{Token lead} & \textbf{Zero at} \\
\midrule
reported & T & contracted & 0 & $15.5$ [$15.2$, $15.9$] & $-4.88$ [$-5.56$, $-4.20$] & $-20.38$ [$-21.25$, $-19.52$] & $38.9$ \\
reported & T & contracted & 25 & $5.5$ [$5.2$, $5.8$] & $-16.89$ [$-17.35$, $-16.43$] & $-27.88$ [$-28.61$, $-27.15$] &  \\
reported & T & contracted & 50 & $-4.4$ [$-4.9$, $-3.8$] & $-23.85$ [$-24.25$, $-23.46$] & $-32.71$ [$-33.59$, $-31.84$] &  \\
reported & T & contracted $+$ EMA & 0 & $15.1$ [$14.8$, $15.4$] & $-1.73$ [$-2.04$, $-1.43$] & $-16.25$ [$-16.83$, $-15.67$] & $31.3$ \\
reported & T & contracted $+$ EMA & 25 & $3.7$ [$3.3$, $4.1$] & $-13.99$ [$-14.32$, $-13.66$] & $-23.26$ [$-23.95$, $-22.57$] &  \\
reported & T & contracted $+$ EMA & 50 & $-10.9$ [$-11.3$, $-10.6$] & $-19.95$ [$-20.22$, $-19.68$] & $-24.49$ [$-25.12$, $-23.87$] &  \\
\midrule
reported & X & contracted & 0 & $16.4$ [$15.9$, $16.9$] & $-2.34$ [$-3.13$, $-1.55$] & $-17.83$ [$-18.87$, $-16.79$] & $41.0$ \\
reported & X & contracted & 25 & $6.3$ [$5.9$, $6.8$] & $-14.28$ [$-14.79$, $-13.77$] & $-25.33$ [$-26.19$, $-24.47$] &  \\
reported & X & contracted & 50 & $-3.6$ [$-4.2$, $-2.9$] & $-21.19$ [$-21.60$, $-20.78$] & $-30.16$ [$-31.12$, $-29.20$] &  \\
reported & X & contracted $+$ EMA & 0 & $15.9$ [$15.4$, $16.4$] & $0.78$ [$0.31$, $1.25$] & $-13.70$ [$-14.58$, $-12.82$] & $32.6$ \\
reported & X & contracted $+$ EMA & 25 & $4.5$ [$4.0$, $5.0$] & $-11.40$ [$-11.84$, $-10.95$] & $-20.71$ [$-21.61$, $-19.80$] &  \\
reported & X & contracted $+$ EMA & 50 & $-10.2$ [$-10.7$, $-9.7$] & $-17.32$ [$-17.58$, $-17.06$] & $-21.94$ [$-22.62$, $-21.26$] &  \\
\midrule
reported & S & contracted & 0 & $-7.5$ [$-7.8$, $-7.2$] & $6.81$ [$6.50$, $7.11$] & $12.56$ [$12.14$, $12.99$] & none \\
reported & S & contracted & 25 & $-31.1$ [$-31.7$, $-30.6$] & $-0.41$ [$-0.72$, $-0.09$] & $11.59$ [$11.04$, $12.15$] &  \\
reported & S & contracted & 50 & $-53.7$ [$-54.1$, $-53.3$] & $-6.59$ [$-6.90$, $-6.27$] & $8.98$ [$8.51$, $9.45$] &  \\
reported & S & contracted $+$ EMA & 0 & $-7.4$ [$-7.7$, $-7.1$] & $6.82$ [$6.49$, $7.14$] & $12.51$ [$12.05$, $12.98$] & none \\
reported & S & contracted $+$ EMA & 25 & $-31.3$ [$-31.7$, $-30.9$] & $-0.38$ [$-0.70$, $-0.06$] & $11.70$ [$11.18$, $12.21$] &  \\
reported & S & contracted $+$ EMA & 50 & $-54.5$ [$-54.8$, $-54.1$] & $-6.39$ [$-6.74$, $-6.04$] & $9.81$ [$9.30$, $10.32$] &  \\
\midrule[\heavyrulewidth]
steep & T & contracted & 0 & $213.5$ [$209.1$, $218.0$] & $32.87$ [$31.95$, $33.78$] & $-20.09$ [$-20.83$, $-19.36$] & none \\
steep & T & contracted & 25 & $203.7$ [$199.2$, $208.2$] & $22.31$ [$21.56$, $23.06$] & $-27.70$ [$-28.41$, $-26.98$] &  \\
steep & T & contracted & 50 & $195.8$ [$191.9$, $199.8$] & $16.52$ [$15.82$, $17.21$] & $-32.50$ [$-33.43$, $-31.56$] &  \\
steep & T & contracted $+$ EMA & 0 & $212.7$ [$208.7$, $216.6$] & $35.31$ [$34.65$, $35.97$] & $-16.00$ [$-16.47$, $-15.53$] & none \\
steep & T & contracted $+$ EMA & 25 & $201.6$ [$196.9$, $206.3$] & $24.69$ [$23.98$, $25.40$] & $-23.14$ [$-23.82$, $-22.46$] &  \\
steep & T & contracted $+$ EMA & 50 & $188.1$ [$183.7$, $192.5$] & $19.60$ [$19.02$, $20.18$] & $-24.14$ [$-24.75$, $-23.53$] &  \\
\midrule
steep & X & contracted & 0 & $217.4$ [$213.6$, $221.2$] & $31.91$ [$31.11$, $32.72$] & $-17.03$ [$-17.91$, $-16.15$] & none \\
steep & X & contracted & 25 & $207.7$ [$203.1$, $212.2$] & $21.54$ [$20.75$, $22.33$] & $-24.64$ [$-25.27$, $-24.00$] &  \\
steep & X & contracted & 50 & $199.7$ [$195.6$, $203.8$] & $15.85$ [$15.11$, $16.60$] & $-29.44$ [$-30.36$, $-28.52$] &  \\
steep & X & contracted $+$ EMA & 0 & $216.4$ [$212.5$, $220.4$] & $34.30$ [$33.75$, $34.86$] & $-12.94$ [$-13.60$, $-12.28$] & none \\
steep & X & contracted $+$ EMA & 25 & $205.5$ [$201.3$, $209.6$] & $23.87$ [$23.30$, $24.45$] & $-20.08$ [$-20.87$, $-19.30$] &  \\
steep & X & contracted $+$ EMA & 50 & $191.8$ [$188.1$, $195.5$] & $18.87$ [$18.23$, $19.51$] & $-21.08$ [$-21.61$, $-20.55$] &  \\
\midrule
steep & S & contracted & 0 & $-90.3$ [$-96.1$, $-84.5$] & $-4.95$ [$-5.58$, $-4.33$] & $12.56$ [$12.14$, $12.99$] & none \\
steep & S & contracted & 25 & $-90.2$ [$-97.3$, $-83.2$] & $-4.39$ [$-5.02$, $-3.77$] & $11.59$ [$11.04$, $12.14$] &  \\
steep & S & contracted & 50 & $-86.9$ [$-91.7$, $-82.1$] & $-4.43$ [$-5.01$, $-3.85$] & $8.98$ [$8.54$, $9.41$] &  \\
steep & S & contracted $+$ EMA & 0 & $-85.1$ [$-91.5$, $-78.7$] & $-4.53$ [$-5.13$, $-3.93$] & $12.51$ [$12.05$, $12.98$] & none \\
steep & S & contracted $+$ EMA & 25 & $-93.5$ [$-98.4$, $-88.7$] & $-4.60$ [$-5.15$, $-4.06$] & $11.69$ [$11.18$, $12.21$] &  \\
steep & S & contracted $+$ EMA & 50 & $-95.2$ [$-98.4$, $-92.0$] & $-4.84$ [$-5.32$, $-4.37$] & $9.76$ [$9.24$, $10.28$] &  \\
\bottomrule
\end{tabular}
\end{table*}

On that substrate the contraction buys exactness on a crossing instance, indistinguishability from the tree on the advertised side, and a welfare-over-ex-post-prefix-optimum range held within $0.31$ on the tree and on the crossing instance across the eight settings of the parameter sweep, against $0.71$ to $0.78$ on the same two without it. The ordering inverts on S, whose contracted range spans $0.82$ against $0.46$ uncontracted and whose welfare effect changes sign with the setting, a gain at the reported level and a loss at the steep one. On T and X the contraction costs admitted volume. \cref{tab:exp6-sensitivity} gives the sweep setting by setting. The deployment question is whether that trade is worth the translation overhead, and the answer depends on how the deployment values a bounded latency and an exact allocation.

\section{Statistical Analysis}
\label{app:statistics}

This section describes the statistical methodology used throughout the evaluation and presents the test results for each experiment. Hypothesis tests are non-parametric, which suits the $10$~seeds per condition, and the interval attached to a reported number depends on where the number comes from: node-level per-cell means and paired contrasts carry Student-$t$ intervals over the ten seeds, and the bootstrap intervals are scoped to the arms that use them. Each test's own $n$ is the number of seed-condition observations it consumes, and is quoted with the statistic.

\subsection{Methodology}

\paragraph*{Confidence intervals}
Every bracketed node-level figure in \suppname{} and in \mainname{} is a Student-$t$ $95\%$ interval over the seeds of its cell: ten seeds unless a block states otherwise, twenty on the tuned comparison and five on the eight-setting parameter sweep; the contrasts of a mis-set posted price against the tuned market are Welch intervals, the two sides sharing no seed. The released pipeline emits the intervals of the structure, governance, architecture, overhead, architecture-by-governance, mechanism, sweep, faithfulness and agentic blocks cell by cell and metric by metric, each with its $n$, to \texttt{results/node-intervals.csv}; those of the population sweep, the cap-target probe and the incentive blocks are computed from their per-seed targets. The bootstrap is used where the per-tier arms use it: bias-corrected and accelerated (BCa) $95\%$ intervals with $B = 2{,}000$ resamples over the seeds~\cite{efron1994introduction}, falling back to the percentile method where BCa does not converge, as on a zero-variance metric. Two further scopes. The per-tier smoothing filter's price-dispersion reduction and its paired-seed interval are a percentile bootstrap at $B = 10{,}000$, written by the pipeline to \texttt{results/exp4-volatility-reduction.csv} and transcribed from there. The emulated testbed's intervals are BCa over rounds within one seed (\cref{app:emul-testbed}).

\paragraph*{Hypothesis tests}
For multi-group comparisons we use the Kruskal--Wallis $H$-test~\cite{kruskal1952wallis}, a non-parametric analogue of one-way ANOVA. The choice is motivated by the small sample, $10$~seeds per condition, which precludes a reliable normality assessment, and by the nature of the simulation metrics (latency, drop rate, price volatility), which are skewed or bounded.

\paragraph*{Reported $n$ and degrees of freedom}
Every statistic below is quoted with its degrees of freedom and with the $n$ the test actually consumed. That $n$ varies by test. Where a condition contributes no finite value the test drops it: a posted price at a markup of four times the reserve admits nothing on any instance, so median latency is undefined there and the mechanism test on latency runs at $n = 240$, $330$ or $350$ of the full grid's $370$ or $390$; the mechanisms that post no price contribute no volatility series, so the volatility test in a mechanism cell runs at $n = 330$ or $350$. A second class of undefined statistic comes from a constant series. On the two laminar instances the exactness shortfall is zero in every seed and every round, so the rank statistic on that response is not defined in those cells: the constancy is an identity of the instrument (\cref{app:exp1-detailed}), and the cell is reported as a structural zero. The same treatment applies to price volatility below the dispersion onset.

\paragraph*{Effect sizes and interactions}
A contrast between two conditions is reported as the paired difference over the ten seeds, with a Student-$t$ $95\%$ interval on the node substrate and a bootstrap interval on the per-tier arms, so its magnitude is read in the metric's own units. Factorial interactions are reported the same way. The architecture factorial decomposes its reduction into two marginals and an interaction term (\cref{app:exp4-ablation}); the architecture$\times$governance experiment reports substitution, for the reason below.

\paragraph*{Composition of two remedies}
For the Architecture$\times$Governance experiment we assess how the contraction and a coordinate cap compose. On allocative exactness the question does not arise as a synergy: each remedy alone restores exactness on the crossing instance, so no additive decomposition is defined there and the composition is reported as substitution. On welfare and on admitted volume the four cells are compared directly, and a paired welfare contrast between any two of them is reported with its Student-$t$ $95\%$ interval in the metric's own units.

\subsection{Structure: Statistical Summary}
\label{app:stat-exp1}

Kruskal--Wallis tests on the three-instance contrast ($df = 2$, $n = 30$ per cell) give a main effect of the instance on every metric at every load, with one exception. At high load: median latency $H = 25.30$ ($p = 3.2\times10^{-6}$), the drop rate $H = 25.55$ ($p = 2.8\times10^{-6}$), and welfare, utilisation and price volatility each at $H = 25.81$ ($p = 2.5\times10^{-6}$), the largest value the statistic can take for three groups of ten, reached when the ranks never interleave. On the scale metrics the instance contrast also carries S's smaller population and leaf-block capacity ($N = 45$, $K_c = 50$), so it contrasts families and scale together. At medium load median latency, welfare, the drop rate, utilisation and volatility each sit at $H = 25.81$. The drop-rate row is a control, for the reason \cref{app:exp1-detailed} gives, and the structural row is the exactness shortfall, which is a constant zero on T and on S and positive on X alone.

The exception is price volatility at low load. Every observation in that cell is zero, on all three topologies and all ten seeds, so the rank statistic is not defined: the series is constant and there is nothing to rank. The constancy is itself the measurement.

\subsection{Scaling: Statistical Summary}
\label{app:stat-exp2}

Spearman rank correlations~\cite{spearman1904proof} between agent count $N$ and each metric confirm instance-dependent scaling over the extended grid. The onset statistic is read in offered load: it sits at $\rho = 0.60$ in five of the six instance-by-load cells and at $0.90$ in the sixth, S at high load, where the population grid has no point inside the interval that brackets it. The per-node Poisson first-crossing law carries the two qualifications stated in \cref{app:exp2}, and \cref{tab:exp2-onset} sets the measured onsets against it.

These correlations and the onset statistic come from the node-level scaling target; the onset rows are also in the flat statistics report.

\subsection{Governance: Statistical Summary}
\label{app:stat-exp3}

Kruskal--Wallis tests on the six-level cap contrast ($df = 5$, $n = 60$ per cell) give a significant effect on every metric in every cell. At the uniform mix, for T under high load: median latency $H = 48.76$ ($p = 2.5\times10^{-9}$), drop rate $H = 50.87$, welfare $H = 50.55$, admitted tokens $H = 50.88$, the allocative ratio $H = 50.88$ and price volatility $H = 50.88$. For S under high load: median latency $H = 45.06$ ($p = 1.4\times10^{-8}$), drop rate $H = 50.49$, welfare $H = 45.85$, admitted tokens $H = 50.88$, the allocative ratio $H = 50.88$ and price volatility $H = 49.63$.

The X cells are of the same order. For X under high load: median latency $H = 42.00$ ($p = 5.9\times10^{-8}$), drop rate $H = 47.99$, welfare $H = 47.42$, admitted tokens $H = 49.44$, the allocative ratio $H = 45.37$ and price volatility $H = 41.90$. The cap level moves every metric on all three instances at both loads. The direction and the size of the move differ between cells. Exactness is the exception: on S it is a constant zero at every level, so no rank statistic is defined there; on T it resolves only through the cross-node coupling; and on X it resolves ($H = 42.21$, $p = 5.3\times10^{-8}$ at high load) because the jurisdiction cap and the domain slice repair it and the other levels do not.

\subsection{Architecture: Statistical Summary}
\label{app:stat-exp4}

Two tests read this experiment. The first treats the four factorial arms as one four-level factor within each instance-by-load cell ($df = 3$). The second pools the grid and splits it on each factorial component in turn ($df = 1$, $n = 240$), which is the reading quoted here and in \mainname, because the four-level test cannot attribute a difference to either component.

On the post-burn-in price CV the contraction does not resolve ($H(1) = 0.9604$, $p = 0.3271$) while smoothing does ($H(1) = 150.99$, $p = 1.1\times10^{-34}$). On the three allocation responses the split runs the other way: the contraction resolves on median latency ($H(1) = 66.94$, $p = 2.8\times10^{-16}$), on welfare ($H(1) = 9.21$, $p = 0.0024$) and on the drop rate ($H(1) = 7.00$, $p = 0.0081$), and smoothing resolves on none of the three ($p = 0.62$, $0.13$ and $0.32$). In these pooled tests the two components separate, the filter on the price series and the contraction on the allocation; the filter's smaller allocation effects are resolved by the paired per-cell contrasts of \cref{app:exp4-ablation}. On the full series the contraction resolves on price volatility as well ($H(1) = 12.13$, $p = 5.0\times10^{-4}$), but that window still contains the opening rounds in which a smoothed price is converging from its entry level, so the post-burn-in rows are the ones quoted.

\subsection{Architecture $\times$ Governance: Statistical Summary}
\label{app:stat-exp5}

\paragraph*{Main effects} The cap level moves median latency, the drop rate, coverage, admitted volume and welfare, and the architecture factor moves median latency, welfare, admitted volume and the exactness shortfall. At $n = 20$ per instance-by-load cell the two-level tests saturate, most of them at $H(1) = 14.29$ ($p = 1.6\times10^{-4}$), which is perfect rank separation and reports no effect size; the aligned-rank~\cite{wobbrock2011aligned} interaction table resolves thirteen of its fifteen terms at $p < 10^{-17}$, the architecture term ($p = 0.045$) and the architecture-by-policy term ($p = 0.012$) being the exceptions. The paired per-cell contrasts of \cref{app:exp5-detailed} are therefore the readable form, and the pooled tables are quoted only for the presence of an interaction.

\paragraph*{Substitution} On exactness the two remedies are substitutes on X, so no synergy term is defined on that response; on welfare and throughput the composition is read from the four cells of \cref{app:exp5-detailed} directly.

\subsection{Mechanism: results}
\label{app:exp6}

The Mechanism experiment compares the allocation rules grouped by comparator tier, across three instances, two load levels, two architectures and both congestion levels, on a level grid of $9{,}120$~runs, a tuning grid of $13{,}080$ and an evaluation grid of $3{,}840$. Tier~A supplies the two reference optima the rest are scored against. Tier~B holds the truth-assuming single-domain planners, value-greedy, the Kubernetes-style priority rank and the value-ranked posted price, with random feasible packing and earliest-deadline-first as controls. Tier~C holds the cross-domain mechanisms, the t\^{a}tonnement market, the arrival-order and the deadline-priority posted prices over a re-levelled markup grid, a demand-responsive posted price, and the mixed arm, which runs on the contracted architecture alone, with the congestion-consistent market beside them as a diagnostic variant. Every other arm runs on both architectures. \cref{tab:exp6-naive} gives every arm its four numbers and its feasibility condition.

\paragraph*{The market against the planners} The market-against-planner comparison depends on the architecture and on the congestion parametrisation. At the steep sensitivity level the uncontracted market leads value-greedy in all six cells, by $+11.01$ [$+10.44$, $+11.58$]~a.u.\ at T high, $+0.52$ [$+0.35$, $+0.70$] at T medium, $+12.55$ [$+11.90$, $+13.21$] at X high, $+1.32$ [$+1.10$, $+1.54$] at X medium, $+8.17$ [$+7.75$, $+8.58$] at S high and $+4.17$ [$+3.79$, $+4.54$] at S medium. Under the contraction it trails in four of the six, by $-4.60$ [$-5.18$, $-4.02$] at T high and $-4.45$ [$-5.02$, $-3.88$] at X high, and leads only on S, because the contraction removes the queue that was the market's advantage. These are the level grid's untuned market (reserve markup $1$) against value-greedy, paired over its ten seeds. At the reported congestion level the same paired difference is $-0.74$ [$-0.91$, $-0.57$]~a.u.\ on T, $-0.48$ [$-0.78$, $-0.17$] on X and $-6.91$ [$-7.23$, $-6.60$] on S at high load. The reference distance travels with each: on the allocative side value-greedy attains $0.999$ of the exact allocative optimum on T against the market's $0.954$, so the two objectives disagree in sign wherever the market leads on realised welfare.

\paragraph*{The deployable comparison} Among the Tier~C rules, which need no single owner of capacity, the two static posted prices, arrival-order and deadline-priority, stay within $0.0004$ of each other in welfare over the ex-post prefix optimum in every cell at the reported level, so the participation screen decides what they deliver and the rationing rule behind it barely matters. At the reported congestion level and high load both sit at parity with the tuned market, a parity that holds under truthful reports, which the posted prices do not need. Paired by seed over the twenty evaluation seeds, the arrival-order posted price minus the market is $-0.0067$ [$-0.0079$, $-0.0055$] on the uncontracted tree, $+0.0202$ [$+0.0180$, $+0.0224$] on the uncontracted crossing instance and $+0.0147$ [$+0.0098$, $+0.0196$] on the uncontracted fan, and $+0.0010$ [$-0.0029$, $+0.0049$], $+0.0007$ [$-0.0033$, $+0.0046$] and $-0.0135$ [$-0.0145$, $-0.0126$] on the three contracted cells: ahead in two cells, behind in two and tied in two, with no difference larger than $0.020$, and the deadline-priority rule gives the same pattern to within $0.0004$. At medium load the market leads both in all six cells, by $0.0030$ to $0.0316$, least on S, where the posted level sits at the grid floor. At the steep sensitivity level the ordering is mixed and large, the arrival-order posted price leading the market in three high-load cells, by up to $0.181$ on the uncontracted tree, trailing it in one, by $0.083$ on the uncontracted fan, and within $0.003$ in two. On the other objective the tuned market attains $0.74$ to $0.99$ of the exact allocative optimum over the six high-load cells at the reported level and the two static truth-free posted prices $0.76$ to $0.95$. Every posted level above is tuned on the cell's own load, which is a level set in advance with knowledge of the cell. \cref{tab:misset} sets the level without that knowledge, from existing runs: the level tuned at the other load of the same instance, architecture and congestion level, and one level pooled over the two architectures. Where the carried-over level differs from the cell's own, the tuned market leads the arrival-order posted price in every cell that can be measured at the reported level, by $0.99$ to $5.32$~a.u., $3.2$ to $8.7$ percent of the ex-post prefix optimum; the two high-load cells on S cannot be measured, because the carried-over level was never run on the level grid. With the pooled level the market leads in three of the four cells where that level differs from the cell's own, and trails by $0.84$ on the contracted tree at high load.

\paragraph*{The demand-responsive posted price} The demand-responsive posted price, re-set each round toward a target ratio of the demand its level screens in to capacity, as cloud spot pricing adjusts to demand, is measured on the same twenty evaluation seeds, paired by seed. At the reported congestion level and high load its paired difference from the tuned market is $+0.0075$ [$+0.0038$, $+0.0111$] on the contracted crossing instance and $+0.0276$ [$+0.0253$, $+0.0298$] on the uncontracted one, $-0.0140$ [$-0.0153$, $-0.0126$] on the contracted fan and $+0.0389$ [$+0.0345$, $+0.0432$] on the uncontracted one, and $+0.0079$ [$+0.0043$, $+0.0115$] on the contracted tree and $-0.0091$ [$-0.0111$, $-0.0071$] on the uncontracted one: ahead in four cells by $0.007$ to $0.039$ and behind in two by $0.009$ and $0.014$, at $0.907$ to $0.973$ of the ex-post prefix optimum against the market's $0.899$ to $0.976$. Against the arrival-order posted price it is ahead by $0.0068$ [$0.0060$, $0.0077$] and $0.0074$ [$0.0056$, $0.0092$] on the two crossing cells, $0.0242$ [$0.0228$, $0.0256$] on the uncontracted fan and $0.0069$ [$0.0060$, $0.0077$] on the contracted tree, ties on the contracted fan at $-0.0004$ [$-0.0016$, $+0.0008$] and trails by $0.0024$ [$0.0008$, $0.0040$] on the uncontracted tree. At medium load the market leads it in all six cells, by $0.0029$ [$0.0024$, $0.0034$] on the uncontracted crossing instance to $0.0354$ [$0.0346$, $0.0362$] on the contracted one, and it stays within $0.004$ of the arrival-order posted price except on the uncontracted fan, where it trails by $0.0171$ [$0.0157$, $0.0184$]. At the steep sensitivity level and high load it leads the market in five cells, by $0.009$ to $0.140$, and trails it on the uncontracted fan by $0.041$ [$0.031$, $0.050$]; against the arrival-order posted price it is ahead in three cells and behind in three, by up to $0.051$. The tuned step sits at the grid's lowest value, $0.01$, in four parallel-fan cells. At a step of zero the rule is the arrival-order posted price at the reserve, a level that arm's tuning grid contains; on the tuning seeds the step of $0.01$ gains over it in two of the four cells, $0.9628$ against $0.9405$ and $0.8895$ against $0.8300$, and matches it to the fourth decimal in the other two, where the level never moves, so on the fan the optimal step lies between zero and $0.02$. The tuned target lies inside the grid in every cell, and the tuning surface is flat in two; the knobs per cell are in the pipeline's tuned table (\texttt{node\_exp6\_tuned}).

\paragraph*{The planner behind a screen} The value-ranked posted price, a Tier~B arm, leads the tuned market in five of the six high-load cells at the reported level, by $0.0033$ [$0.0029$, $0.0036$] to $0.0729$ [$0.0683$, $0.0774$], and trails it on the uncontracted tree by $0.0169$ [$0.0157$, $0.0181$], at $0.85$ to $1.00$ of the exact allocative optimum. Its lead over the truth-free posted prices is the lead of ranking by value, which a cross-domain rule cannot read, and its margins move with the congestion calibration.

\begin{table}[!t]
\centering
\caption{Posted price at a level tuned outside its own cell. The arrival-order posted price at a level carried over from the other load of the same instance, architecture and congestion level (\emph{Other load}) or pooled over the two architectures (\emph{Pooled}), read off the level grid ($10$~seeds), minus the tuned market on its twenty evaluation seeds, in welfare (a.u., Welch $95\%$ interval) and as a share of the ex-post prefix optimum. \emph{Own} is the level tuned on the cell's own load and \emph{Used} the carried-over level. Reported congestion level; only cells where the carried level differs from the cell's own are listed.}
\label{tab:misset}
\scriptsize
\setlength{\tabcolsep}{2pt}
\begin{tabular}{@{}llrrlr@{}}
\toprule
\textbf{Variant} & \textbf{Cell} & \textbf{Own} & \textbf{Used} & \textbf{Posted $-$ market} & \textbf{Share} \\
\midrule
Other load & T, unc., high & $2.00$ & $1.00$ & $-5.31$ [$-5.45$, $-5.18$] & $-0.081$ \\
Other load & X, unc., high & $1.75$ & $1.25$ & $-4.20$ [$-4.40$, $-4.00$] & $-0.066$ \\
Other load & S, con., high & $2.00$ & $0.50$ & \multicolumn{2}{c}{level not on the grid} \\
Other load & S, unc., high & $1.00$ & $0.50$ & \multicolumn{2}{c}{level not on the grid} \\
Other load & T, unc., medium & $1.00$ & $2.00$ & $-5.32$ [$-5.59$, $-5.05$] & $-0.087$ \\
Other load & X, unc., medium & $1.25$ & $1.75$ & $-3.26$ [$-3.57$, $-2.96$] & $-0.055$ \\
Other load & S, con., medium & $0.50$ & $2.00$ & $-2.39$ [$-2.57$, $-2.22$] & $-0.077$ \\
Other load & S, unc., medium & $0.50$ & $1.00$ & $-0.99$ [$-1.15$, $-0.84$] & $-0.032$ \\
Pooled & T, con., high & $1.75$ & $2.00$ & $+0.84$ [$+0.56$, $+1.13$] & $+0.013$ \\
Pooled & S, con., high & $2.00$ & $1.00$ & $-2.45$ [$-2.51$, $-2.38$] & $-0.072$ \\
Pooled & T, con., medium & $1.75$ & $1.00$ & $-2.52$ [$-2.79$, $-2.26$] & $-0.041$ \\
Pooled & X, con., medium & $1.75$ & $1.25$ & $-2.49$ [$-2.75$, $-2.24$] & $-0.042$ \\
\bottomrule
\end{tabular}
\end{table}

\paragraph*{The planner tier} The Kubernetes rank sits just below value-greedy, and the gap depends on a normalisation the arm carries: its least-requested term~\cite{k8sdocs2026schedconfig} is normalised across each round's tasks before the decade weight, without which the arrival tiebreak dominates it by about thirty to one. On the level grid at the steep sensitivity level and T under high load, it matches value-greedy on utilisation to $0.0035$ and does not track it on drop rate, sitting $0.0273$ above, and it gives up $0.84$~a.u.\ of welfare, $17.8$ percent of a small total; at the reported level it gives up $5.41$~a.u., $8.7$ percent. Both arms are counterfactuals in the cross-domain setting of the premise: each reads a declared value or a declared priority at face value, and the incentive experiment shows why that reading is unavailable.

\paragraph*{Under the contraction} The ordering does not carry over. At the reported congestion level and high load the static truth-free posted prices tie the market on the contracted tree and crossing instance, both intervals spanning zero, and trail it on the contracted fan, while the value-ranked posted price leads it by $0.0717$ [$0.0672$, $0.0762$] on T and $0.0729$ [$0.0683$, $0.0774$] on X, against the market's lead of $0.0169$ over it on the uncontracted tree. Value-greedy leads the market in all three contracted cells, where uncontracted the market leads it on T and on X. The mixed arm, which posts the integrator's slice price inside a discovered-price market, is the best contracted arm on T and on X, at $0.973$ and $0.992$ of the ex-post prefix optimum at its better level.

\paragraph*{The ascending-price arm} The pipeline also runs an ascending-price variant of the market as a diagnostic. Both market arms admit tasks by descending surplus, reported value less the node-price cost, under the node capacities, and they differ only in the price process: the reported market moves prices up and down, the ascending arm only raises them. Untuned, at reserve markup $1$, the ascending arm reaches $0.971$ to $0.990$ of the ex-post prefix optimum in the six high-load cells at the reported level (level grid, ten seeds), above the tuned market ($0.899$ to $0.976$) and above the best static truth-free posted price ($0.900$ to $0.966$); \cref{tab:exp6-naive} carries its uncontracted row. Its two readings differ by architecture. In the contracted cells it does not clear, its residual excess demand being $0.044$ to $0.056$, and its welfare matches value-greedy's to within $0.002$ of the ex-post prefix optimum: there it is value rationing at a non-clearing price near the reserve, and it is not incentive compatible, since a higher report buys admission at the same price. In the uncontracted cells it nearly clears (residual $0.000$ to $0.032$) at prices below the t\^{a}tonnement's, a mean unit cost of $0.41$ to $0.46$ against $0.45$ to $0.51$, and admits more, $47.1$ against $36.8$ tokens per round on S; that isolates the bidirectional t\^{a}tonnement's price overshoot as a limitation of the implementation the article reports. The arm is diagnostic only: it prices nodes where Proposition~2 prices slice types, runs truncated at $15$ iterations, and carries no certificate or incentive run.

\paragraph*{The parameter sweep behind the welfare band} The mechanism comparison is repeated at eight parameter settings, one factor at a time from the baseline setting: four of them set the execution congestion term, through the utilisation clamp and the queue coefficient, three set the value-decay rate, and one sets the bid-time congestion scale. \cref{tab:exp6-sensitivity} gives the market arm's welfare over the ex-post prefix optimum at each, with and without the contraction. The band in \mainname{} is the range over the forty runs behind each column, eight settings of five seeds: on the tree and on the crossing instance the contracted runs stay within $0.31$ of one another while the uncontracted ones spread over $0.71$ and $0.78$, and on the parallel fan the two swap places, the contracted runs spreading over $0.82$ against $0.46$ uncontracted. Half of what the sweep varies is demand-side or bid-side.

\begin{table*}[!t]
\centering
\caption{Parameter sweep of the mechanism comparison: the untuned market arm's (reserve markup $1$) realised welfare over the ex-post prefix optimum at each of the eight settings, uncontracted (unc.) and contracted (con.). The settings are one factor at a time from the baseline: the utilisation clamp and the queue coefficient $c$ set the execution congestion term, $\lambda_l$ the value decay and $\alpha$ the bid-time congestion scale. High load, three instances, two architectures, $5$~seeds per cell, $30$~market runs per setting and $240$ in all. Baseline row in bold; the article reports the calibrated setting. Over the forty runs behind each instance-and-architecture column the welfare ratio spans $0.31$ contracted against $0.71$ uncontracted on T, $0.31$ against $0.78$ on X, and $0.82$ against $0.46$ on S.}
\label{tab:exp6-sensitivity}
\renewcommand{\arraystretch}{1.1}
\footnotesize
\setlength{\tabcolsep}{4pt}
\begin{tabular}{@{}l l cccc rr rr rr@{}}
\toprule
\textbf{Setting} & \textbf{Factor} & $\lambda_l$ & \textbf{Clamp} & $\alpha$ & $c$ & \multicolumn{2}{c}{\textbf{T}} & \multicolumn{2}{c}{\textbf{X}} & \multicolumn{2}{c}{\textbf{S}} \\
\cmidrule(lr){7-8}\cmidrule(lr){9-10}\cmidrule(lr){11-12}
 & & & & & & unc. & con. & unc. & con. & unc. & con. \\
\midrule
\textbf{Baseline} & Congestion & $9.64\times10^{-3}$ & $0.99$ & $50$ & $2$ & $0.265$ & $0.871$ & $0.283$ & $0.889$ & $0.357$ & $0.205$ \\
Clamp $0.95$ & Congestion & $9.64\times10^{-3}$ & $0.95$ & $50$ & $2$ & $0.734$ & $0.880$ & $0.713$ & $0.898$ & $0.675$ & $0.785$ \\
Clamp $0.995$ & Congestion & $9.64\times10^{-3}$ & $0.995$ & $50$ & $2$ & $0.302$ & $0.870$ & $0.288$ & $0.888$ & $0.455$ & $0.214$ \\
Calibrated & Congestion & $9.64\times10^{-3}$ & $0.95$ & $50$ & $0.75$ & $0.944$ & $0.887$ & $0.923$ & $0.905$ & $0.766$ & $0.977$ \\
\midrule
$\lambda_l = 0.0025$ & Value decay & $2.50\times10^{-3}$ & $0.99$ & $50$ & $2$ & $0.544$ & $0.953$ & $0.539$ & $0.973$ & $0.701$ & $0.624$ \\
$\lambda_l = 0.005$ & Value decay & $5.00\times10^{-3}$ & $0.99$ & $50$ & $2$ & $0.339$ & $0.955$ & $0.347$ & $0.975$ & $0.512$ & $0.378$ \\
$\lambda_l = 0.015$ & Value decay & $1.50\times10^{-2}$ & $0.99$ & $50$ & $2$ & $0.408$ & $0.708$ & $0.523$ & $0.721$ & $0.541$ & $0.296$ \\
\midrule
$\alpha = 200$ & Bid-time scale & $9.64\times10^{-3}$ & $0.99$ & $200$ & $2$ & $0.309$ & $0.665$ & $0.288$ & $0.680$ & $0.362$ & $0.563$ \\
\bottomrule
\end{tabular}
\end{table*}

\subsection{Mechanism: Statistical Summary}
\label{app:stat-exp6}

\paragraph*{Main effects (uncontracted architecture)} Kruskal--Wallis tests for the mechanism factor are significant on welfare in every cell, slack cells included. At T under high load and the reported congestion level, welfare gives $H(36) = 362.61$ at $n = 370$ ($p = 1.4\times10^{-55}$), the drop rate $H(36) = 346.55$ at $n = 370$, median latency $H(32) = 322.81$ at $n = 330$ and price volatility $H(32) = 328.89$ at $n = 330$.

Restricted to the four ablation levels the narrower reading reproduces, and the wider factor's larger statistic comes from the Tier~C arms and the planner tier entering it.

Median latency and price volatility run at the reduced $n$ that \cref{app:statistics} gives.

Pooled across instance and load, the uncontracted architecture at the reported congestion level resolves the mechanism factor on welfare, the drop rate, median latency and price volatility alike, every one at $p < 10^{-20}$. No interaction statistic is quoted from the pooled aligned-rank table. The table is balanced, $37$ mechanism levels of $240$ runs each ($df = 36$ for the mechanism term, $7{,}992$ residual degrees of freedom), with the mixed arm's two levels left out because they run on one architecture; every one of its $31$ terms resolves at $p$ below machine precision, so it separates nothing the per-cell tests do not.

\paragraph*{Pairwise contrasts} The statistics report carries no pairwise or rank-based effect-size rows for this experiment, so the mechanism contrasts above are given as paired differences over the seeds, twenty on the tuned comparison and ten on the level grid, with Student-$t$ intervals, and no post-hoc rank test is used.

\paragraph*{Theory-to-experiment mapping}
\cref{tab:theory-experiment} links each formal result to the experiment that exercises it and states how: whether an experiment varies the result's hypothesis, only instantiates it, or simulates a closed form's prediction. Two rows are not tests. No arm varies Lemma~1's premises, since every task is one unit at one designated leaf, where gross substitutes holds trivially, so the lemma is instantiated. The one-recipe arms of the heterogeneity experiment carry \cref{prop:catalogue-exact}(ii) as an identity, while a simulated proportional-recipe arm measures part~(iii)'s closed form at five ratios; the heterogeneous recipes are non-proportional and so lie outside the proposition's hypothesis altogether, and they exercise the paragraph on why the condition cannot be dropped. Proposition~3 is exercised on S, whose cluster's leaf blocks coincide; on T and X the contracted scalar is an inner exposure, and \cref{prop:inner-exposure} is the licensing result there. The faithfulness probe varies condition~(i) through the advertised scalar at a fixed cluster, and over-commits by up to $17.45$ tokens per round when the scalar is the routed maximum flow; conditions~(iii) and~(iv) are structural hypotheses that a single-integrator simulation does not vary. The simulated $\bar{C}$ is the cluster's least node capacity in tokens, $50$ on each instance.

\begin{table}[!t]
\centering
\caption{Theory-to-experiment mapping. Each formal result is paired with the experiment that exercises it, its observables, and how it enters: \emph{varied} (an arm varies the hypothesis), \emph{instantiated} (every arm satisfies it) or \emph{simulated} (an arm measures a closed-form prediction).}\label{tab:theory-experiment}
\footnotesize
\renewcommand{\arraystretch}{1.15}
\setlength{\tabcolsep}{2pt}
\scriptsize
\begin{tabular}{@{}p{2.2cm}p{2.0cm}p{2.3cm}l@{}}
\toprule
\textbf{Result} & \textbf{Experiment} & \textbf{Observables} & \textbf{Enters} \\
\midrule
Prop.~1 (Polymatroid) & Structure ($-$S) & Exactness, certificate & varied \\
Lemma~1 (GS valuations) & every arm & --- & instantiated \\
Prop.~\ref{prop:catalogue-exact}(ii) & Heterogeneity, one recipe & Greedy-exact ratio & instantiated \\
Prop.~\ref{prop:catalogue-exact}(iii) & Proportional-recipe arm & Over-commitment factor & simulated \\
Prop.~\ref{prop:inner-exposure} (inner) & Architecture on T, X & Over-commitment, volume & instantiated \\
Prop.~2 (DSIC, shading) & Incentives & Per-agent regret & varied \\
Prop.~2 (DSIC, joint) & Joint misreport & Gain, inexactness & varied \\
Prop.~3 (Encaps.) & Architecture on S & Exactness, over-commit. & instantiated \\
Prop.~3, condition (i) & Faithfulness & Over-commitment & varied \\
\bottomrule
\end{tabular}
\end{table}

\paragraph*{Ablation table (complete)}
\cref{tab:ablation-full} summarises the full ablation matrix. Each row indicates which components are removed or toggled and the experiment that supplies the contrast.

\begin{table*}[tp]
\centering
\caption{Complete Ablation Matrix}\label{tab:ablation-full}
\renewcommand{\arraystretch}{1.15}
\begin{tabular}{lll}
\toprule
\textbf{Ablation} & \textbf{Source} & \textbf{Isolates} \\
\midrule
Full system (S+H+G+M) & Arch.$\times$Gov.\ (contracted, capped) & Everything on \\
$-$M (remove market)   & Mechanism (Tier~B vs.\ market) & Price coordination \\
$-$G (remove governance) & Arch.$\times$Gov.\ (no-cap conditions) & Coordinate-wise leaf caps \\
$-$H (remove integrator) & Architecture (uncontracted arm) & Contraction and smoothing together \\
$-$H$_{\mathrm{S}}$ (remove smoothing) & Architecture (contracted without EMA) & Smoothing alone \\
$-$H$_{\mathrm{E}}$ (remove contraction) & Architecture (uncontracted $+$ EMA) & The contraction alone \\
$-$S (break structure) & Structure, Architecture (X) & Structural discipline \\
$-$M$-$H (heuristic) & Mechanism (EDF, uncontracted) & Scheduling-only baseline \\
$-$S$-$H$-$M (worst case) & Mechanism (random, instance X) & Absolute floor \\
Tier~B planner & Mechanism (Kubernetes-style rank) & A truth-assuming single-domain planner \\
Tier~C posted prices & Mechanism (arrival-order, deadline-priority) & A screened level price, no value ranking \\
\bottomrule
\end{tabular}
\end{table*}

\section{Additional Material}
\label{app:additional}

\subsection{Model Overview and Assumptions}
\label{app:model-overview}

The framework in \mainname{} rests on assumptions standard in network/service management and mechanism design. In particular, we assume that: (i)~autonomous agents drive service demand and resource negotiation~\cite{li2024agentic,park2023generative}; (ii)~tasks are latency- and quality-of-service-sensitive~\cite{meuser2024edgeai}; (iii)~the system spans heterogeneous device--edge--cloud resources~\cite{bonomi2012fog,meuser2024edgeai}; (iv)~services depend on multi-layer resource compositions~\cite{meuser2024edgeai,li2020service}; (v)~resource availability and network conditions vary over time~\cite{meuser2024edgeai,yang2022edge}; (vi)~governance and policies constrain resource usage~\cite{fornara2008institutions}; (vii)~agents may withhold or misreport private information~\cite{mcafee1992dominant}; (viii)~reputation reflects provider reliability~\cite{sun2019trust,huang2022trust}; (ix)~architectural encapsulation (slicing, virtualisation) is possible~\cite{foukas2017network,li2020service}; and (x)~at each decision epoch, agents evaluate allocations with respect to a sufficiently aligned representation of the system state~\cite{loven2025llmcontinuum,loven2023semantic}. While private information and strategic incentives remain, agents condition their valuations on a commonly accepted state description~$s_t$ capturing resource capacities, network conditions, governance constraints, and reputational signals. The paper therefore focuses on normative allocation and mechanism design given a shared operational state, and does not model heterogeneous beliefs or epistemic disagreement about the underlying environment. The system state evolves in a Markovian manner driven by allocations, network conditions, and stochastic events.

The per-period utility in \mainname{} is quasilinear and carries no risk term. An extension would add one, $u_i(m(t); \theta(t)) = \sum_{k \in \mathcal{T}_i(t)} V_{ik}(T_{ik}, q_k) - P_i(t) - \gamma_i \, \mathrm{risk}_i(t)$, with $\gamma_i \ge 0$ a risk-aversion coefficient and $\mathrm{risk}_i(t)$ the uncertainty of agent~$i$'s experienced latency; under the congestion model that uncertainty depends on the other agents' allocations as well as on $x_i$, so the extension would not preserve the quasilinear structure in general. No reported arm uses it: every arm sets $\gamma_i = 0$.

In the three experiments that keep the per-tier identical-bundle environment, topology enters the simulation through three channels: the set of tiers its node list spans, the per-tier demand profile it induces, and the sequence of tiers along the DAG's critical path, and node identity, node count and edge structure have no further effect on allocation or price there. Two graphs that span the same tiers, induce the same per-tier demand profile and share the same critical-path tier sequence therefore produce identical trajectories under a common seed in that environment, and we verify this identity directly. On the per-node instances, by contrast, the leaf-block family is the object the theory names, and the three instances \mainname{} evaluates differ in it while sharing one node set and one capacity vector.

\subsection{Scaling}
\label{app:exp2}

The second experiment tests whether the rate of performance degradation as the agent population increases depends on the leaf-block family. We compare T, X and S over a $28$-point population grid from $N = 10$ to $N = 200$ at medium and high load, with the uncontracted architecture and no governance: $1{,}680$~runs.

\begin{figure}[!t]
    \centering
    \includegraphics[width=1\linewidth]{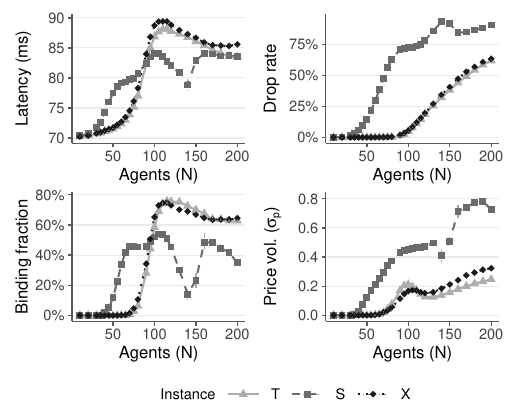}
    \caption{Scaling: instance-dependent scaling as agent count increases over the extended grid, at medium load. Panels: median latency, drop rate, binding fraction, price volatility.}
    \label{fig:exp2-scaling-supp}
\end{figure}

\cref{fig:exp2-scaling-supp} shows the same regime sequence on all three instances at different populations: every metric flat while the reserve pins every leaf price, then dispersion switching on at the onset, then a peak, then a fall-back as the drop rate approaches one. On T the sequence is $\sigma_p$ zero through $N = 55$, non-zero from $N = 60$, a peak of $0.212$ at $N = 100$ and a fall-back to $0.125$ at $N = 130$, where the drop rate is $0.256$.

X reaches the same sequence at the same population, because its leaf-block capacity equals T's by construction. On X it is zero through $N = 55$, non-zero from $N = 60$, a peak of $0.173$ at $N = 105$ and a fall-back to $0.152$ at $N = 120$, where the drop rate is $0.196$. Where the two instances separate is not the sequence but the exactness of the allocation inside it: the greedy packer's shortfall is identically zero on T at every population and positive on X wherever the crossing binds.

S reaches it at half the population, because its leaf-block capacity is half theirs. On S it is zero through $N = 20$, non-zero from $N = 30$, a peak of $0.497$ at $N = 130$ and a fall-back to $0.409$ at $N = 140$, where the drop rate is $0.932$. Its residual excess demand at the terminal price vector is the largest of the three at high load, $0.033$ [$0.031$, $0.035$] against $0.018$ [$0.016$, $0.021$] on X and $0.008$ [$0.006$, $0.009$] on T, which is a congestion reading: S is laminar.

The scaling ceiling, the agent count at which an instance moves from stable to degraded operation, is set by the leaf-block capacity of the effective region. The t\^{a}tonnement pins every leaf price at the reserve, so the agent-facing price is constant, and $\sigma_p$ zero, in every round whose arrivals fit that capacity: the region admits at most $K_c$ tasks per round, which is $100$ on T, $100$ on X and $50$ on S, and arrivals are Poisson with mean $\lambda N$. If a leaf price left the reserve only in a round whose arrivals exceed $K_c$, then over $200$ rounds and ten seeds the probability that the mean over seeds leaves zero would be $1 - P(\mathrm{Poisson}(\lambda N) \le K_c)^{2000}$; the measured onset comes earlier than that law predicts, for the reason below. \cref{tab:exp2-onset} sets the measured onsets against that law, read in offered load: the onset sits at $\rho = 0.60$ in five of the six instance-by-load cells and at $0.90$ in the sixth, where the grid has no point inside the bracketing interval. Two qualifications travel with the law: it under-counts the onset by a factor of three to five, because a leaf price leaves the reserve before utilisation reaches one, and its binding node is a single node, so it assigns X exactly T's predicted probability, while X's observed binding fractions run $1.1$ to $1.7$ times T's. T and X coincide by construction and S sits at half, which is a check that the instrument reads $K_c$ and not the drawing; beyond onset the dynamics are the market's own.

\begin{table}[htbp]
\centering
\caption{Onset of price dispersion in the population sweep, read in offered load. \emph{Measured $N$} is the smallest grid population whose mean $\sigma_p$ over ten seeds is non-zero, with the fraction of seeds non-zero in parentheses; \emph{All seeds} is the smallest population non-zero in every seed; \emph{Onset $\rho$} is the offered load at the measured onset. The first-crossing law under-counts the onset by a factor of three to five and assigns X exactly T's predicted probability, so it is reported as a bracket.}
\label{tab:exp2-onset}
\scriptsize
\setlength{\tabcolsep}{3pt}
\begin{tabular}{@{}llrlrr@{}}
\toprule
\textbf{Instance} & \textbf{Load} & $K_c$ & \textbf{Measured $N$ (seeds)} & \textbf{All seeds} & \textbf{Onset $\rho$} \\
\midrule
T & medium & 100 & 60 (0.2) & 75 & 0.60 \\
T & high & 100 & 40 (0.2) & 50 & 0.60 \\
X & medium & 100 & 60 (0.2) & 70 & 0.60 \\
X & high & 100 & 40 (0.3) & 50 & 0.60 \\
S & medium & 50 & 30 (0.1) & 40 & 0.60 \\
S & high & 50 & 30 (1.0) & 30 & 0.90 \\
\bottomrule
\end{tabular}
\end{table}

\subsection{Existence, Exactness and Stability Beyond the Evaluated Instances}
\label{app:crossing-sweep}

Two questions about the leaf-block family cannot be settled on the three evaluated instances, because all three have a totally unimodular incidence matrix: whether clearing prices exist, and what a crossing costs when they do. A third, whether the dependency structure orders the stability of the price process, is measured on the three instances themselves. This subsection reports the three instruments.

\paragraph*{The existence certificate on the evaluated instances} With linear anonymous node prices and single-bundle unit demand, a Walrasian equilibrium floored at the common reserve exists for a round if and only if the linear relaxation of that round's packing on reserve-adjusted values has an integral optimum~\cite{bikhchandani1997competitive}. Solving the relaxation and the integer program on one matrix per round over T, X and S at both leaf mixes, ten seeds and twenty rounds gives $1{,}200$~rounds, every one integral, the largest relative gap $1.6\times10^{-13}$. The instrument can detect a fractional optimum: on three pairwise-crossing unit-capacity blocks it returns a relaxation of $1.5$ against an integer optimum of $1$. The zero is therefore a property of the instances, and it follows from total unimodularity, which makes every vertex of the relaxation integral for integral right-hand sides~\cite{hoffman1956integral}: X's family is a laminar family plus one set and every block is an interval in the leaf order, so the matrix is an interval matrix and totally unimodular~\cite{fulkerson1965incidence}.

\paragraph*{Where existence fails} A sweep over generated leaf-block families locates the boundary. It fills five strata to $20$~instances each, with three to eight leaves and two to six internal nodes: the pipeline's own three instances, two hand-built triangles and a domain-level reading of the pipeline templates of a companion deployment~\cite{loven2026nps} take their places in the strata and $94$~generated families fill the rest, and a two-terminal series-parallel rejoin instance is appended outside the quota. Both instruments run on every round of the $101$~instances at both leaf mixes, ten seeds and twenty rounds: $40{,}400$~rounds in all. Demand mirrors the substrate, and internal capacities sit at $0.6$ of expected block demand. \cref{tab:crossing-sweep} reports it.

\begin{table*}[tp]
\centering
\caption{Existence and exactness by leaf-block family, uniform leaf mix. Strata are the generated families only; the named instances sit on their own rows and are excluded from the stratum means. \emph{Cross.} is the range of pairwise-crossing block pairs, \emph{Bip.} whether the crossing graph is bipartite, \emph{Int.} whether the leaves admit an interval order and \emph{TU} the brute-force total-unimodularity verdict on the incidence matrix. \emph{Pos.\ gap} is the fraction of rounds with a positive relaxation-to-integer gap, \emph{Rel.\ gap} its mean and maximum, \emph{Exactness} the greedy packer's mean and worst round against the integer optimum, and \emph{Adm.} the round's mean admitted count. The series-parallel rejoin row is a two-terminal series-parallel network that is not a tree (the device node feeds two parallel nodes that rejoin at one node before the leaves), laminar, at the substrate's node count and the generator's capacity rule; it is the series-parallel case of the proposition exercised on the two structural instruments. The two triangle rows are the same family at unit capacities and at the generator's own capacity rule; its failure rate is a property of the former. The pipeline-template row is a domain-level reading of the companion deployment~\cite{loven2026nps}, its three templates read here as demand classes over five domain nodes re-split from that deployment's four administrative domains; the companion did not run the three templates together.}
\label{tab:crossing-sweep}
\scriptsize
\setlength{\tabcolsep}{3pt}
\begin{tabular}{@{}l rr l lll r ll r@{}}
\toprule
\textbf{Family} & \textbf{Inst.} & \textbf{Rounds} & \textbf{Cross.} & \textbf{Bip.} & \textbf{Int.} & \textbf{TU} & \textbf{Pos.\ gap} & \textbf{Rel.\ gap (mean / max)} & \textbf{Exactness (mean / worst)} & \textbf{Adm.} \\
\midrule
\multicolumn{11}{@{}l}{\emph{Generated strata}} \\
Laminar & 17 & 3{,}400 & 0 & Yes & Yes & Yes & 0.000 & $3.8\times10^{-17}$ / $1.4\times10^{-15}$ & 1.0000 / 1.0000 & 81.0 \\
One crossing & 19 & 3{,}800 & 1 & Yes & Yes & Yes & 0.000 & $2.4\times10^{-17}$ / $3.4\times10^{-15}$ & 0.9779 / 0.8805 & 86.1 \\
Bi-laminar & 20 & 4{,}000 & 2 to 6 & Yes & Mixed & Yes & 0.000 & $3.2\times10^{-17}$ / $3.3\times10^{-15}$ & 0.9859 / 0.8855 & 84.3 \\
Odd cycle, interval & 20 & 4{,}000 & 3 to 6 & No & Yes & Yes & 0.000 & $5.0\times10^{-17}$ / $3.2\times10^{-15}$ & 0.9854 / 0.8944 & 84.3 \\
Odd cycle, non-interval & 18 & 3{,}600 & 3 to 11 & No & No & No & 0.082 & $8.2\times10^{-5}$ / $4.7\times10^{-3}$ & 0.9833 / 0.9085 & 82.9 \\
\midrule
\multicolumn{11}{@{}l}{\emph{Named instances}} \\
T & 1 & 200 & 0 & Yes & Yes & Yes & 0.000 & $2.1\times10^{-17}$ / $1.8\times10^{-15}$ & 1.0000 / 1.0000 & 100.0 \\
X & 1 & 200 & 1 & Yes & Yes & Yes & 0.000 & $-5.7\times10^{-18}$ / $1.8\times10^{-15}$ & 0.9996 / 0.9522 & 100.0 \\
S & 1 & 200 & 0 & Yes & Yes & Yes & 0.000 & $2.8\times10^{-17}$ / $3.9\times10^{-16}$ & 1.0000 / 1.0000 & 50.0 \\
Series-parallel rejoin & 1 & 200 & 0 & Yes & Yes & Yes & 0.000 & $3.0\times10^{-17}$ / $1.4\times10^{-15}$ & 1.0000 / 1.0000 & 81.0 \\
Pipeline templates over domains~\cite{loven2026nps} & 1 & 200 & 0 & Yes & Yes & Yes & 0.000 & $2.4\times10^{-17}$ / $1.5\times10^{-15}$ & 1.0000 / 1.0000 & 81.0 \\
Triangle, unit capacities & 1 & 200 & 3 & No & No & No & 1.000 & $0.488$ / $0.4996$ & 1.0000 / 1.0000 & 1.0 \\
Triangle, generator's capacities & 1 & 200 & 3 & No & No & No & 0.000 & $8.7\times10^{-18}$ / $1.1\times10^{-15}$ & 0.9822 / 0.9200 & 81.0 \\
\bottomrule
\end{tabular}
\end{table*}

The structural predicates decide existence exactly. With the pipeline's named instances counted in their strata, the brute-force verdict is true on all $21$ laminar, all $20$ one-crossing, all $20$ bi-laminar and all $20$ odd-cycle interval-order families, and false on all $20$ odd-cycle non-interval families, and a positive gap occurs only where it is false: the $81$ unimodular instances are integral in all $32{,}400$ of their rounds. On the $18$ generated odd-cycle non-interval families the fraction of rounds with a positive gap is $0.082$ at the uniform mix and $0.009$ at the skewed one, with $11$ and $8$ of the $18$ failing somewhere and a worst instance at $0.415$. That fraction is not ordered by the number of pairwise-crossing triples. The hand-built triangle is fractional in every round at unit capacities and integral in every round at the generator's own capacity rule, where only its greedy packer fails, so its failure rate is a property of the capacities and is reported only with them named.

A crossing costs exactness. On the $60$ instances whose families cross while remaining unimodular, the existence check is clean in all $24{,}000$ of their rounds while the greedy packer is inexact somewhere on all $60$ at the uniform mix and on $40$ at the skewed one, with a worst round at $0.875$ of the integer optimum. Exactness is $1$ on every laminar family, as the laminar-matroid theorem requires, so on those families it checks the instrument. The loss appears at the first crossing and does not deepen with further ones: across the crossing counts the unimodular crossing families take, one to six, the mean exactness stays between $0.986$ and $0.993$ over both mixes with no trend. The shortfall is displaced tasks over admitted tasks, and the admitted count is the covariate: the generated families admit about eighty tasks a round, so one displaced task costs about a percent, while the pipeline's own instances admit $100$, $100$ and $50$ under the sweep's demand.

\paragraph*{A structural tie to a companion deployment} A domain-level reading of the three pipeline templates of the companion deployment~\cite{loven2026nps} over its administrative domains yields a laminar family, hence clearing prices in every round of the sweep's demand family. The reading is a translation and not a measurement of that deployment: it treats the three templates as demand classes arriving together, which is a counterfactual the companion campaign did not run, and its five domain nodes, the distributed unit, the centralised unit, the near-real-time and the non-real-time RAN intelligent controllers, and the service-management-and-orchestration layer, are a re-split of the companion's four administrative domains. The tie is structural only, and no price-stability, latency or welfare consequence follows from it.

\paragraph*{Stability senses the structure does not order} Four senses of stability were measured on the three evaluated instances at high load with the uniform leaf mix, in rounds where the crossing is demonstrably active: at this operating point the crossing breaks the greedy packer's exactness in $0.147$ [$0.124$, $0.170$] of X's rounds and in none of T's or S's. One of the four, the equilibrium check, separates X from T as a measurement, by less than the two laminar instances differ from each other. The price process passes the equilibrium check in $0.330$ [$0.247$, $0.413$] of rounds on T and $0.220$ [$0.150$, $0.290$] on X at a budget of $1{,}000$ iterations, a paired X-minus-T difference of $-0.110$ [$-0.207$, $-0.013$], and in $0.530$ [$0.488$, $0.572$] on S, a T-minus-S difference of $-0.200$ [$-0.283$, $-0.117$]. Terminal prices are determinate in $0.100$ [$0.055$, $0.145$] of rounds on T, $0.075$ [$0.030$, $0.120$] on X and $0.530$ [$0.488$, $0.572$] on S, and the paired X-minus-T difference is $-0.025$ [$-0.086$, $+0.036$], null, against a T-minus-S difference of $-0.430$ [$-0.489$, $-0.371$]. Under a small report change and a re-clear, the share of perturbations moving more than one exchange at the largest downward step is $0.023$ [$0.015$, $0.032$] on T and $0.022$ [$0.012$, $0.031$] on X, and the paired difference over all four step sizes is $-0.0010$ [$-0.0052$, $+0.0032$], null. After a shock the contracted T and X rows are identical by construction, because the contraction makes their quotients the same object. Laminarity orders none of this: T and S are both laminar and differ from each other more, on every one of the four responses, than X differs from T. T and S also differ in population, leaf-block capacity and arc count, so their gap is not a structural yardstick either, and the battery attributes the ordering it measures to no factor. Two paired separations of X from T are measured, the converged fraction above and the burst-shock price overshoot, $+0.446$ [$+0.274$, $+0.617$] for X over T; the gap between the two laminar instances on the same responses is larger, $-0.200$ on the converged fraction and $-1.121$ [$-1.933$, $-0.309$] on the overshoot for T against S, and neither separation is attributed to the crossing.

\paragraph*{The skewed mix} Repeating the battery at the skewed leaf mix costs T and X on every sense, $0.265$ [$0.181$, $0.349$] and $0.150$ [$0.064$, $0.236$] of their converged fraction and $1.03$ [$0.83$, $1.22$] and $0.86$ [$0.65$, $1.07$] floors of added price spread, and leaves S untouched. Those rows measure a harder load on the binding edge node. The skewed mix moves $0.55$ of arrivals onto one leaf and cuts the shared leaf from a quarter of arrivals to a twentieth, so it loads one side of the crossing and starves the other: the crossing pair is jointly over-demanded in $0.0$ percent of the skewed rounds the battery clears, against $48.1$ percent of the uniform ones, and the substrate's own diagnostic reports X's binding set at the skewed mix as the single node that is also T's. A null between those two configurations is a null between two structurally identical ones, so the uniform rows remain the battery's evidence on the crossing.

\paragraph*{The exchange bound} The exchange bound is the one stability sense the leaf-block family decides, and it is decided at fixed prices rather than under a re-clear. Holding the price vector at the round's own terminal value and re-packing the perturbed report isolates the packing rule's response, which a matroidal region bounds at the perturbed task plus one exchange partner. Under that reading the bound holds in every one of the $31{,}360$ perturbations on the two laminar instances, the worst Hamming distance being two, and it is exceeded on the crossing instance four times in $8{,}000$ perturbations at the uniform mix, in two rounds of two hundred, with a worst Hamming distance of four; at the skewed mix, where the crossing is inactive, it is exceeded in none of $8{,}000$. The effect is a structural statement with a measured frequency and not a significant contrast over ten seeds, and the magnitude travels with it. Under a re-clear the same bound is exceeded on the laminar instance too, in $2.3$ percent of the perturbations at the largest downward step, which is the price vector's response and not the packing rule's.

\subsection{Extended Mechanism-Design Discussion}
\label{app:mechanism-details}

This section provides additional detail on two aspects of the mechanism-design results that are condensed in \mainname{} for space.

\paragraph*{Governance preservation of polymatroidal structure}
Governance constraints $\mathcal{X}_{\mathrm{gov}}$ impose coordinate-wise upper bounds $x_l \le u_l$ on each leaf service, encoding trust-based admission, jurisdictional filters, and access control. The intersection of a polymatroid $P(f)$ with coordinate truncations $\{x : x_l \le u_l\}$ is again a polymatroid, with rank function $f'(S) = \min_{T \subseteq S}\bigl\{f(T) + \sum_{l \in S \setminus T} u_l\bigr\}$~\cite{fujishige2005submodular}. Hence $\mathcal{X}_t = \mathcal{X}_{\mathrm{res}} \cap \mathcal{X}_{\mathrm{gov}}$ inherits polymatroidal structure. This preservation result relies on governance constraints being expressible as independent per-leaf upper bounds; more complex governance couplings (e.g., mutual exclusivity across leaves, cross-jurisdictional budget quotas, or path-specific exclusions) would generally not preserve polymatroidal structure and require separate treatment.

\paragraph*{Roles of polymatroidal structure and inter-epoch limitations}
The polymatroidal structure serves three distinct roles in the chain of Propositions~1--2 and Lemma~1: characterising the feasibility geometry of the supply region (Proposition~1), enabling Walrasian equilibrium existence under GS valuations via discrete convex analysis (Proposition~2(i)), and supporting the clinching-auction implementation in the single-parameter case (\cref{rem:clinching}). Across epochs, load-induced changes to slice latencies and capacities may alter the items being traded; the mechanism is re-invoked at each epoch under updated parameters, and the GS/equilibrium guarantees apply epoch-by-epoch and yield no dynamic-equilibrium result. In particular, the repeated interaction across epochs can induce strategic dynamics (learning, reputation gaming, or inter-epoch demand manipulation) that are outside the scope of the present per-epoch theorems.

\section{Emulated Testbed}
\label{app:emul-testbed}

This replay measures the load response of real inference, against which the model's congestion term is calibrated, and carries no encapsulation claim; its numbers are scoped to that calibration. The admitted-task sequence of the agentic workload (\cref{app:exp8-agentic}) is replayed on three containers, one per tier, each running the same multi-step tool-using agent against one shared inference backend on the model \texttt{mistral:7b-instruct-q4\_K\_M}. Per-tier admission caps are $2$ at device, $3$ at edge and $5$ at cloud, the smallest integer realisation of the simulator's own capacity ratio $200{:}300{:}500$. A concurrency probe of the shared backend puts its useful concurrency at $c^{*} = 1$, so these caps sit two to five times above it: the admission gate is slack by design, and the queue that responds to load is the backend's. The offered scale is $0.016667$. Two loads are replayed, low and high, at $50$ rounds each, interleaved in blocks of ten; the mean number of concurrent agent instances per round is $1.90$ at low load and $3.00$ at high load, and the market generates $101.8$ tasks per round at low load against $300.2$ at high, so offered load rises by a factor of $2.948$; the response of each world is its median-latency ratio over that same step, and the elasticity below divides by the factor less one, $1.948$. Both worlds divide by the one step: the offered-load series is read once from the market's own allocation export and used as the denominator on each side, so the ratio between the two responses does not depend on which common step is chosen. Normalising the testbed instead by its own realised concurrency step, $1.90$ to $3.00$ instances, while the simulator keeps the offered step would compare two different variables. Cross-tier delays of $5$, $15$ and $50$~ms are installed on the device, edge and cloud hops. These delays are declared assumptions recorded in the run metadata, with no network path measured. Backend identity is enforced by a fingerprint check before each replay: every container's model list must carry both the model and a host-only sentinel tag, so a container answering from its own model store is refused. The clean replay carries $245$ tasks and $980$ stage calls. The two run directories behind these tables are named with the invocations in the repository's testbed section; their artefacts are held outside the code deposit, and the tables are reproduced by re-running the recorded invocations against a local Ollama serving the same model, where exact figures will differ because real inference is not deterministic.

\cref{tab:emul-compare} gives the calibration statistics. Both sides are restricted to the same replayed task identifiers, and the intervals are BCa bootstrap intervals over rounds within one seed ($B = 2{,}000$).

\begin{table*}[tp]
\centering
\caption{Calibration statistics, simulation against emulated testbed. Intervals are BCa bootstrap over rounds within one seed ($B = 2{,}000$); $n$ is rounds and tasks per side.}
\label{tab:emul-compare}
\footnotesize
\setlength{\tabcolsep}{4pt}
\begin{tabular}{@{}llllccl@{}}
\toprule
Statistic & Load & Simulation [CI] & Testbed [CI] & Overlap & Sign & $n$ \\
\midrule
Latency elasticity & high/low & $0.0496$ [$0.0354$, $0.0614$] & $0.2597$ [$0.1777$, $0.3753$] & No & Same & $100$ / $245$ \\
Tail ratio & low & $1.1305$ [$1.0958$, $1.2025$] & $1.6729$ [$1.4415$, $1.9889$] & No & Same & $50$ / $95$ \\
Tail ratio & high & $1.1780$ [$1.1293$, $1.2293$] & $1.4045$ [$1.2900$, $1.5635$] & No & Same & $50$ / $150$ \\
Drop rate & low & $0.0000$ [none] & $0.9474$ [$0.8816$, $0.9796$] & n/a & Differ & $50$ / $95$ \\
Drop rate & high & $0.0067$ [$0.0000$, $0.0200$] & $0.9933$ [$0.9800$, $1.0000$] & No & Same & $50$ / $150$ \\
Drop-rate step & high$-$low & $0.0067$ [$0.0000$, $0.0200$] & $0.0460$ [$0.0075$, $0.1099$] & Yes & Same & $100$ / $245$ \\
\bottomrule
\end{tabular}
\end{table*}

The \emph{Simulation} column is the simulated market whose admitted sequence the replay runs, a model distinct from the node substrate whose congestion term is calibrated to the testbed. The latency elasticity has the same sign in both worlds. The tail ratios are at least one by construction, so their agreement in sign carries no information, and the simulator's drop-rate step has an interval that touches zero. The intervals separate on the latency elasticity and on both tail ratios: the model understates the sensitivity of latency to offered load by a factor of about $5.2$ at these two operating points, and its tail is lighter than the tail real inference produces. That factor is a ratio of two elasticities taken over one shared offered-load step, so it is invariant to the choice of step. The $0.267$ the node-level congestion term is calibrated to (\cref{app:simulation}) is a different pairing: it is the node grid's own elasticity over its own medium-to-high step, set beside the $0.2597$ this table reports for the testbed, so the two are compared as scale-free elasticities and not over one shared step. The drop-rate step's intervals overlap, but the testbed's drop rates are $0.9474$ and $0.9933$, close to the ceiling of that statistic, so little is read from it.

\cref{tab:emul-decomposition} decomposes each stage call into the wait at the tier's admission gate and the service time at the backend.

\begin{table}[htbp]
\centering
\caption{Per-tier decomposition of the clean replay into admission wait and service time, in ms. Wait fraction is the median admission wait over the sum of the median wait and the median service time. The $p_{95}$ columns are nearest-rank.}
\label{tab:emul-decomposition}
\footnotesize
\setlength{\tabcolsep}{3pt}
\begin{tabular}{@{}lrrrrrrr@{}}
\toprule
Tier & Cap & $n$ & W.\ p50 & W.\ p95 & S.\ p50 & S.\ p95 & W.\ frac. \\
\midrule
Device & $2$ & $245$ & $0.0$ & $1042.0$ & $1380.0$ & $2839.5$ & $0.000$ \\
Edge & $3$ & $490$ & $0.0$ & $3364.3$ & $2943.2$ & $7158.9$ & $0.000$ \\
Cloud & $5$ & $245$ & $0.0$ & $0.0$ & $3870.6$ & $8331.5$ & $0.000$ \\
\bottomrule
\end{tabular}
\end{table}

The admission wait is zero at the median on all three tiers and the median-based wait fraction is $0.000$ everywhere, although the tail admission wait at edge ($3364.3$~ms at $p_{95}$) exceeds the median service time there, while the per-call service median rises from $2286.4$~ms at low load ($n = 380$ stage calls) to $2754.9$~ms at high load ($n = 600$) at an all-but-identical token median, $84.5$ at low load against $84.0$ at high. The load dependence therefore sits in the service term and not in the per-tier admission queues the model places it in: the binding queue is inside the one inference backend the three tiers share, and the testbed's analogue of the model's congestion term sits in service time.

A second replay runs high load only, $50$ rounds, with the cloud container killed before round $25$, and carries $150$ tasks and $600$ stage records. The calibration statistics above are computed on the clean replay alone; a replay containing a kill is refused as their input. \cref{tab:emul-failure} gives the completion gap.

\begin{table}[htbp]
\centering
\caption{Completion gap under a permanent loss of the cloud container. Testbed completion is the share of tasks that finish; the simulation column is the share that meet their deadline in the same rounds.}
\label{tab:emul-failure}
\footnotesize
\setlength{\tabcolsep}{4pt}
\begin{tabular}{@{}lrrrrr@{}}
\toprule
Phase & Rnds & Tasks & Testbed & Sim. & Gap \\
\midrule
Pre-kill ($1$--$24$) & $24$ & $72$ & $1.0000$ & $1.0000$ & $0.0000$ \\
Post-kill ($25$--$50$) & $26$ & $78$ & $0.0000$ & $0.9872$ & $0.9872$ \\
\bottomrule
\end{tabular}
\end{table}

Completion is $1.0000$ in every round before the kill and $0.0000$ in every round after it, against a simulation that scores $1.0000$ in $49$ of the $50$ rounds because it carries no failure model and goes on admitting tasks the testbed can no longer run. The gap saturates in the first degraded round. Stage-level attribution after the kill is $0$ successful calls and $78$ connection errors at cloud, $78$ successful at device and $156$ successful at edge: the killed tier hosts the aggregation stage, the single sink of the replayed call graph, so the two surviving tiers keep answering normally and the task still fails.

The replay is a three-container emulation on one machine, not a deployment, and its comparisons are relative. Absolute latency and deadline outcomes do not transfer, and the replay calibrates the relative load response only: the testbed-over-simulation median latency ratio is $2.65$ at low load and $3.36$ at high load, and the testbed misses nearly every deadline while the simulator misses one of $150$ at high load and none at low, because the deadlines are derived from the measured profile while the testbed pays for real inference. One topology is replayed, so the replay does not test the structural results. The three tiers share one inference backend, so they are not independent, and the service-time result measures that coupling. With one container per tier there is nothing within a tier to re-route to, and the tier-aggregated model has no intra-tier structure either, so fault tolerance, failover and re-routing lie outside the replay. The cross-tier delays are declared assumptions and no network path was measured, so wide-area conditions lie outside it too. One seed's admitted sequence is replayed and the intervals are over rounds within that seed, so they carry no seed-level inference.

\bibliographystyle{IEEEtran}
\bibliography{IEEEabrv,bib/refs}

\vskip -2\baselineskip plus -1fil
\begin{IEEEbiographynophoto}{Lauri Lov\'{e}n}
is an Assistant Professor (tenure track) and head of the Future Computing Group at the University of Oulu, Finland. %
\end{IEEEbiographynophoto}
\vskip -3.0\baselineskip plus -1fil
\begin{IEEEbiographynophoto}{Alaa Saleh}is a doctoral candidate at the Future Computing Group, in the Center for Applied Computing, University of Oulu, Finland. %
\end{IEEEbiographynophoto}
\vskip -3.0\baselineskip plus -1fil
\begin{IEEEbiographynophoto}{Reza Farahani} is a Postdoctoral Researcher at the Distributed Systems Group, TU Wien, Vienna, Austria.%
\end{IEEEbiographynophoto}
\vskip -3.0\baselineskip plus -1fil
\begin{IEEEbiographynophoto}{Ilir Murturi} is an Assistant Professor at the University of Prishtina and an external Senior Researcher at the Distributed Systems Group, TU Wien. %
\end{IEEEbiographynophoto}
\vskip -3.0\baselineskip plus -1fil
\begin{IEEEbiographynophoto}{Sujit Gujar} is an Associate Professor and the CA Technologies Faculty Chair at the International Institute of Information Technology, Hyderabad (IIITH), India.
\end{IEEEbiographynophoto}
\vskip -3.0\baselineskip plus -1fil
\begin{IEEEbiographynophoto}{Miguel Bordallo López} %
is an Associate Professor with the University of Oulu, where he leads distributed intelligence research at the 6G Flagship
program. %
\end{IEEEbiographynophoto}
\vskip -3.0\baselineskip plus -1fil
\begin{IEEEbiographynophoto}{Praveen Kumar Donta} is associate professor (docent) at the Department of Computer and Systems Sciences, Stockholm University, Sweden. %
\end{IEEEbiographynophoto}
\vskip -3.0\baselineskip plus -1fil
\begin{IEEEbiographynophoto}{Schahram Dustdar} is Full Professor of Computer Science heading the Distributed Systems research division at TU Wien, Austria, and an ICREA Research Professor in Barcelona, Spain. %
\end{IEEEbiographynophoto}

\end{document}